\RequirePackage{amsmath}
\documentclass[runningheads,a4paper,envcountsame]{llncs}
\usepackage[T1]{fontenc}
\usepackage[utf8]{inputenc}

\usepackage[toc,page]{appendix}

\usepackage{mathtools}
\usepackage{multicol}
\usepackage[shortlabels]{enumitem}
\newlist{alphenum}{enumerate}{2}
\setlist[alphenum]{label=(\alph*),beginpenalty=10000}
\newlist{menum}{enumerate}{2}
\setlist[menum]{label=(\arabic*),beginpenalty=10000}

\newif\ifextended%
\usepackage{comment}
\newenvironment{extended}{\ifextended\else\comment\fi}{\ifextended\else\endcomment\fi}
\newenvironment{short}{\ifextended\comment\else\fi}{\ifextended\endcomment\else\fi}

\newcommand{\extonly}[1]{\ifextended#1\else\fi}

\usepackage{centernot}
\usepackage{verbatim}
\usepackage{mathdots}
\usepackage{yhmath}
\usepackage{cancel}
\usepackage{color}
\usepackage{array}
\usepackage{multirow}
\usepackage{amssymb}
\usepackage{tabularx}
\usepackage{extarrows}
\usepackage{booktabs}
\usepackage{stmaryrd}
\usepackage{array}
\usepackage{calc}
\usepackage{framed}
\usepackage{fancyvrb}
\usepackage{makecell}

\usepackage{graphicx}
\usepackage{tikz}
\usetikzlibrary{backgrounds,positioning,trees,shapes,arrows,patterns,topaths,calc}
\usetikzlibrary{shapes.geometric,positioning}
\usetikzlibrary{fadings}
\usetikzlibrary{shadows.blur}

\usepackage{pgfplots}
\pgfplotsset{compat=1.17}

\renewcommand{\paragraph}[1]{\smallskip\par\noindent\textbf{#1}. }
\newcommand{\paragraphnopoint}[1]{\smallskip\par\noindent\textbf{#1} }

\usepackage{todonotes}

\usepackage{amssymb}
\usepackage{amsfonts}
\usepackage{mathtools}
\usepackage{bm}

\usepackage{float}
\usepackage{booktabs}

\usepackage[hidelinks]{hyperref}
\usepackage[nameinlink]{cleveref}
\usepackage{color}

\newtheorem{fact}[theorem]{Fact}
\newtheorem{observation}[theorem]{Observation}

\renewcommand{\longleftrightarrow}{\leftrightarrow}
\newcommand{\Top}{\mathsf{Top}}
\newcommand{\Rel}{\mathsf{Rel}}

\renewcommand{\phi}{\varphi}
\newcommand{\kn}{\unlhd}
\newcommand{\prop}{\mathsf{Prop}}

\newcommand{\share}{\ensuremath{\mathsf{share}}}

\newcommand{\mE}{\mathcal{E}}

\newcommand{\mfM}{\mathfrak{M}}

\newcommand{\A}{\forall}
\newcommand{\E}{\exists}

\newcommand{\B}{\Box}

\extendedtrue
\ifthenelse{\equal{\detokenize{long}}{\jobname}}{\extendedtrue}{}
\begin{document}
\title{Virtual Group Knowledge and Group Belief in Topological Evidence Models\extonly{ (Extended Version)}}
\titlerunning{Virtual Group Knowledge and Group Belief in Top.\ E.\ Models\extonly{ (Ext. Vers.)}}
\author{Alexandru Baltag\inst{1}
  \and
  Malvin Gattinger\inst{1}
  \and
  Djanira Gomes\inst{2}}
\authorrunning{A. Baltag et al.}
\institute{ILLC, University of Amsterdam, The Netherlands
\and Institute of Computer Science, University of Bern, Switzerland}
\maketitle
\begin{abstract}
We study notions of (virtual) group knowledge and group belief within multi-agent evidence models, obtained by extending the topological semantics of evidence-based belief and fallible knowledge from individuals to groups. We completely axiomatize and show the decidability of the logic of (``hard'' and ``soft'') group evidence, and do the same for an especially interesting fragment of it: the logic of group knowledge and group belief. We also extend these languages with dynamic evidence-sharing operators, and completely axiomatize the corresponding logics, showing that they are co-expressive with their static bases.
\end{abstract}

\section{Introduction}

A natural framework for reasoning about knowledge in distributed systems is Epistemic Logic: an umbrella term for modal logics that formalize notions of knowledge and belief for rational agents. Traditionally, these logics are interpreted on \emph{relational (Kripke) models}, according to Hintikka's semantics~\cite{hintikka}. It is also useful to have notions of knowledge and belief associated with \emph{groups}~\cite{ogdistrkn}. The best-known are \emph{distributed} and \emph{common knowledge}. The first is inherently linked to communication: it describes what a group of agents \emph{could come to know} after sharing their individual information with the group~\cite{halpernmoses}.
This ``virtual'' or ``potential'' aspect is made explicit in Dynamic Epistemic Logic~\cite{BDM2008:PhilInfoChapter,SEP,Hans-DELbook}, with dynamic operators for information sharing~\cite{resolution,BaltagEtAl18:GroupKnowledgeInterrogative,subgroups,Goldbach15:ModellingDemocraticDeliberation}.

Recently, \emph{topological models} for epistemic logics have gained popularity, see e.g.~\cite{aybukephd,benthempacuit,defknowledgebaltag,justifmodels,argbelief,aldomaster,saulmaster,saulpaper}.
An advantage of topological semantics is that it comes with a natural, semantical notion of \emph{evidence}, making the evidential basis of knowledge and belief apparent.

In this paper we use multi-agent topological evidence models, or \emph{topo-e-models}, which explicitly represent the topology of evidence~\cite{baltagevidence,baltagbeliefknowledgeevidence}. One way to interpret knowledge and belief in topo-e-models is to apply the \emph{interior semantics} of McKinsey and Tarski~\cite{mckinseytarski} to (a basis for) the so-called \emph{dense-open topology}. This restricts the evidential topology to \emph{dense open sets}, which represemt \emph{``uncontroversial'' evidential justifications}: 
pieces of evidence consistent with all other evidence. Belief amounts to having such a justification, and (fallible, defeasible) knowledge is interpreted as \emph{correctly justified belief}~\cite{baltagevidence}. 

A natural continuation of this research is to extend the framework to the multi-agent case and to incorporate a notion of group knowledge. It has long been noticed~\cite{aldomaster,saulmaster,aybukephd,saulpaper} that the most straightforward such extension is obtained by applying the same definitions (as for individual knowledge and belief) to the \emph{join topology}, obtained by pooling together all the individual evidence. One objection~\cite{aldomaster,saulmaster} raised against this notion is that it 
loses the main characteristic property of classical distributed knowledge, namely \emph{Group Monotonicity} (saying that \emph{a group potentially knows everything known by any subgroup}): in topo-e-models, group knowledge is \emph{not} monotonic with respect to group inclusion. In fact, a group may sometimes know even \emph{less} than \emph{any} of its members~\cite{baltagevidence}.\footnote{In order to resolve this, two alternatives of this semantics have been proposed~\cite{aldomaster,saulmaster}, both ensuring the validity of the Group Monotonicity property.}

Nevertheless, in this paper we argue that Monotonicity will have to fail for \emph{any} realistic notion of group knowledge.
Since fallible knowledge is not fully introspective, agents cannot separate it from other beliefs in order to share it; so the best they can do is to share all their \emph{evidence}. And it turns out that the topological notion of group knowledge matches the knowledge that can be obtained after evidence-sharing. In this sense, topological group knowledge accurately captures the \emph{group's epistemic potential}: its true ``virtual'' knowledge.
In a nutshell:
the failure of Group Monotonicity is a ``feature'', not a ``bug''. 

Given this fact, it becomes imperative to study \emph{the laws governing this natural notion of group knowledge, and the corresponding concept of group belief}. In this paper, we provide \emph{complete and decidable axiomatizations} of these notions, as well as of the related concepts of \emph{group evidence}. While our axioms of group evidence are the expected ones (similar to any other distributed attitude in Epistemic Logic), it turns out that virtual group knowledge obeys new interesting laws, that can be seen as subtle forms of weakening Group Monotonicity. The completeness proof for the logic of group knowledge is also more intricate, relying on a new representation result. In order to make explicit the sense in which our 
notion captures a group's epistemic potential, we add \emph{dynamic evidence-sharing modalities}, and we completely axiomatize the resulting dynamic logics.

The paper is structured as follows.
\Cref{sec:background} presents topo-e-models
and defines our key epistemic notions.
\Cref{sec:Logics} gives the syntax, semantics and axiomatizations of our logics, and states our completeness/decidability results.
\Cref{sec:Conclusion} contains some conclusions and an open question for future work. 

\begin{short}
\paragraph{Extended Version}
\end{short}
This paper is based on the Master thesis of the third author~\cite{dosSantosGomes2025:thesisVGK}. The original proofs in~\cite{dosSantosGomes2025:thesisVGK} use somewhat different notations and definitions than the ones adopted here.
\begin{short}
  We provide an extended version of the paper at
  \url{https://malv.in/2025/VirtualGroupK-extended.pdf}
  with full proof details fitting the current version.
\end{short}

\section{Topological Knowledge and Evidence-Sharing}\label{sec:background}\label{sec:Background}

In this section we introduce \emph{multi-agent topological evidence models}, and define the notions of \emph{hard and soft evidence, knowledge and belief}, and their \emph{natural extensions to groups}.
We discuss the crucial differences between virtual group knowledge and the standard concept of distributed knowledge, and we explain and defend the first from a communication-based perspective, formalized in terms of an \emph{evidence-sharing update}.
The presentation is purely semantical-mathematical: we postpone the introduction of our formal languages to \Cref{sec:Logics}. 

\paragraph{Topological prerequisites} We first recall some basic topological notions. Given a set $X$, a \emph{topology} $\tau$ is a family of subsets of $X$, called \emph{open sets}. The \emph{closed} sets are given by their complements: $\bar{\tau}=\{X \setminus U \mid U \in \tau\}$. The topology $\tau$ by definition contains $\emptyset$ and $X$ as elements, and is closed under finite intersections and arbitrary unions. A set $X$ equipped with a topology $\tau$ is called a \emph{topological space}, denoted $(X, \tau)$. 

Given a space $X$, every set $A\subseteq X$ has an \emph{interior} and a \emph{closure}, which are computed by the \emph{interior} and \emph{closure} operators $Int_\tau, Cl_\tau : \mathcal{P}(X)\to \mathcal{P}(X)$, respectively. The \emph{interior} $Int_\tau(A)$ of $A\subseteq X$ is the union of all open subsets of $A$; the \emph{closure} $Cl_\tau(A)$ is its dual:
\begin{align*}
    Int_\tau(A) &= \bigcup\{U\in\tau\mid U\subseteq A\} \\
    Cl_\tau(A) &= \bigcap\{C\in\bar{\tau}\mid A\subseteq C\}.
\end{align*}
While the interior of $A$ is the largest open set contained in $A$, its closure is the least closed set containing $A$. 

A family $\mathcal{B} \subseteq \tau$ is a \emph{topological basis} for a topological space $(X, \tau)$ if every non-empty open subset of $X$ can be written as a union of elements of $\mathcal{B}$. A \emph{subbasis} for $(X, \tau)$ is a family $\mathcal{B} \subseteq \tau$, whose closure under finite intersections forms a basis for $(X, \tau)$. Given any family of subsets $\mathcal{E}\subseteq \mathcal{P}(X)$, we obtain the \emph{generated topology} by closing $\mathcal{E}$ under finite intersections and, subsequently, under arbitrary unions. The topology generated by $\mathcal{E}$ is the smallest topology $\tau$ on $X$ s.t.\ $\mathcal{E}\subseteq \tau$. The \emph{join} $\bigvee_{i\in I} \tau_i$ of a family ${\{\tau_i\}}_{i\in I}$ of topologies on the same set $X$ is defined as the topology generated by the union $\bigcup_{i\in I} \tau_i$.\footnote{This is the same as the \emph{supremum} of the family ${\{\tau_i\}}_{i\in I}$ in the lattice of all topologies on $X$ with inclusion.}

\paragraph{Topology-partition pairs, local density and the dense-open topology}
We shall consider \emph{topology-partition pairs} $(\tau, \Pi)$, consisting of a topology $\tau$ on a set $X$ and a partition $\Pi = \{\Pi(x) \mid x \in X\}$ of $X$ (where each $x$ belongs to a unique partition cell $\Pi(x)$), s.t.\ all partition cells are open (i.e., $\Pi \subseteq \tau$). For every open set $U\in \tau$, we denote by $\Pi(U):=\bigcup \{\Pi(x) \mid x\in U\}$ the union of cells of all points in $U$.
For a point $x\in X$, we say that $U$ is \emph{locally dense in $\Pi(x)$} (or ``locally dense at $x$'') if $Cl_\tau (U)\supseteq \Pi(x)$. 
We say that $U$ is \emph{locally dense in $\Pi$} (or just ``locally dense'', when $\Pi$ is understood) if $U$ is locally dense at all its points, i.e., $Cl_\tau (U)\supseteq \Pi(U)$.
It is easy to see that the family
\[\tau^{dense(\Pi)}\, := \, \{U\in \tau \mid Cl_\tau(U)\supseteq \Pi(U)\}\cup \{\emptyset\}\]
(consisting of all locally dense open sets and $\emptyset$) is itself a topology, called \emph{the dense-open topology} for $(\tau, \Pi)$. Once again, when $\Pi$ is understood from context, we skip it and just write $\tau^{dense}$ instead. 

\subsection{Knowledge and Belief in Multi-Agent Topo-Evidence Models}\label{backgr-topsingle}

\emph{Topological evidence models}~\cite{baltagevidence,baltagbeliefknowledgeevidence} are a variant of the \emph{evidence models} defined by~\cite{benthempacuit}, in which the role of the topology is stressed and the definition of belief is streamlined (to ensure its consistency). While~\cite{baltagevidence} studied these notions within a single-agent setting, this has been generalized to multi-agent models in~\cite{aldomaster,saulmaster,saulpaper}.

\paragraph{Vocabulary: atoms, agents and groups}
Throughout this paper, we fix a vocabulary, consisting of: a finite or countable set $\prop$ of \emph{atomic formulas} $p,q, \ldots$, intuitively denoting ``ontic facts'': non-epistemic features of the world; and a finite set $A=\{1, 2 \ldots, n\}$ of \emph{agents}, labeled by numbers, and
denoted by meta-variables $i, j, k, \ldots$. A \emph{group} is a non-empty set of agents (i.e., any $I\subseteq A$ with $I\neq \emptyset$). We use capital letters $I, J, K, \ldots$ as meta-variables for groups.

\begin{definition}[Topo-E-Models]\label{multiag-topoemodel}
  A \emph{multi-agent topological evidence model} (or ``topo-e-model'', for short) is a tuple
  $\mfM = (X, \Pi_1,\ldots, \Pi_n, \tau_1, \ldots, \tau_n, \llbracket \cdot \rrbracket)$ (or ${(X, \Pi_i, \tau_i, \llbracket \cdot \rrbracket)}_{i\in A}$ for short), where:
  \begin{itemize}
  \item[$\bullet$] $X$ is a set of \emph{states} (or ``possible worlds'');
  \item[$\bullet$] 
  For each $i \in A$, the family 
$\Pi_i\subseteq \mathcal{P}(X)$ is a partition of $X$, called \emph{agent $i$'s information partition}, and consisting of mutually disjoint partition cells. Every state $x\in X$ belongs to a
  \emph{unique cell} $\Pi_i(x)\in \Pi_i$, representing the \emph{ private information} --- the ``hard evidence'' --- possessed by agent $i$ in state $x$. The states $x'\in \Pi_i(x)$ are said to be \emph{indistinguishable} from $x$ by agent $i$;  
  \item[$\bullet$] 
    For each $i \in A$, $\tau_i \subseteq \mathcal{P}(X)$ is a topology on $X$, called \emph{agent $i$'s evidential topology}, and subject to the constraint that $\Pi_i \subseteq \tau_i$ (``\emph{hard evidence is evidence}''). The 
 non-empty open sets ($U\in \tau_i$ s.t.\ $U\neq\emptyset$) represent \emph{agent $i$'s (``soft'') evidence}. For any state $x\in X$, $\tau_i^* (\Pi_i(x)) :=\{U\in \tau_i \mid \emptyset\neq U\subseteq \Pi_i(x)\}$ is the collection of all \emph{soft evidence possessed by agent $i$ at state $x$}; while
    $\tau_i^*(x) := \{ U \in \tau_i \mid x \in U \subseteq \Pi_i(x) \}$ is the collection of \emph{agent $i$'s factive (``true'') evidence at state $x$}.\footnote{For the consistency of our notation, note that $\tau_i^*(\Pi_i(x))= \bigcup\{\tau_i^*(y): y\in \Pi_i(x)\}$.} We denote by $Cl_i$ and $Int_i$ the \emph{closure $Cl_{\tau_i}$ and interior $Int_{\tau_i}$ operators} with respect to agent $i$'s evidential topology $\tau_i$. 
  \item[$\bullet$] $\llbracket \cdot \rrbracket:X\rightarrow\mathcal{P}\left(\prop\right)$ is a \emph{valuation} function, mapping each atomic formula $p\in \prop$
  to the set $\llbracket p \rrbracket\subseteq X$ of states ``satisfying'' $p$.
  \end{itemize}
\end{definition}
The intuition is that in state $x$, each agent $i\in A$ has some ``hard'' evidence $\Pi_i(x)$, as well as some pieces of ``soft'' evidence  $U\in \tau_i^*(\Pi_i(x))$.  Since $x\in \Pi_i(x)$, the hard evidence is \emph{infallibly true} (i.e., true with absolute certainty),\footnote{This is the reason we assigned \emph{only one piece of (private) hard evidence} $\Pi_i(x)$ to each agent $i$ at each state $x$. In principle, one can of course have many pieces of hard evidence; but, since they are mutually consistent (being all true in the actual world), the agent can just combine all of them by taking their intersection.} while soft evidence can be false (when $x\not\in U$); moreover, two pieces of soft evidence $U, V\in \tau^*_i(\Pi_i(x))$ may be mutually inconsistent (when $U\cap V=\emptyset$).
\smallskip

\paragraph{Subbasis presentation}
The evidential topology is sometimes specified using a \emph{designated subbasis} $\mE_i^0\subseteq \mathcal{P}(X)$, with $\emptyset\not\in \mE_i^0$. Intuitively, the sets $U\in \mE_i^0$ represent the ``basic'' or \emph{``primary'' evidence}: the pieces of evidence that are \emph{directly observable}. 
The agent then forms the family $\mE_i$ of \emph{conjunctive evidence} by taking the closure of $\mE_i^0$ under finite intersections.\footnote{Note that $X=\bigcap\emptyset\in \mE_i$.} 
Finally, she forms the topology $\tau_i$, as the family of  \emph{disjunctive evidence} (also known as ``arguments''),  by closing $\mE_i$ under unions.\footnote{Note that $\tau_i$ equals the topology \emph{generated} by $\mE_i^0$.} While the subbasis presentation is 
computationally less demanding, the
distinction between primary evidence and indirect (conjunctive or disjunctive) evidence does not play any role in the semantics.

\paragraph{Propositions and Operators} A \emph{proposition} in model $\mfM = {(X, \Pi_i, \tau_i, \llbracket \cdot \rrbracket)}_{i\in A}$ is a set of states $P\in \mathcal{P}(X)$. An example are \emph{atomic propositions}: those of the form $\llbracket p \rrbracket$, for $p\in \prop$. Note that the family $\mathcal{P}(X)$ forms a Boolean algebra, with the operations of set-complementation, intersection and union. Next, we define a number of (unary) \emph{propositional operators} $\Gamma: \mathcal{P}(X) \to \mathcal{P}(X)$. 

\paragraph{Hard evidence gives infallible knowledge}
Given a proposition $P\subseteq X$, we say that \emph{an agent $i$ has hard evidence for $P$ (or ``infallibly knows'' $P$) at state $x$} if $P$ is true at all states that are indistinguishable for $i$ from $x$, i.e., if $\Pi_i(x)\subseteq P$. 
Formally, the proposition 
``agent $i$ infallibly knows $P$'' is denoted by
\vspace{-1mm}
\[ {[\forall]}_i(P) \, :=\, \{x \in X \mid \Pi_i(x)\subseteq P\}. \vspace{-2mm}\]
This is an absolutely certain, ``infallible'' type of knowledge, hence it is \emph{factive}, i.e., we have ${[\forall]}_i (P) \subseteq P$, and \emph{fully (=positively and negatively) introspective}, i.e., we have
${[\forall]}_i(P) = {[\forall]}_i( {[\forall]}_i (P))$ and
$X - {[\forall]}_i (P)= {[\forall]}_i (X- {[\forall]}_i (P))$.

\paragraph{Interior as ``soft evidence'' operator}
We say that \emph{agent $i$ has factive evidence for $P$ at state $x$} if there is some 
$U\in \tau_i^* (x)$ with $U\subseteq P$;\footnote{Requiring $U\in \tau_i^* (x)$ with $U\subseteq P$ is in fact equivalent to requiring $U\in \tau_i$ with $x\in U\subseteq P$, as $U\in \tau_i^* (x)$ implies that $x\in U$ and that $U\in \tau_i$ and, conversely, the existence of an $U\in \tau_i$ with $x\in U\subseteq P$ implies the existence of an $U'\in \tau_i^* (x)$ with $U'\subseteq P$. Hence, throughout the paper, we use the two specifications interchangeably.} equivalently, if $x\in Int_i(P)$. The proposition ``$i$ has factive evidence for $P$'' is denoted by:
\vspace{-1mm}
\[ \Box_i(P) \, :=\, \{x\in X \mid \exists U \in \tau_i : x \in U \subseteq P \} = Int_i (P). \vspace{-1mm}\]
This attitude is again \emph{factive}, i.e., $\Box_i (P)=Int_i (P)\subseteq P$,  and \emph{positively (but not negatively) introspective}, i.e., $\Box_i (P)= \Box_i (\Box_i (P))$.
The dual of $\Box_i$ is denoted by $\Diamond_i (P)$ and matches \emph{topological closure}:
$\Diamond_i (P) \, :=\, X- \Box_i (X-P) = Cl_i (P)$.

\paragraph{Justified Belief} According to the dense-interior semantics~\cite{baltagevidence,baltagbeliefknowledgeevidence}, \emph{rational agents base their beliefs only on ``uncontroversial'' evidence}: those pieces of evidence that are not contradicted by any other evidence available to them.\footnote{Note that ``uncontroversial'' does not mean ``factive'': such evidence can be false.} \emph{Agent $i$ believes $P$ at state $x$} if $i$ has such ``uncontroversial'' evidence for $P$: some $U\in \tau_i^*(\Pi_i(x))$ s.t.\ $U\subseteq P$ and $U\cap V\neq \emptyset$ for all $V\in \tau^*_i (\Pi_i(x))$.  It is easy to see that an open subset $U\subseteq P$ is an uncontroversial piece of evidence for $P$ at $x$ for agent $i$ iff $U$ is \emph{locally dense at $x$} with respect to $(\tau_i, \Pi_i)$, i.e., iff $Cl_i (U)\supseteq \Pi_i(x)$. In this case, $U$ can be thought of as a \emph{justification} for (believing) $P$: one that ``coheres'' with all the available evidence. 
Equivalently, \emph{$P$ is believed at $x$ iff its interior is locally dense at $x$}. 
The operator for \emph{agent $i$'s belief} is denoted by
\vspace{-2mm}
\[ B_i (P) \,\, :=\, \{ x \in X \mid \Pi_i(x) \subseteq Cl_{i}(Int_{i}(P)),\}, \vspace{-1mm}\]
while its dual $\langle B_i\rangle (P)  :=  X- B_i (X-P)$ captures ``doxastic possibility''.

\paragraph{Fallible Knowledge}\footnote{Notions of knowledge that do not imply absolute certainty are called \emph{fallible}. In our setting, only the ``hard'' evidence $\Pi_i(x)$ provides ``infallible'' knowledge.} 
We say that \emph{an agent $i$ ``knows'' $P$ at state $x$} if she has a \emph{factive justification} (= true uncontroversial evidence) for $P$: there is some $U \in \tau_i^*(x)$, with $U \subseteq P$, and $Cl_{i}(U) \supseteq \Pi_i(x)$.
Equivalently, iff $x$ is in the locally dense interior of $P$ for $i$: $x\in Int_{i}(P)$ and $Cl_{i}(Int_{i}(P)) \supseteq \Pi_i (x)$.
We denote by $K_i (P)$ the proposition ``$i$ knows $P$'':
\[ K_i (P) \, :=\, \{x\in X\mid \exists U \in \tau_i : x\in U\subseteq P \mbox{ and }  Cl_i (U) \supseteq \Pi_i(x) \}.\vspace{-2mm} \]
\noindent In words: \emph{knowledge is correctly justified belief}.\footnote{Note the difference between \emph{correctly} justified belief and \emph{true} justified belief~\cite{aybukephd}.}
In contrast to ${[\forall]}_i(P)$, this type of knowledge is ``defeasible'': it can be defeated by ``misleading'' evidence~\cite{baltagevidence}.
Its dual 
$\langle K_i\rangle (P) \, :=\, X- K_i (X-P)$ captures a notion of ``soft epistemic possibility''.

\paragraph{Knowledge as Interior in the Dense-Open Topology} We characterized knowledge $K_i(P)$ of a proposition $P$ as the locally dense interior of $P$ (for $i$). Equivalently, we can characterize knowledge as the \emph{interior in the dense-open topology $\tau_i^{dense}$}: 
\vspace{-1mm}\[ K_i (P)= Int_{\tau_i^{dense}}(P).\]
That is, under our characterization, $K_i(P)$ coincides with the interior of $P$ in agent $i$'s topology of locally dense open sets. 

\paragraph{Connections between operators} For $P\subseteq X$ and $i\in A$, we have:
\vspace{-1mm}
\[ {[\forall]}_i (P) \subseteq \Box_i (P), \,\, \, \, \,\, \, \, K_i (P) \subseteq B_i (P).
\vspace{-2mm}
\]
\noindent In words: \emph{hard evidence is also soft evidence}, and \emph{agents believe the things they know}. More interestingly, we have the following equations, which will allow us to define belief and knowledge as abbreviations in one of our formal languages:
\[ B_i (P) = {[\forall]}_i \Diamond_i \Box_i (P),  \,\, \, \, \,\, \, \,
K_i (P) = \Box_i (P) \cap B_i (P), \,\, \, \, \,\, \, \, B_i (P) = \langle K_i\rangle K_i (P).\vspace{-1mm}\]
The first equation follows directly from the characterizations of ${[\forall]}_i$, $\Diamond_i$, and $\Box_i$. The second states that having a correct justification of $P$ amounts to having a justification for $P$, as well as a piece of factive evidence $U$ for $P$.\footnote{The left-to-right inclusion of this equation is immediate; for the converse inclusion, recall that agent $i$ has a justification for $P$ at $x$ iff $Int_i(P)$ is locally dense at $x$. By definition, $Int_i(P)$ contains $U$, which contains $x$, hence, the justification is correct.} Finally, the last equation says that belief is also definable in terms of fallible knowledge: \emph{belief is the ``soft possibility'' of knowledge}.\footnote{This observation was taken by Stalnaker as the basis of a version of knowledge-first epistemology, which differs from the more well-known Williamsonian knowledge-first conception, by the fact that it is positively introspective.}   

\subsection{Group Evidence, Group Belief and Group Knowledge}

The most natural way to generalize the above notions from individual agents $i\in A$ to \emph{groups} $I\subseteq A$ is to pool together all the hard and soft evidence possessed by agents in $I$ into a group partition $\Pi_I$ and a group evidential topology $\tau_I$.

\paragraph{Group Evidence: join partition and join topology}
Given a group $I\subseteq A$ and a topo-e-model $\mfM = {(X, \Pi_i,\tau_i, \llbracket \cdot \rrbracket)}_{i\in A}$, \emph{group $I$'s hard evidence
at a state $x\in X$} is the intersection (conjunction) of all individual group members' hard evidence at state $x$. 
The group's hard-evidence sets form again a partition $\Pi_I$, 
called \emph{group $I$'s partition}, which coincides with the join (supremum) $\bigvee \Pi_i$ of all individual partitions (in the lattice of partitions on $X$ with inclusion):\footnote{This is the smallest partition $\Pi_I$ that includes that union $\bigcup_{i\in I} \Pi_i$.}
\[ \Pi_I:\, \, = \textstyle\bigvee_{i\in I}\Pi_i = \{\Pi_I (x) \mid x\in X\}, \,\, \mbox{ where } \,\,
\Pi_I (x):=\textstyle\bigcap_{i \in I}\Pi_i(x). 
\vspace{-1mm}
\]
Similarly, \emph{group $I$'s evidential topology} $\tau_I$ is 
just the \emph{join topology}
\[ \tau_I\,\, :=\,\, \textstyle\bigvee_{i \in I} \tau _i \, \, (=\mbox{the topology generated by the union $\bigcup_{i\in I} \tau_i$}). \]
To motivate this, note that $\tau_I$ is also
generated by \emph{the group's ``joint evidence''}, i.e.\ by the family of all non-empty intersections $\bigcap_{i\in I} U_i\neq\emptyset$ of individual pieces of soft evidence $U_i$ possessed by any of the group's members $i\in I$.
As before, we use $Int_I$ and $Cl_I$ for the interior and closure operators w.r.t.~$\tau_I$.

\paragraph{Group Operators}
A \emph{group operator} on a set $X$ is a group-indexed family $\Gamma={\{\Gamma_I\}}_{I\subseteq A, I\neq \emptyset}$ of propositional operators $\Gamma_I:\mathcal{P}(X)\to \mathcal{P}(X)$. As usual, when $I=\{i\}$ is a singleton consisting of a single agent, we write $\Gamma_i$ instead of $\Gamma_{\{i\}}$.

\paragraph{Examples: group evidence, group belief, group knowledge}
As important examples, we define \emph{group analogues} of all the individual attitudes, by simply \emph{applying the same definitions to the group partition and the group's soft evidence}:
\[
  \begin{array}{lrl}%
    {[\forall]}_I (P) & := & \{x\in X \mid \Pi_I(x)\subseteq P\},                                                               \\
    \Box_I (P)        & := & \{x\in X \mid \exists U \in \tau_I : x\in U\subseteq P \}  = Int_I (P),                              \\
    B_I (P)           & := & \{x\in X\mid \exists U \in \tau_I : U\subseteq P \mbox{ and } Cl_I (U)\supseteq \Pi_I(x) \}        \\
                      & =  & \{x\in X\mid  \Pi_I(x)\subseteq Cl_{I} (Int_{I}(P))\},                                             \\
    K_I (P)           & := & \{x\in X\mid \exists U \in \tau_I : x \in U \subseteq P \mbox{ and } Cl_I(U) \supseteq \Pi_I(x) \} \\
                      & =  & \{x\in X \mid x\in Int_{I}(P) \mbox{ and } \Pi_I (x)\subseteq Cl_{I}(Int_{I} (P))\}.
  \end{array}
\]
The \emph{Diamond (possibility) operators} $\Diamond_I (P)$, $\langle B_I\rangle (P)$ and $\langle K_I\rangle (P)$ are defined in the same way (as De Morgan duals) for groups $I\subseteq A$ as for individuals $i\in A$. 

\emph{Group operators are connected in the same way as the individual ones}: we have 
${[\forall]}_I (P) \subseteq \Box_I (P)$, $K_I (P) \subseteq B_I (P)$, $B_I (P) = {[\forall]}_I \Diamond_I \Box_I (P)$, $K_I (P) = \Box_I (P) \cap B_I (P)$ and $B_I (P) = \langle K_I\rangle K_I (P)$.  As a consequence, we will \emph{define group belief and group knowledge as abbreviations} in one of our formal languages.

\paragraph{Group Knowledge is Interior in the Dense-Open Join Topology}
Similar to the alternative characterization of individual knowledge $K_i$ as interior w.r.t.\ the individual dense-open topology $\tau_i^{dense}$, we can equivalently characterize group knowledge $K_I$ as the \emph{interior operator w.r.t.\ the dense-open topology $\tau_I^{dense}=\tau_I^{dense(\Pi_I)} $ associated to the join topology $\tau_I$}:
$K_I(P) \,\, = \,\, Int_{\tau_I^{dense}} (P)$.

\paragraph{Interpreting the group operators}
The above definitions seem natural from a mathematical point of view. But what is the \emph{interpretation} of these group operators? Are they just formal analogues of the individual ones, with no intrinsic meaning or practical application, or do they capture some useful group attitudes?
To give a partial answer, we need the following generalized notions:

\paragraph{Monotonicity and Distributedness}
A group operator $\Gamma$ is \emph{monotonic} if it satisfies the \emph{Group Monotonicity} condition: $ I\subseteq J \, \, \mbox{ implies } \,\, \Gamma_I (P) \subseteq \Gamma_J (P)$. The operator $\Gamma$ is \emph{distributed} if it satisfies the \emph{Group Distributedness} condition:
\[ x \in \Gamma_I (P) \, \, \mbox{ iff } \, \, x\in \textstyle\bigcap_{i\in I} \Gamma_i (P_i) \mbox{ for some } {(P_i)}_{i\in I} \mbox{ s.t. } \textstyle\bigcap_{i\in I} P_i\subseteq P. \]
Distributedness implies that \emph{$\Gamma$'s behavior on sets can be recovered from its behavior on singletons}.\footnote{In philosophical jargon, the distributed group operators are \emph{summative} attitudes.}
Moreover, it is easy to see that \emph{every distributed operator is monotonic}. 

\paragraph{Example: distributed knowledge in relational structures}
The standard example of a distributed operator is the classical relational concept of \emph{distributed knowledge} $D_I$ in a multi-agent epistemic Kripke model, defined as the \emph{Kripke modality for the intersection of all agents' accessibility relations}. This notion satisfies Group Distributedness (and hence also Group Monotonicity).\footnote{Indeed, Group Monotonicity is the main axiom for $D_I$ in standard Epistemic Logic.} This  fits the intended meaning of $D_I$: a group's distributed knowledge is simply the result of ``adding'' or ``aggregating'' all the knowledge possessed by the individuals. 

\paragraph{Group evidence is distributed evidence}
It is easy to see that our group evidences operators ${[\forall]}_I$ and $\Box_I$ are distributed (and thus also monotonic). This provides the promised interpretation: \emph{a group's evidence is the result of ``adding'' or ``aggregating'' all the evidence possessed by the individuals}. 

\subsection{The ``Problem'' of Non-Monotonicity}
Unfortunately, we cannot use the ``distributed knowledge'' interpretation for our topological group knowledge and belief operators: \emph{neither $K_I$ nor $B_I$ are distributed group operators}, and \emph{they do not even satisfy the weaker Group Monotonicity property}! 
Moreover, a group may even \emph{fail to (know or even just) believe} facts that are \emph{known by} \emph{all} its members: in general, we have $\bigcap_{i\in I} K_i (P) \not\subseteq B_I (P)$, as shown by the following counterexample.\begin{short}
\end{short}

\begin{example}\label{counterexample}
  Let $\mfM = {(X, \Pi_i, \tau_i, \llbracket \cdot \rrbracket)}_{i\in A}$ 
  be given by: $\prop = \{p\}$; $A = \{a,b\}$;
  $X = \{w_1,w_2,w_3,w_4\}$; $\llbracket p \rrbracket=\{w_1, w_2, w_4\}$;
  partitions $\Pi_a = \Pi_b=\{\{X\}\}$; and topologies $\tau_a$ and $\tau_b$ are  generated respectively by subases 
  $\mE_a^0=\{\{w_2, w_4\}, \{w_3, w_4\}\}$ and
  $\mE_b^0=\{\{w_1, w_2\}, \{w_1, w_3\}\}$, representing each agent's primary or ``direct'' evidence.
  We can then calculate the topologies $\tau_a$, $\tau_b$ and  $\tau_{\{a,b\}} =\tau_a \vee \tau_b= \mathcal{P}(X)$. Note that $\tau_{\{a,b\}}$ is the discrete topology, generated by $\mE_A^0=\mE_a^0\cup \mE_b^0$.
  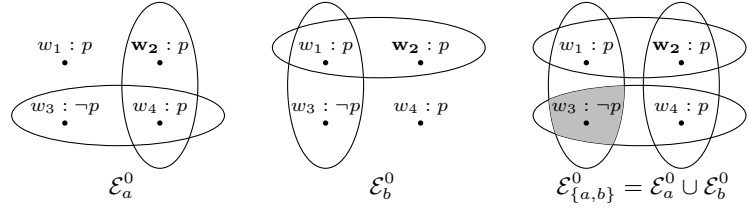
\begin{figure}[H]
    \centering
    \begin{tikzpicture}[scale=0.5, baseline={(0,-0.5)}] 
      \draw(1.5,3.2) ellipse (1 and 2.2); 
      \draw(0.4,2.4) ellipse (2.8 and 0.8); 
      \fill (-1,2.2) circle (2pt); 
      \fill (1.5,2.2) circle (2pt); 
      \fill (1.5,3.8) circle (2pt); 
      \fill (-1,3.8) circle (2pt); 
      \node[below] at (0.5, 1.1) {$\mE_a^0$}; 
      \node[above] at (-1,2.2) {\scriptsize$w_3:\neg p$};
      \node[above] at (1.5,2.2) {\scriptsize$w_4:p$}; 
      \node[above] at (1.5,3.8) {\scriptsize$\mathbf{w_2}:p$};
      \node[above] at (-1,3.8) {\scriptsize$w_1:p$};
    \end{tikzpicture}
    \hspace{0.4cm}
    \begin{tikzpicture}[scale=0.5, baseline={(0,-0.5)}] 
      \draw (-1,3.2) ellipse (1 and 2.2); 
      \draw(0.4,4.2) ellipse (2.8 and 0.8); 
      \fill (-1,2.2) circle (2pt); 
      \fill (1.5,2.2) circle (2pt); 
      \fill (1.5,3.8) circle (2pt); 
      \fill (-1,3.8) circle (2pt); 
      \node[below] at (0.5, 1.1) {$\mE_b^0$}; 
      \node[above] at (-1,2.2) {\scriptsize$w_3:\neg p$};
      \node[above] at (1.5,2.2) {\scriptsize$w_4:p$}; 
      \node[above] at (1.5,3.8) {\scriptsize$\mathbf{w_2}:p$};
      \node[above] at (-1,3.8) {\scriptsize$w_1:p$};
    \end{tikzpicture}
    \hspace{0.4cm}
    \begin{tikzpicture}[scale=0.5, baseline={(0,-0.5)}] 
      \def\secondellipse{(-1,3.2) ellipse (1 and 2.2)}
      \draw (-1,3.2) ellipse (1 and 2.2); 
      \draw(1.5,3.2) ellipse (1 and 2.2); 
      \def\firstellipse{(0.4,2.4) ellipse (2.8 and 0.8)}
      \draw(0.4,2.4) ellipse (2.8 and 0.8); 
      \draw(0.4,4.2) ellipse (2.8 and 0.8); 
      \begin{scope}
        \clip \firstellipse;
        \fill[lightgray] \secondellipse;
      \end{scope}
      \fill (-1,2.2) circle (2pt); 
      \fill (1.5,2.2) circle (2pt); 
      \fill (1.5,3.8) circle (2pt); 
      \fill (-1,3.8) circle (2pt); 
      \node[below] at (0.5, 1.1) {$\mE_{\{a,b\}}^0=\mE_a^0 \cup \mE_b^0$}; 
      \node[above] at (-1,2.2) {\scriptsize$w_3:\neg p$};
      \node[above] at (1.5,2.2) {\scriptsize$w_4:p$}; 
      \node[above] at (1.5,3.8) {\scriptsize$\mathbf{w_2}:p$};
      \node[above] at (-1,3.8) {\scriptsize$w_1:p$};
    \end{tikzpicture}
    \vspace{-6mm}
    \caption{The model from \Cref{counterexample}. 
     For each topology, we draw only the primary evidence (the subbases $\mE_a^0$, $\mE_b^0$ and $\mE_A^0=\mE_a^0\cup \mE_b^0$), and omit the single-cell partitions. Take $P=\llbracket p\rrbracket=\{w_1, w_2, w_4\}$. At $w_2$, $a$ has $\tau_a$-dense factive evidence $U_a=\{w_2, w_4\}$ for $P$, and $b$ has $\tau_b$-dense factive evidence $U_b=\{w_1, w_2\}$ for $P$, hence $w_2\in K_a (P)\cap K_b(P)$. But
     $\{w_3\}=\{w_3,w_4\}\cap \{w_1, w_3\}\in \mE_A$ is disjoint from $P$, hence $w_2\not\in B_{\{a,b\}}(P)$.}\label{fig:weather2}
  \end{figure}
\end{example}

\vspace{-7mm}

The failure of Group Monotonicity was taken as an \emph{objection} against the topological definition of group knowledge~\cite{aldomaster,saulmaster,aybukephd}. Consequently, Ramírez~\cite{aldomaster} and Fernández~\cite{saulmaster,saulpaper} proposed alternative notions of group knowledge in topo-evidence models, designed to ``save'' Group Monotonicity. Here we only present Fernández' solution, because of its relevance for our discussion and our axioms.

\paragraph{Fernández' approach: topological distributed knowledge}
In his Master thesis~\cite{saulmaster}, Fernández proposes a different topological definition of group knowledge, later developed and investigated by Baltag et al.~\cite{saulpaper}.  As we saw, individual knowledge $K_i$
for agent $i$ coincides with interior in agent $i$'s dense-open topology $\tau_i^{dense}$; while the virtual group knowledge operator $K_I$ coincides with interior in the \emph{group's} dense-open topology $\tau_I^{dense}$ (which is the dense-open topology for the pair $(\tau_I, \Pi_I)$, obtained by taking the joins of all the individual partitions and respectively all individual topologies). Fernández' proposal is to use instead
the natural topological analogue of distributed knowledge $D_I$, as the \emph{interior operator w.r.t.
the join $\bigvee_{i\in I} \tau_i^{dense}$
of all individuals' dense-open topologies}:
\vspace{-1mm}
\[ D_I (P) \, := \, Int_{\bigvee_{i\in I} \tau_i^{dense}} (P). 
\vspace{-1mm}\]
This topological notion generalizes the relational definition of distributed knowledge in $S4$ (or $S5$) Kripke models,\footnote{$S4$-frames are a special case of topological spaces (the Alexandroff spaces): the standard Kripke modality coincides with the interior operator in this case, and the relational definition of $D_I$ coincides with Fernández' topological definition.} and moreover \emph{the topological $D_I$ is indeed ``distributed''} (in the above sense), and it thus also satisfies Group Monotonicity.

\subsection{Dynamics: A Communication-Based View on Group Knowledge}\label{sec:shareI} 

In contrast to the mentioned authors, we will argue that the non-monotonic notion $K_I$ fits better than $D_I$ with a 
\emph{communication-based interpretation of group knowledge}. In the context of distributed systems (see e.g.~\cite{ogdistrkn}), the concepts of knowledge and communication are intertwined. A realistic notion of ``virtual'' group knowledge should be ``realizable'' (as individual knowledge) through in-group communication. As we will see, $K_I$ fulfills this desideratum (while $D_I$ does not), so $K_I$ is in fact more realistic and useful than $D_I$. 
To show this, we look at the group dynamics induced by evidence-exchange. 

\paragraph{Evidence-Sharing Dynamics}
For each group $I\subseteq A$, one can define an operator $\share(I)$ on topo-e-models, that represents the action of \emph{sharing all evidence (soft and hard) within group $I$}. This is a ``semi-public'' action in the sense of~\cite{subgroups}: intuitively, the outsiders $j \not\in I$ know that this sharing is happening within group $I$, but they do not necessarily have access to the evidence that is being shared; in fact, it is common knowledge among \emph{all} agents that this information-sharing event $\share(I)$ is happening; while the insiders $i \in I$ have more information: they gain common knowledge of which evidence is being shared among them.
This is an ``evidential'' version of other group-sharing operators in the literature: the ``deliberation'' action in~\cite{Goldbach15:ModellingDemocraticDeliberation}, the ``share'' action in~\cite{BaltagEtAl18:GroupKnowledgeInterrogative}, the ``resolution'' action in~\cite{resolution}, or the semi-public sharing actions considered in~\cite{subgroups}. 

\begin{definition}\label{def:share}
  Given a topo-e-model $\mathfrak{M} = {(X, \Pi_i, \tau_i, \llbracket \cdot \rrbracket)}_{i\in A}$ and a group $I \subseteq A$, \emph{the updated model}
  $\mfM({\share_I}):= (X, \Pi({\share_I}), \tau({\share_I}), \llbracket \cdot \rrbracket)$ has the same set of states and valuation, while the new partitions and topologies are given by:
  \vspace{-2mm}
  \[
    \begin{array}{llllll}
        {\tau_i({\share_I})} &= \tau_I, &\quad\quad \Pi_i({\share_I}) &= \Pi_I && \mbox{ (for ``insiders'' $i\in I$)},\\ 
        {\tau_j({\share_I})} &= \tau_j &\quad\quad \Pi_j({\share_I}), &= \Pi_j && \mbox{ (for ``outsiders'' $j\not\in I$)},
    \end{array}
      \vspace{-2mm}
    \]
\noindent where $\tau_I$ is the group's topology, and $\Pi_I$ is the group's partition. 
\end{definition}
Since the set of states $X$ and the valuation $\llbracket \cdot \rrbracket$ stay the same when moving from $\mathfrak{M}$ to the updated model $\mfM({\share_I})$, we can talk about the same
semantic propositions $P\subseteq X$ in both models. However, the meaning of our operators ${[\forall]}_i, \Box_i, K_i, B_i$ differs in the two models!
So we use ${[\forall]}_i^\mathfrak{M}, \Box_i^\mathfrak{M}, K_i^\mathfrak{M}, B_i^\mathfrak{M}$ to denote the operators in the model $\mathfrak{M}$, and ${[\forall]}_i^{\mathfrak{M}(\share_I)}$, $\Box_i^{\mathfrak{M}(\share_I)}$, $K_i^{\mathfrak{M}(\share_I)}$, $B_i^{\mathfrak{M}(\share_I)}$ to denote the operators
in the updated model ${\mathfrak{M}(\share_I)}$.

With these notations, we can now make the following key observation:

\begin{proposition}\label{Virtual-K}
Let $\mathfrak{M} = {(X, \Pi_i, \tau_i, \llbracket \cdot \rrbracket)}_{i\in A}$ be a topo-e-model. Then for every proposition $P\subseteq X$, every group $I\subseteq A$ and every group member $i\in I$, we have: 
\vspace{-2mm}
\[ {[\forall]}_i^{\mathfrak{M}(\share_I)}(P) = {[\forall]}_I^\mathfrak{M}(P), \,\,\,\,\,\,\, \,  \,\,\,\,\,\,\,
\Box_i^{\mathfrak{M}(\share_I)}(P) = \Box_I^\mathfrak{M}(P),\]
\vspace{-5mm}
\[B_i^{\mathfrak{M}(\share_I)} (P)= B_I^\mathfrak{M} (P), \,\,\,\,\,\,\, \,  \,\,\,\,\,\,\,
K_i^{\mathfrak{M}(\share_I)}(P) = K_I^\mathfrak{M}(P).\]
In words: the \emph{individual} group members' hard information ${[\forall]}_i$, soft evidence $\Box_i$, knowledge $K_i$, and belief $B_i$
\emph{after evidence-sharing} match the corresponding \emph{group} attitudes ${[\forall]}_I$, $\Box_I$, $K_I$, $B_I$ \emph{before} the evidence-sharing.
\end{proposition}

\paragraph{Interpretation}
This result provides a uniform interpretation of all our group operators: \emph{they simply ``pre-encode'' the individual members' attitudes after in-group evidence-sharing}! This justifies our name of ``virtual group knowledge'', and vindicates our topological definition of $K_I$ and $B_I$, from a communication-based perspective.
\emph{Topological group knowledge/belief is simply the knowledge/belief that the individual members could acquire by sharing all their evidence}.\footnote{As we will see, at the syntactic level, the above equalities have to be replaced by more complex Reduction laws, because the same sentence $\phi$ may denote different sets of worlds in $\mathfrak{M}$ and in $\mathfrak{M}(\share_I)$.}

\paragraphnopoint{Why not directly share knowledge?}
At first sight, it might seem that Fernández' topological distributed knowledge $D_I$ could be similarly given a communication-based interpretation, in terms of the epistemic situation after agents share all their \emph{knowledge} (rather than evidence). Such knowledge-sharing actions were considered in~\cite{Goldbach15:ModellingDemocraticDeliberation,BaltagEtAl18:GroupKnowledgeInterrogative,resolution,subgroups}, but all these proposals assumed an $S5$ setting, in which knowledge is absolutely certain and fully introspective: agents can infallibly distinguish what they know from what they don't know, so they can share exactly only what they know. In our non-$S5$ context, this is not realistic: when interested in fallible knowledge, we cannot assume such infallible powers of discrimination. Agents cannot be sure which of their beliefs (or pieces of evidence) are true and which are not, and they cannot select only those that constitute ``knowledge''. The best the agents can do is to either share all their \emph{beliefs}, or else share all their \emph{evidence} (as in $\share_I$), and then use this to build new consistent and justified beliefs (and thus obtain new knowledge).

\paragraph{Fallible knowledge must violate Group Monotonicity} From a communication based perspective on group knowledge, it would be questionable to impose Group Monotonicity on a \emph{fallible} notion of knowledge.
As widely recognized in the field of Belief Revision Theory~\cite{AGM1985}, the dynamics of belief (and so also the dynamics of fallible knowledge) must be \emph{non-monotonic}: if an agent fallibly knows a proposition, then further evidence might defeat that knowledge again.
The failure of Group Monotonicity is then simply an inescapable consequence of this non-monotonic dynamics: after receiving new ``soft'' evidence from other members of the group, agents may radically revise their beliefs, and thus may lose some of their prior ``knowledge''. This is a feature, not a bug: any realistic notion of (fallible) group knowledge will invalidate Monotonicity.

\paragraph{A concrete scenario for \Cref{counterexample}} To illustrate this point more concretely, consider the following scenario, underlying the model in \Cref{counterexample}.
Daisy was brutally murdered. Detective Bob is leading the case, and Alice is the jury foreperson in the murder trial. The accused is Daisy's husband: the lawyer Charles. The evidence at hand concerns whether Charles got caught in the act ($C:=\{w_1,w_2\}$ in \Cref{counterexample}), as well as his intent to kill ($I:=\{w_2,w_4\}$). Both killing and intent to kill are a crime. Therefore, both $C$ and $I$ individually imply that Charles is \emph{guilty} (proposition $p$ in \Cref{counterexample}). In the actual world ($w_2$), both $C$ and $I$ are factive. Charles is innocent only in world $w_3$. 

But now suppose Bob's evidence was deemed inadmissible, hence Alice does not have access to it. Conversely, Bob does not have access to Alice's evidence, as it is obtained through witness testimony in court. 
Moreover, suppose Alice has the following evidence:
\vspace{-2mm}
\begin{itemize}[$\bullet$]
    \item $I=\{w_2,w_4\}$: Testimonial evidence from Charles' colleague reveals that Charles was inquiring at work about the legalities of collecting life insurance after sudden death. Moreover, he did this only a week after having taken out a life insurance policy for Daisy, and days before her death. This evidence of intent is factive. 
    \item $\neg C=\{w_3,w_4\}$: Testimonial evidence from Charles' friend Ed provides an alibi for Charles at the time of the crime: they were watching tv at home. Ed lied under oath, and therefore this evidence is not factive. 
\end{itemize} 
Bob has the following evidence: 
\begin{itemize}[$\bullet$]
    \item $C=\{w_1,w_2\}$: An alcoholic, who was drunk when he witnessed the murdering of Daisy, identified Charles as the killer in a statement to the police. His statement was not confirmed by any third party: the evidence, although factive, is deemed inadmissible on grounds of being unreliable.  
    \item $\neg I=\{w_1,w_3\}$: Charles handed over to the police his periodical handwritten love letters to Daisy, dating back more than ten years, as evidence against intent. The letters, which appeared to be (and were, in fact) fabricated over the past week, were deemed inadmissible. 
\end{itemize}
Bob and Alice both individually know (fallibly) that Charles is guilty. However, sharing their evidence would result in reasonable doubt, since their (factive) individual evidence is defeated by some of the other's (non-factive) evidence.
\vspace{-1mm}

\section{Logics and Axiomatizations}\label{sec:Logics}
\vspace{-2mm}

In this section we introduce our logics for evidence, knowledge, belief and sharing of evidence, and present our main results on completeness and decidability.

\vspace{-2mm}

\subsection{The Logic of Group Evidence}

Our language of group-evidence $\mathcal{L}_{\B {[\A]}_I}$ will have modalities for soft and hard (group) evidence. We also study a fragment $\mathcal{L}_{{\B[\A]}_{i,A}}$, obtained by restricting these modalities to individuals and the full group $A$.

\paragraph{Notational convention} For concision, we use the symbol $\alpha\in \{A\}\cup A$ to denote either singletons $\{i\}\subseteq A$ or $A$ itself, 
when considering notions of group evidence, knowledge, and belief restricted to individuals or the full group $A$.

\begin{definition}[Syntax and Semantics of $\mathcal{L}_{{\B[\A]}_I}$ and $\mathcal{L}_{{\B[\A]}_{i,A}}$]\label{def:our-full-language}
    The language $\mathcal{L}_{{\B[\A]}_I}$ of evidence is defined recursively as
    \begin{align*}
        \phi::=p\mid\neg \phi\mid \phi \wedge \phi \mid \B_I \phi \mid {[\A]}_I\phi
         \vspace{-1mm}
    \end{align*}
    where $p \in \prop$ and $I$ is any group. The fragment $\mathcal{L}_{{\B[\A]}_{i,A}}$ of $\mathcal{L}_{{{\B[\A]}_I}}$ is obtained by restricting the modalities to $\B_\alpha$ and ${[\A]}_\alpha$, with $\alpha \in \{A\}\cup A$.
    For simplicity, we will write ${[\forall]}_i$ and $\Box_i$ instead of ${[\forall]}_{\{i\}}$ and $\Box_{\{i\}}$.
    
    Given a topo-e-model $\mathfrak{M} = {(X, \Pi_i, \tau_i, \llbracket \cdot \rrbracket)}_{i\in A}$, we define an \emph{interpretation} function $\llbracket \cdot \rrbracket^\mfM$, mapping every formula 
    $\phi$ of  $\mathcal{L}_{{\B[\A]}_{I}}$ to a proposition $\llbracket\phi\rrbracket^\mathfrak{M}\subseteq X$. 
    The interpretation extends the valuation, so whenever the model is understood we can skip the superscript without ambiguity, writing $\llbracket\phi\rrbracket$.
    The definition is by recursion on formulas: for atoms, $\llbracket p\rrbracket$ is just the valuation, and we let
    \[
    \begin{array}{llll}
        \llbracket \neg \phi\rrbracket & :=  X\setminus \llbracket \neg \phi\rrbracket, \,\,\, &
        \llbracket\phi \wedge \psi\rrbracket & :=  \llbracket\phi\rrbracket \cap  \llbracket\psi\rrbracket,\\
        \llbracket\B_I \phi\rrbracket & := \Box_I (\llbracket \phi \rrbracket)= Int_I (\llbracket \phi \rrbracket),\, &
        \llbracket [\A_I] \phi \rrbracket & := [\forall_I] (\llbracket \phi \rrbracket) := \{x\in X: \Pi_I(x)\subseteq \llbracket \phi \rrbracket\}.
    \end{array}
    \]
\vspace{-3mm}

\par\noindent The interpretation for $\mathcal{L}_{{\B[\A]}_{i,A}}$ is simply the restriction of $\llbracket\cdot\rrbracket$ to this language. As usual, we sometimes write $x\models \phi$ for $x\in \llbracket \phi\rrbracket$.
\end{definition}

\vspace{-1mm}

\paragraph{Abbreviations}
The Boolean connectives $\vee$, $\to$, $\longleftrightarrow$, and the modality $\Diamond_I$ (dual to $\Box_I$) are defined as abbreviations as usual.
\emph{Knowledge and belief are also abbreviations}: $B_I \phi \, :=\, {[\forall]}_I \Diamond_I \Box_I \phi$ and $K_I \phi  \, :=\,  \Box_I \phi \wedge B_I \phi$. It is easy to see that we have
$\llbracket B_I \phi\rrbracket = B_I (\llbracket \phi\rrbracket)$, $ \llbracket K_I \phi\rrbracket = K_I (\llbracket \phi\rrbracket)$.

\begin{theorem}\label{corr:compness-boxall-frag}
    The proof system $\bm{{\B[\A]}_{I}}$ from \Cref{pf-syst} is sound and complete for $\mathcal{L}_{{\B[\A]}_I}$ w.r.t.\ multi-agent topo-e-models, and the logic $\mathcal{L}_{{\B[\A]}_I}$ is decidable. All these properties are inherited by the proof system $\bm{{\B[\A]}_{i,A}}$ and the logic $\mathcal{L}_{{\B[\A]}_{i,A}}$.
\end{theorem}
\vspace{-6mm}
\begin{table}[H]
    \centering
    \begin{tabular}{ll}
    \toprule
     ($\mathsf{S4}_{\B}$) &\quad\quad $\mathsf{S4}$ axioms and rules for $\B_I$ \\
     ($\mathsf{S5}_{[\A]}$) &\quad\quad $\mathsf{S5}$ axioms and rules for ${[\A]}_I$ \\
    Monotonicity &\quad\quad $\B_J\phi\rightarrow \B_I\phi$, \,\,\,\,\,\,  \,\,\,\,\,\,  ${[\A]}_J\phi\rightarrow {[\A]}_I\phi$ \,\,\,\,\,\,  \,\,\,\,\,\,  (for $J\subseteq I$) \\
    Inclusion &\quad\quad ${[\A]}_I\phi\rightarrow \B_I\phi$ \\
    \bottomrule
    \vspace{0.01mm}
    \end{tabular}
    \caption{The proof system $\bm{{\B[\A]}_I}$, where $I, J\subseteq A$ are groups. The proof system $\bm{{\B[\A]}_{i,A}}$ for the fragment $\mathcal{L}_{{\B[\A]}_{i,A}}$ is obtained by restricting all axioms to $\mathcal{L}_{{\B[\A]}_{i,A}}$.
    }\label{pf-syst}
    \vspace{-7mm}
\end{table}

\subsection{The Logic of Group Knowledge and Group Belief}

To reason about knowledge and belief without explicitly mentioning notions of evidence, we introduce languages in which $K_I$ and $B_I$ are primitive operators.\footnote{As already noted, $B_I$ is definable in terms of $K_I$, so the belief operator is redundant. But our axioms are clearer when stated in terms of both modalities.}

\begin{definition}[Syntax and Semantics of $\mathcal{L}_{KB_I}$ and $\mathcal{L}_{KB_{i,A}}$]\label{def:KB-our-full-language}
    The language $\mathcal{L}_{KB_I}$ of group knowledge and belief is defined recursively as
    \vspace{-1mm}    
    \begin{align*}
        \phi::=p\mid\neg \phi\mid \phi \wedge \phi \mid B_I \phi \mid K_I\phi
    \end{align*}
    where $p \in \prop$ and $I$ is any group. As before, the fragment $\mathcal{L}_{KB_{i,A}}$ of $\mathcal{L}_{{KB_I}}$ is obtained by restricting the evidence modalities to $B_\alpha$ and $K_\alpha$, for all $\alpha \in \{A\}\cup A$.
    
    Given a topo-e-model $\mathfrak{M} = {(X, \Pi_i, \tau_i, \llbracket \cdot \rrbracket)}_{i\in A}$, the \emph{interpretation} map $\llbracket \cdot \rrbracket^\mathfrak{M}$ is as before for atoms and Boolean connectives, while for $B_I$ and $K_I$ we use the corresponding semantic operators (with $B_\alpha$ and $K_\alpha$ as special cases):
    \[
    \llbracket B_I \phi\rrbracket := B_I (\llbracket \phi \rrbracket)
    \hspace{2cm}
    \llbracket K_I \phi \rrbracket := K_I (\llbracket \phi \rrbracket)\]
\end{definition}

\begin{theorem}\label{corr:compness-kb}
The proof system $\bm{KB_{i,A}}$ listed in \Cref{pf-systKB} is sound and complete for $\mathcal{L}_{KB_{i,A}}$ w.r.t.\ multi-agent topo-e-models. Moreover this logic is decidable.
\end{theorem}
\begin{table}[ht]
\centering
\begin{tabular}{ll}
  \toprule
  \multicolumn{2}{l}{(KB) \hfill \textbf{Axioms \& rules of normal modal logic for $K$ \& $B$}} \\
  \midrule
  \multicolumn{2}{l}{\textbf{Stalnaker’s Epistemic-Doxastic Axioms:}} \\
Truthfulness of knowledge (T)   &\quad\quad  $K_\alpha\phi\rightarrow\phi$ \\
Pos. Intro.\ of knowledge (KK)  &\quad\quad  $K_\alpha\phi\rightarrow K_\alpha K_\alpha\phi$ \\
Consistency of Beliefs (CB)  &\quad\quad  $B_\alpha\phi\rightarrow \neg B_\alpha\neg \phi$ \\
Strong Pos. Intro.\ of beliefs (SPI)  &\quad\quad  $B_\alpha\phi\rightarrow K_\alpha B_\alpha\phi$ \\
Strong Neg. Intro.\ of beliefs (SNI)  &\quad\quad  $\neg B_\alpha\phi \rightarrow K_\alpha\neg B_\alpha\phi$ \\
Knowledge implies Belief (KB)  &\quad\quad  $K_\alpha\phi\rightarrow B_\alpha\phi$ \\
Full Belief (FB)  &\quad\quad  $B_\alpha\phi\rightarrow B_\alpha K_\alpha\phi$ \\
 \midrule
\multicolumn{2}{l}{\textbf{Group Knowledge Axioms:}} \\
Super-Introspection (SI) &\quad\quad $B_i\phi \rightarrow K_A B_i\phi$ \\
Weak Monotonicity (WM) &\quad\quad $(K_i\phi \wedge B_A\phi)\rightarrow K_A\phi$
\vspace{1mm}
\\ 
Consistency of group Belief with &\quad\quad $(\bigwedge_{i\in A} K_i\phi_i) \, \rightarrow \, \langle B_A\rangle (\bigwedge_{i\in A} \phi_i)$ \\
Distributed knowledge (CBD)  &\quad\quad\quad\quad (where $\{\phi_i\mid i\in A\}$ are arb.\ formulas) \\
\bottomrule
\vspace{0.01mm}
\end{tabular}
\caption{The proof system $\bm{KB_{i,A}}$, where $A$ is the group of all agents, $i\in A$ ranges over agents, and $\alpha\in \{A\}\cup A$ denotes either individual agents or the full group $A$.}
\label{pf-systKB}
\vspace{-7mm}
\end{table}
We briefly discuss the axioms. The first two groups contain generalizations (to multiple agents and groups)  of Stalnaker's axioms and rules for (individual) knowledge and belief~\cite{stalnaker}. These axioms were shown in~\cite{stalnakersknowledge} to be complete for the topological interpretation, and their completeness for multiple agents was shown in~\cite{saulmaster,saulpaper}. 
All these axioms and rules are standard in epistemic-doxastic logic, except for the Full Belief axiom (FB), which is specific to Stalnaker's conception of \emph{belief as the ``subjective feeling'' of knowledge}. Stalnaker calls this ``strong belief'', but we follow the terminology in~\cite{stalnakersknowledge}, referring to it as ``full belief''. The intuition is that an agent ``fully believes'' $\phi$ when she \emph{believes that she knows it}:
from a \emph{first-person} perspective, full belief and fallible knowledge are indistinguishable. 

Moving on to the Group Knowledge axioms, Super-Introspection (SI) is a strengthening of ordinary (strong) introspection of beliefs, stating that \emph{a group virtually knows the beliefs of its members}. Weak Monotonicity (WM) is a (valid) weakening of the (invalid) Group Monotonicity: individual knowledge of $\phi$ does imply virtual group knowledge of $\phi$
\emph{provided that the group virtually believes $\phi$}. 

Finally, Consistency of group Belief with Distributed knowledge (CBD) says that \emph{a group's virtual belief is
consistent with its distributed knowledge}.
In terms of Fernández' $D$-operator~\cite{saulmaster}, this could be stated as $D_A \phi \to \langle B_A \rangle \phi$.
Our language does not include a distributed knowledge modality, but (CBD) gives an equivalent statement in terms of conjunctions of individual pieces of knowledge.

\paragraph{Translation into the languages of evidence}
Every formula $\phi$ of $\mathcal{L}_{KB_I}$ and $\mathcal{L}_{KB_{i,A}}$ can be translated into a formula $tr(\phi)$ of the corresponding evidence languages $\mathcal{L}_{{\B[\A]}_I}$ and $\mathcal{L}_{{\B[\A]}_{i,A}}$: $tr(p)=p$, $tr(\neg \phi)=\neg tr(\phi)$, $tr(\phi \wedge \psi)= tr(\phi)\wedge tr(\psi)$, $tr(B_I \phi)= \forall_I \Diamond_I \Box_I tr(\phi)$,  $tr(K_I \phi)=\Box_I tr(\phi) \wedge \forall_I \Diamond_I \Box_I tr(\phi)$ (with $B_i, K_i, B_A, K_A$ as special cases). This translation is \emph{faithful}, i.e., $ \llbracket tr(\phi)\rrbracket=  \llbracket \phi\rrbracket$.
\vspace{-2mm}
\subsection{The Dynamic Logics of Evidence-Sharing}

We now extend our languages with \emph{dynamic modalities} $[\share_I]$ for evidence-sharing. Given the above completeness results, we only axiomatize two such logics: the extension of $\mathcal{L}_{{\B[\A]}_I}$ with $[\share_I]$ for arbitrary groups $I\subseteq A$; and the extension of $\mathcal{L}_{KB_{i,A}}$ with $[\share_A]$ for the full group $A$ of all agents.

\begin{definition}[Syntax and Semantics with {$[\share_I]$}]\label{def:syntax-dynamic}
The dynamic language $\mathcal{L}_{{\B[\A]}_I [\share_I]}$ is defined recursively as
\vspace{-1mm}   
    \begin{align*}
        \phi::=p\mid\neg \phi\mid \phi \wedge \phi \mid \B_I \phi \mid {[\A]}_I\phi \mid [\share_I]\phi 
       \vspace{-2mm}
    \end{align*}
(where $p\in \prop$ and $I\subseteq A$ is any group);  while $\mathcal{L}_{KB_{i,A} [\share_A]}$ is given by
\vspace{-1mm}   
    \begin{align*}
        \phi::=p\mid\neg \phi\mid \phi \wedge \phi \mid K_i \phi \mid K_A\phi \mid [\share_A]\phi 
      \vspace{-2mm}
    \end{align*}
(where $p \in \prop$, and $i\in A$ is any agent).

Given a topo-e-model $\mfM$, the \emph{interpretation} map $\llbracket\cdot\rrbracket^\mfM$ uses the clauses from \Cref{def:our-full-language} for the static connectives, while for the dynamic modalities, we put
    \[
    \llbracket [\share_I]\phi\rrbracket^{\mfM} = \llbracket \phi \rrbracket^{\mfM({\share_I})}
    \]
and apply the special case $I=A$ of this clause to interpret $[\share_A]\phi$.
\end{definition} 

\begin{theorem}\label{compness-boxall-dyn}
    The proof systems listed in Table~\ref{pf-syst-boxall-dyn} and Table~\ref{KB-dyn} are sound and complete for the corresponding logics $\mathcal{L}_{{\B[\A]}_I [\share_I]}$ and $\mathcal{L}_{KB_{i,A} [\share_A]}$ w.r.t.\ multi-agent topo-e-models.
    Moreover, these logics are provably co-expressive with their static bases $\mathcal{L}_{{\B[\A]}_I}$ and respectively $\mathcal{L}_{KB_{i,A}}$, and thus they are decidable.
\vspace{-4mm}
\end{theorem}
\vspace{-3mm}
\begin{table}[ht]
    \centering
    \begin{tabular}{ll}
    \toprule
    ($\bm{{\B[\A]}_I}$) &  \textbf{Axioms and rules of $\bm{{\B[\A]}_I}$} \\
    \midrule
    ($[\share_I]$) & \textbf{Axioms and rules of normal modal logic for $[\share_I]$}  \\
    \midrule
        &  \textbf{Reduction Axioms for $[\share_I]$}:\\
        (Atomic Reduction) &$[\share_I]p \leftrightarrow p$ \qquad\qquad\qquad\qquad (for atomic propositions $p$) \\
        (Negation Reduction) &$[\share_I]\neg \phi \leftrightarrow \neg [\share_I]\phi$ \\
        ($\B$-Reduction) &$[\share_I]\B_J\phi \leftrightarrow \B_{J/\!+ I}[\share_I]\phi$  \\
        ($\A$-Reduction) &$[\share_I]{[\A]}_J\phi \leftrightarrow {[\A]}_{J/\!+ I}[\share_I]\phi$  \\
    \bottomrule
   \vspace{0.001mm}
    \end{tabular}
    \caption{The proof system $\bm{{\B[\A]}_I[\share_I]}$, where $I,J\subseteq A$ are groups, and we 
    use the notation
    $J/\!\!+ I:=J\cup I$ when $I\cap J\neq\emptyset$, and $J/\!\!+ I:=J$ when $I\cap J=\emptyset$.}
    \label{pf-syst-boxall-dyn}
\vspace{-2mm}\end{table}
\vspace{-6mm}
\begin{table}[ht]
    \centering
    \begin{tabular}{ll}
    \toprule
    ($\bm{{\B[\A]}_I}$) &  \textbf{Axioms and rules of $\bm{KB_{i,A}}$} \\
    \midrule
    ($[\share_A]$) & \textbf{Axioms and rules of normal modal logic for $[\share_A]$}  \\
    \midrule
        &  \textbf{Reduction Axioms for $[\share_A]$}:\\
        (Atomic Reduction) &$[\share_A]p \leftrightarrow p$ \qquad\qquad\qquad\qquad (for atomic propositions $p$) \\
        (Negation Reduction) &$[\share_A]\neg \phi \leftrightarrow \neg [\share_A]\phi$ \\
        ($K$-Reduction) &$[\share_A]K_\alpha\phi \leftrightarrow K_{A}[\share_A]\phi$  \\
        ($B$-Reduction) &$[\share_A] B_\alpha\phi \leftrightarrow B_{A}[\share_A]\phi$  \\
    \bottomrule
    \vspace{0.001mm}
    \end{tabular}
    \caption{System $\bm{KB_{i,A}[\share_A]}$, where $\alpha\in A\cup \{A\}$ is an individual or the full group.}
    \label{KB-dyn}
    \vspace{-2mm}
\end{table}
\vspace{-2mm}

As usual in DEL, there is also a Conjunction Reduction: $[\share_I](\phi\wedge\psi) \leftrightarrow ([\share_I]\phi \wedge [\share_I]\psi)$. But this is provable from the axioms and rules of normal modal logic for $[\share_I]$ together with the Negation Reduction axiom for $[share_I]$. Its specical case $[\share_A](\phi\wedge\psi) \leftrightarrow ([\share_A]\phi \wedge [\share_A]\psi)$ is similarly 
provable from the normality of $[\share_A]$ and the Negation Reduction axiom for $[share_A]$. 

\vspace{-2mm}

\subsection{Proofs of Completeness and Decidability}

\begin{short}
The full proofs are relegated to the \href{https://malv.in/2025/VirtualGroupKnowledge-extended.pdf}{extended online version of this paper}.\footnote{See \url{https://malv.in/2025/VirtualGroupK-extended.pdf}.}
\end{short}
\begin{extended}
The full proofs can be found in the appendix.
\end{extended}
Here we sketch a brief summary of the proof plan and the main ideas of the proofs. 

For each of the proof systems $\bm{{\B[\A]}_I}$, $\bm{{\B[\A]}_{i,A}}$, and $\bm{KB_{i,A}}$, we first
show completeness w.r.t.\ \emph{non-standard relational structures} (pseudo-models), which are tailored to the respective languages. For $\bm{{\B[\A]}_I}$ and $\bm{{\B[\A]}_{i,A}}$, this is done using appropriate versions of the standard modal technique of \emph{filtration}, which gives us \emph{finite pseudo-models} for $\mathcal{L}_{{\B[\A]}_I}$ and $\mathcal{L}_{{\B[\A]}_{i,A}}$. In the case of $\bm{KB_{i,A}}$, we use the classical method of \emph{canonical structures}, obtaining an \emph{(infinite) canonical pseudo-model for $\mathcal{L}_{KB_{i,A}}$}, having an additional special property (\emph{max-density}).

The next step is to go back to the (intended) topo-e-models. For $\bm{{\B[\A]}_I}$ and $\bm{{\B[\A]}_{i,A}}$, we use a version of the well-known technique of \emph{unraveling}, showing that every pseudo-model for these logics is modally equivalent to its unraveled ``associated model'': this is a tree-like relational model, which is itself equivalent to a multi-agent topo-e-model. This finishes the proof of \Cref{corr:compness-boxall-frag}.  

For $\bm{KB_{i,A}}$, things are more complex: we have to first prove a representation theorem, showing that every pseudo-model for $\mathcal{L}_{KB_{i,A}}$ having the additional max-density property can be represented as a (p-morphic image of) a pseudo-model for $\mathcal{L}_{{\B[\A]}_{i,A}}$, in a way that preserves the truth of all formulas in $\mathcal{L}_{KB_{i,A}}$. This representation theorem is the key step, and its proof is non-trivial and uses an innovative technique.\footnote{Moreover, in contrast to the unraveling technique used for $\bm{{\B[\A]}_I}$, it is not clear how to generalize this step to arbitrary subgroups, i.e., to the logic $\bm{KB_I}$.} Given this and the above unraveling result, we obtain completeness of $\bm{KB_{i,A}}$  w.r.t.\ topo-e-models. This concludes \Cref{corr:compness-kb}. 

Finally, the completeness proof for the dynamic extensions (\Cref{compness-boxall-dyn}) follows a standard approach in Dynamic Epistemic Logic: we use the reduction axioms to show that these extensions are provably co-expressive with their static bases. Putting this together with Theorems~\ref{corr:compness-boxall-frag} and~\ref{corr:compness-kb}, we obtain \Cref{compness-boxall-dyn}.

\vspace{-2mm}

\section{Conclusion}\label{sec:Conclusion}

\vspace{-2mm}

The key theoretical contribution of this paper is the \emph{complete axiomatization} of non-monotonic, evidence-based notions of \emph{(virtual) group knowledge and group belief}, in the shape of the logic $\bm{KB_{i,A}}$. Compared to previous attempts at topological accounts of group knowledge (corresponding to a \emph{traditional} interpretation in terms of \emph{distributed} knowledge), \emph{the notion studied here is better suited to match the epistemic dynamics of knowledge induced by evidence-sharing}. This is a small step towards applying topological semantics to realistic, practical settings, such as distributed computing and the epistemology of social networks. 

As an auxiliary tool, we also studied\emph{ the logic of group evidence} over the larger language $\mathcal{L}_{{\B[\A]}_I}$, and showed that it is sound and complete, as well as decidable. In its turn, this result was an important step in showing the completeness and decidability of the above-mentioned logic $\bm{KB_{i,A}}$. 

\smallskip 

Unfortunately, we do not have a completeness result for the full logic $\bm{KB_I}$ of group knowledge and belief for \emph{arbitrary subgroups} $I\subseteq A$. All the above axioms have sound analogues for the general operators $K_I$ and $B_I$. E.g., the following generalizations of Super-Introspection and Weak Monotonicity hold:
\[
  B_J \phi \to K_I B_J \phi,
  \,\,\,\,\,\,\,\,\,\,\,\,\,\,
  (K_J \phi \wedge B_I \phi) \to K_I \phi,
  \,\,\,\,\,\,\,\,\,\,\,\,\,\,
  (\mbox{for groups $J\subseteq I \subseteq A$})
  .
  \vspace{-1mm}
\]
Similarly, axiom (CBD) can be generalized to
$(\textstyle\bigwedge_{J\subseteq I} K_J\phi_J) \, \rightarrow \, \langle B_I\rangle (\textstyle\bigwedge_{J\subseteq I} \phi_J)$.

But it is not at all clear that the resulting axiomatization is complete! Our proof methods do not seem to work for this extension. On the other hand, we know that \emph{the logic $\bm{KB_I}$ is decidable} (since it can be translated into a fragment of the decidable logic $\bm{{\B[\A]}_{i,A}}$), so there must exist a recursive axiomatization!

\medskip

This leads to our oustanding unsolved problem:

\paragraph{Open Question}
Find a complete proof system for the logic $\bm{KB_I}$.
\smallskip 

The investigation of this intriguing question is left for future work.

\paragraph{Acknowledgements.}
We thank the anonymous referees for their helpful comments.
D.~Gomes is supported by the Swiss National Science Foundation (SNSF) under grant No. 10000440 (Epistemic Group Attitudes).

\bibliographystyle{splncs04}
\bibliography{references}

\providecommand{\noopsort}[1]{}
\begin{thebibliography}{10}
\providecommand{\url}[1]{\texttt{#1}}
\providecommand{\urlprefix}{URL }
\providecommand{\doi}[1]{https://doi.org/#1}

\bibitem{resolution}
{\AA}gotnes, T., W{\'a}ng, Y.N.: Resolving distributed knowledge. Artificial
  Intelligence  \textbf{252},  1--21 (2017). \doi{10.1016/j.artint.2017.07.002}

\bibitem{AGM1985}
Alchourrón, C.E., Gärdenfors, P., Makinson, D.: On the logic of theory
  change: Partial meet contraction and revision functions. Journal of Symbolic
  Logic  \textbf{50}(2),  510--530 (1985). \doi{10.2307/2274239}

\bibitem{BaltagEtAl18:GroupKnowledgeInterrogative}
Baltag, A., Boddy, R., Smets, S.: Group {{Knowledge}} in {{Interrogative
  Epistemology}}. In: {van Ditmarsch}, H., Sandu, G. (eds.) Jaakko {{Hintikka}}
  on {{Knowledge}} and {{Game-Theoretical Semantics}}, pp. 131--164. Springer,
  Cham (2018). \doi{10.1007/978-3-319-62864-6_5}

\bibitem{BDM2008:PhilInfoChapter}
Baltag, A., van Ditmarsch, H.P., Moss, L.S.: Epistemic logic and information
  update. In: Adriaans, P., van Benthem, J. (eds.) Philosophy of Information.
  MIT Press (2008)

\bibitem{saulpaper}
Baltag, A., Bezhanishvili, N., Fern{\'a}ndez~Gonz{\'a}lez, S.: Topological
  evidence logics: Multi-agent setting. In: Özgün, A., Zinova, Y. (eds.)
  Language, Logic, and Computation. pp. 237--257 (2022).
  \doi{10.1007/978-3-030-98479-3_12}

\bibitem{stalnakersknowledge}
Baltag, A., Bezhanishvili, N., Özgün, A., Smets, S.: The topology of belief,
  belief revision and defeasible knowledge. In: Lecture notes in computer
  science, pp. 27--40. Springer (2013). \doi{10.1007/978-3-642-40948-6_3}

\bibitem{baltagevidence}
Baltag, A., Bezhanishvili, N., Özgün, A., Smets, S.: Justified belief and the
  topology of evidence. In: Logic, Language, Information, and Computation: 23rd
  International Workshop, WoLLIC 2016, pp. 83--103. Springer (2016).
  \doi{10.1007/978-3-662-52921-8_6}

\bibitem{baltagbeliefknowledgeevidence}
Baltag, A., Bezhanishvili, N., Özgün, A., Smets, S.: Justified belief,
  knowledge, and the topology of evidence. Synthese  \textbf{200} (2022).
  \doi{10.1007/s11229-022-03967-6}

\bibitem{justifmodels}
Baltag, A., Fiutek, V., Smets, S.: Beliefs and evidence in justification
  models. In: Beklemishev, L., Demri, S., Máté, A. (eds.) Advances in Modal
  Logic. pp. 156--176 (2016),
  \url{http://www.aiml.net/volumes/volume11/Baltag-Fiutek-Smets.pdf}

\bibitem{SEP}
Baltag, A., Renne, B.: {Dynamic Epistemic Logic}. In: Zalta, E.N. (ed.) The
  {Stanford} Encyclopedia of Philosophy. Metaphysics Research Lab, Stanford
  University, {W}inter 2016 edn. (2016),
  \url{https://plato.stanford.edu/archives/win2016/entries/dynamic-epistemic}

\bibitem{defknowledgebaltag}
Baltag, A., Renne, B., Smets, S.: The logic of justified belief change, soft
  evidence and defeasible knowledge. In: Ong, L., de~Queiroz, R. (eds.) Logic,
  Language, Information and Computation. vol.~7456, pp. 168--190 (2012).
  \doi{10.1007/978-3-642-32621-9_13}

\bibitem{subgroups}
Baltag, A., Smets, S.: Learning what others know. In: Kovacs, L., E., A. (eds.)
  LPAR23 proceedings of the International Conference on Logic for Programming
  AI and Reasoning. vol.~73, pp. 90--110 (2020).
  \doi{10.48550/arXiv.2109.07255}

\bibitem{handbook5}
van Benthem, J., Bezhanishvili, G.: Modal Logics of Space, chap.~5, pp.
  217--298. Springer-Verlag (2007). \doi{10.1007/978-1-4020-5587-4}

\bibitem{benthempacuit}
van Benthem, J., Pacuit, E.: Dynamic logics of evidence-based beliefs. Studia
  Logica  \textbf{99}(1-3),  61--92 (2011). \doi{10.1007/s11225-011-9347-x}

\bibitem{blackbrijkeven}
Blackburn, P., de~Rijke, M., Venema, Y.: Modal Logic. No.~53 in Cambridge
  Tracts in Theoretical Computer Science, Cambridge University Press (2001).
  \doi{10.1017/CBO9781107050884}

\bibitem{Hans-DELbook}
van Ditmarsch, H., van~der Hoek, W., Kooi, B.: Dynamic Epistemic Logic.
  Springer-Verlag (2007)

\bibitem{saulmaster}
Fernández~González, S.: Generic Models for Topological Evidence Logics.
  Master's thesis, University of Amsterdam (2018),
  \url{https://eprints.illc.uva.nl/id/eprint/1641/}

\bibitem{Goldbach15:ModellingDemocraticDeliberation}
Goldbach, R.: Modelling {{Democratic Deliberation}}. Master thesis, University
  of Amsterdam (2015), \url{https://eprints.illc.uva.nl/id/eprint/946/}

\bibitem{dosSantosGomes2025:thesisVGK}
{\noopsort{Gomes}dos Santos Gomes}, D.: Virtual group knowledge on topological
  evidence models. Master thesis, University of Amsterdam (2025),
  \url{https://eprints.illc.uva.nl/id/eprint/2356/}

\bibitem{ogdistrkn}
Halpern, J.Y., Moses, Y.: Knowledge and common knowledge in a distributed
  environment. J. ACM  \textbf{37}(3),  549–--587 (1990).
  \doi{10.1145/79147.79161}

\bibitem{halpernmoses}
Halpern, J.Y., Moses, Y.: A guide to completeness and complexity for modal
  logics of knowledge and belief. Artificial Intelligence  \textbf{54}(3),
  319--379 (1992). \doi{10.1016/0004-3702(92)90049-4}

\bibitem{hintikka}
Hintikka, J.: Knowledge and Belief. Cornell University Press, Ithaca, N.Y.
  (1962)

\bibitem{mckinseytarski}
McKinsey, J.C.C., Tarski, A.: The algebra of topology. The Annals of
  Mathematics  \textbf{45}(1),  141--191 (1944). \doi{10.2307/1969080}

\bibitem{aldomaster}
Ramírez~Abarca, A.I.: Topological Models for Group Knowledge and Belief.
  Master's thesis, University of Amsterdam (2015),
  \url{https://eprints.illc.uva.nl/id/eprint/2250/}

\bibitem{argbelief}
Shi, C., Smets, S., Vel{\'a}zquez-Quesada, F.R.: Argument-based belief in
  topological structures. Electronic Proceedings in Theoretical Computer
  Science  \textbf{251},  489--503 (2017). \doi{10.4204/EPTCS.251.36}

\bibitem{stalnaker}
Stalnaker, R.: On logics of knowledge and belief. Philosophical Studies: An
  International Journal for Philosophy in the Analytic Tradition
  \textbf{128}(1),  169--199 (2006). \doi{10.1007/s11098-005-4062-y}

\bibitem{aybukephd}
Özgün, A.: Evidence in Epistemic Logic: a Topological Perspective. Ph.D.
  thesis, University of Amsterdam (2017),
  \url{https://eprints.illc.uva.nl/id/eprint/2147/}

\end{thebibliography}

\extonly{
\clearpage

\appendix

\section{APPENDIX\@: Proofs of completeness and decidability}

This section contains the proofs of Theorems~\ref{corr:compness-boxall-frag},~\ref{corr:compness-kb}, and~\ref{compness-boxall-dyn}. For Theorems~\ref{corr:compness-boxall-frag} and~\ref{corr:compness-kb} (concerning the static languages), we prove completeness via pseudo-models and relational evidence models, instead of directly for topo-e-models. That is, for each proof system, we first prove the claim for a class of structures that is tailored to the respective language. Next, we define correspondences between these structures and relational evidence models. A complete overview of correspondences used in these proofs is depicted in \Cref{compness-flowchart}. 
\begin{figure}[H]

\begin{center}
\begin{tikzpicture}[
    node distance=1.7cm and 3.8cm,
    every node/.style={align=center},
    every path/.style={draw, -latex},
    >=latex,
]

\node (TopoE) {(Alexandroff) \\ Topo-E-Model \\ (Def. \ref{multiag-topoemodel})};
\node (RelEv) [right=of TopoE] {Relational Evidence Model \\ (Def. \ref{def:rela-ev-model})};
\node (StdPM) [below=of RelEv] {Standard Pseudo-Model for $\mathcal{L}_{\B[\A]_I}$ \\ (Def. \ref{standard-ps})};
\node (AssocM) [below=of StdPM] {Associated Model \\ (Def. \ref{def:assoc-model})};
\node (PseudoI) [left=of AssocM] {Pseudo-Model for $\mathcal{L}_{\B[\A]_I}$ \\ (Def. \ref{pseudo-model})};
\node (PseudoIA) [below=of AssocM] {Pseudo-Model for $\mathcal{L}_{\B[\A]_{i,A}}$ \\ (Def. \ref{pseudo-model})};
\node (PseudoKB) [below=of PseudoI] {Pseudo-Model for $\mathcal{L}_{KB_{i,A}}$ \\ (Def. \ref{kb-pseudo})};

\path (TopoE) edge[<->, below]
    node[midway, below] {Equiv. w.r.t. $\mathcal{L}_{\B[\A]_I}$ \\ (Prop. \ref{cor:rel-sem-equiv-concl_OLD})}
    (RelEv);
\path (RelEv) edge[<->, below] node[right] {(\Cref{fact-rel-standard_OLD})} (StdPM);
\path (StdPM) edge[<-, below] node[right] {(\Cref{assoc-as-standardps_OLD})} (AssocM);
\path (PseudoI) edge[above] 
    node[right, below] {Equiv. w.r.t. $\mathcal{L}_{\B[\A]_I}$ \\ (Cor. \ref{coroll-bisim_OLD})} 
    (AssocM);
\path (PseudoIA) edge[->] 
    node[right] {Bisim. w.r.t. $\mathcal{L}_{\B[\A]_{i,A}}$ \\ (Cor. \ref{coroll-bisim-frag_OLD})} 
    (AssocM);
\path (PseudoIA) edge[<->, below] 
    node[midway, below] {Equiv. w.r.t. $\mathcal{L}_{KB_{i,A}}$ \\ (Cor. \ref{comp-sim_OLD})}
    (PseudoKB);
\end{tikzpicture}
\caption{Flowchart of the correspondences we prove. An arrow from $X$ to $Y$ signifies a map from models of type $X$ to models of type $Y$.
Associated models are standard pseudo-models; however, not every standard pseudo-model is an associated model.}
\label{compness-flowchart}
\end{center}
\end{figure}
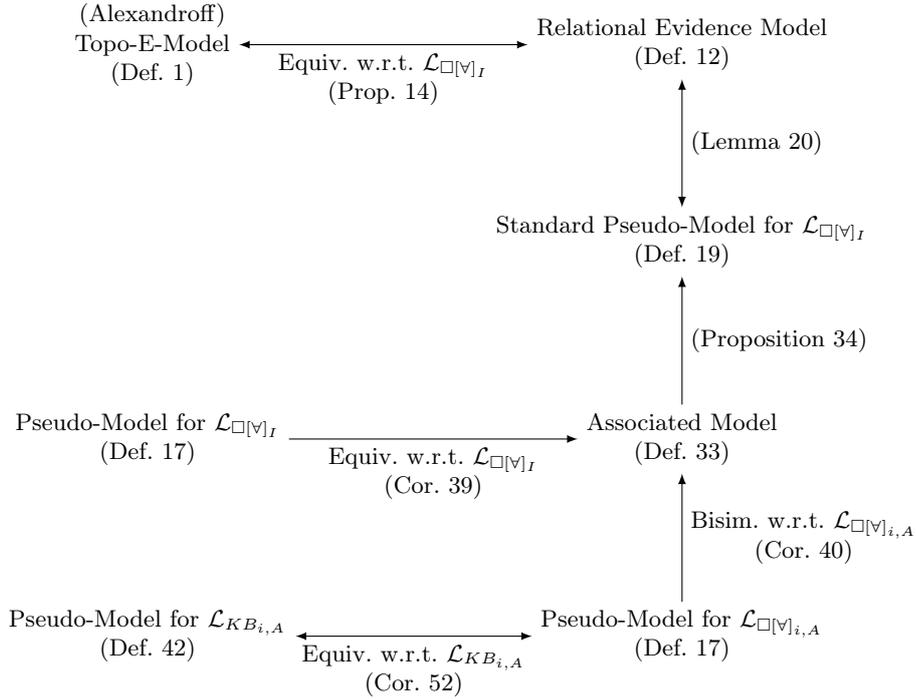

\subsection{Relational Semantics for Alexandroff topo-e-models}\label{pf-alexandroff-evidence-mod}

We first focus at an important special case of topo-e-models: the ones whose underlying topologies are \emph{Alexandroff}:

\begin{definition}[Alexandroff Topo-E-Model]\label{def:alexandroff-multi}
    A multi-agent topo-e-model $\mathfrak{M}={(X, \Pi_i, \tau_i,\llbracket\cdot\rrbracket)}_{i\in A}$ is \emph{Alexandroff} if for all $i\in A$, $\tau_i$ is closed under arbitrary intersections, i.e., $\bigcap \mathcal{C} \in \tau_i$ for any $\mathcal{C} \subseteq \tau_i$. 
\end{definition}

It turns out that  Alexandroff topo-e-models can be given an alternative \textit{relational} representation:

\begin{definition}[Relational Evidence Model]\label{def:rela-ev-model}
    A \emph{relational evidence model} is a structure  $\mathbf{X}={(X,\leq_i,\sim_i,\llbracket\cdot\rrbracket)}_{i\in A}$, where: $X$ is a set of states; for each agent $i\in A$, the relation $\leq_i\ \subseteq X\times X$ is a preorder and $\sim_i\ \subseteq X\times X$ is an equivalence relation, satisfying $\leq_i\ \subseteq\ \sim_i$; and $\llbracket\cdot\rrbracket: \prop\to \mathcal{P}(X)$ is a valuation map.

    In a relational evidence model, we \emph{define the group relations as abbreviations}: we put $ \leq_I := \bigcap_{i\in I}\leq_i$ and $\sim_I := \bigcap_{i\in I}\sim_i$, for all groups $I\subseteq A$.
\end{definition}

The semantics of $\mathcal{L}_{{\B[\A]}_I}$ on relational evidence models is as follows. 

\begin{definition}[Relational Semantics of $\mathcal{L}_{{\B[\A]}_I}$]\label{sem-relev} Given a relational evidence model $\mathbf{X}={(X,\leq_i,\sim_i,\llbracket\cdot\rrbracket)}_{i\in A}$ (over a countable vocabulary $\prop$) and a state $x\in X$, we recursively define the satisfaction relation by:
    \[
    \begin{array}{lll}
        (\mathbf{X}, x) \vDash p & \text{ iff } & x \in \llbracket p\rrbracket \\
        (\mathbf{X}, x) \vDash \neg \phi & \text{ iff } &  (\mathbf{X}, x) \not\vDash \phi \\
        (\mathbf{X}, x) \vDash \phi \wedge \psi & \text{ iff } & (\mathbf{X}, x) \vDash \phi \text{ and } (\mathbf{X}, x) \vDash \psi \\
        (\mathbf{X}, x) \vDash \B_I \phi & \text{ iff } & \text{ for all } y\in X \text{ s.t. } x\leq_I y: (\mathbf{X}, y) \vDash \phi \\
        (\mathbf{X}, x) \vDash {[\A]}_I \phi & \text{ iff } & \text{ for all } y\in X \text{ s.t. } x\sim_I y: (\mathbf{X}, y) \vDash \phi
    \end{array}
    \]
    where $p\in \prop$ is any proposition, $I\subseteq A$ is any group, and $\leq_I$ and $\sim_I$ are the abbreviations from \Cref{def:rela-ev-model}. The interpretation map is given by putting $\llbracket \phi\rrbracket =\{x\in X\mid x\models \phi \}$.
\end{definition}

We conclude by stating the correspondence, which we will use in the proofs of Theorems~\ref{corr:compness-boxall-frag} and~\ref{corr:compness-kb}. 

\begin{proposition}\label{cor:rel-sem-equiv-concl_OLD}
    For every Alexandroff multi-agent topo-e-models there exists a  $\mathcal{L}_{{\B[\A]}_I}$-equivalent relational evidence model, and vice versa. Hence, the $\mathcal{L}_{{\B[\A]}_I}$-logic of Alexandroff topo-e-models is the same as the  $\mathcal{L}_{{\B[\A]}_I}$-logic of relational evidence models.
\end{proposition}
\begin{proof}
We will define maps $\Rel(\cdot)$ (\Cref{thm:top-to-rel-pres-truth_OLD}) and $\Top(\cdot)$ (\Cref{thm:rel-to-top-pres-truth_OLD}) between relational evidence models and Alexandroff topo-e-models, which preserve truth w.r.t.\ formulas over the language $\mathcal{L}_{{\B[\A]}_I}$. The claim then follows immediately from Lemmas~\ref{thm:top-to-rel-pres-truth_OLD} and~\ref{thm:rel-to-top-pres-truth_OLD} below.
\end{proof}

\begin{lemma}\label{thm:top-to-rel-pres-truth_OLD}
    Every Alexandroff multi-agent topo-e-model $\mathfrak{M}={(X, \Pi_i, \tau_i,\llbracket\cdot\rrbracket)}_{i\in A}$ is $\mathcal{L}_{{\B[\A]}_I}$-equivalent to a relational evidence model $\Rel(\mfM)={(X,\leq_i,\sim_i,\llbracket\cdot\rrbracket)}_{i\in A}$ with the same set of states $X$ and same valuation; i.e.,
   for every formula $\phi\in \mathcal{L}_{{\B[\A]}_I}(\prop)$, we have  
    \[
    \begin{array}{lll}
        (\mathfrak{M}, x)\vDash \phi \quad \text{ iff } \quad (\Rel(\mfM), x) \vDash \phi.
    \end{array}
    \]
\end{lemma}
\begin{proof}
    We construct a truth-preserving map $\mfM \mapsto \Rel(\mfM)$. Given $\mfM$, we define for each $i\in A$ and any $x,y\in X$:
    \begin{enumerate}
        \item $x\leq_i y$ if and only if $\Pi_i(x)=\Pi_i(y)$ and $x\sqsubseteq_{\tau_i}y$, where $\sqsubseteq_{\tau_i}$ is the specialization pre-order\footnotemark{} for $\tau_i$; 
        \item $x\sim_i y$ if and only if $\Pi_i(x)=\Pi_i(y)$. 
    \end{enumerate}\footnotetext{The \emph{specialization pre-order} $\sqsubseteq_\tau$ on a topological space $(X,\tau)$ is defined as $x \sqsubseteq_\tau y$ iff $x \in Cl_\tau(\{y\})$ iff $(\A U \in \tau)(x \in U$ implies $y \in U)$~\cite{handbook5}.}
    To show that $\Rel(\mfM)={(X,\leq_i,\sim_i,\llbracket\cdot\rrbracket)}_{i\in A}$ is a relational evidence model, we check the conditions from \Cref{def:rela-ev-model}. Let $i\in A$. 
    
    First, the relation $\leq_i$ is a pre-order. For reflexivity, we have $x\leq_i x$ if and only if $\Pi_i(x)=\Pi_i(x)$; $x \sqsubseteq_{\tau_i} x$ follows from the definition of the specialization pre-order. To see that $\leq_i$ is transitive, let $x\leq_i y\leq_i z$, i.e., let $\Pi_i(x)=\Pi_i(y)=\Pi_i(z)$ and $x \sqsubseteq_{\tau_i} y \sqsubseteq_{\tau_i} z$. Then $\Pi_i(x)=\Pi_i(z)$ and $x \sqsubseteq_{\tau_i} z$, so $x\leq_i z$.

    Second, the relation $\sim_i$ is an equivalence relation. This follows directly from the definition of $\Rel(\cdot)$ and the properties of a partition.

    Finally, inclusion is satisfied, i.e., $\leq_i\ \subseteq\ \sim_i$: suppose $x\leq_i y$. Then by definition of $\Rel(\cdot)$, $\Pi_i(x)=\Pi_i(y)$ and therefore, $x\sim_i y$, as required. Thus, $\Rel(\mfM)$ is indeed a relational evidence model.

    We now prove the modal equivalence claim by induction on the complexity of $\phi$. Let $\mathfrak{M}={(X, \Pi_i, \tau_i,\llbracket\cdot\rrbracket)}_{i\in A}$ be a topo-e-model and consider $\Rel(\mfM)={(X,\leq_i,\sim_i,\llbracket\cdot\rrbracket)}_{i\in A}$. The base case of atomic propositions, and the boolean cases of the induction step, are standard. We only show the proof of the modality $\B_I$; the proof of ${[\A]}_I$ is similar and less complicated, as the definition of the $\sim_I$ relations only concerns the partition, whereas the $\leq_I$ relations additionally involve the specialization pre-order.

    For the case where $\phi=\B_I\psi$, suppose for the left-to-right direction that $(\mathfrak{M}, x) \vDash\B_I\psi$. Then by \Cref{def:our-full-language}, $x\in Int_I(\llbracket\psi\rrbracket)$, i.e., there is $U=\bigcap_{i\in I}U_i$ with $U_i\in \tau_i$ for all $i\in I$, such that $x\in U\subseteq \llbracket\psi\rrbracket$. Now suppose for contradiction that $(\Rel(\mfM), x)\nvDash \B_I\psi$, i.e., suppose there is $y\in X$ such that $x\leq_I y$ but $(\Rel(\mfM), y)\nvDash \psi$ (by \Cref{sem-relev}). Then by the induction hypothesis, $(\mathfrak{M}, y)\nvDash\psi$. By definition of $\leq_I$, we have for all $i\in I$ that $x\leq_i y$ and thus, $x\sqsubseteq_{\tau_i}y$. So by definition of the specialization pre-order we have for all $i\in I$, for all $V\in\tau_i$, that $x\in V$ implies $y\in V$. In particular, this means that $y\in \bigcap_{i\in I}U_i=U\subseteq \llbracket\psi\rrbracket$ and therefore $(\mathfrak{M}, y)\vDash\psi$, giving us the desired contradiction. We conclude that $(\mathfrak{M}, x)\vDash \B_I\psi$.

    For the converse direction, suppose that $(\Rel(\mfM), x)\vDash \B_I\psi$. Then we have for all $y\in X$ such that $x\leq_I y$, $(\Rel(\mfM), y)\vDash\psi$ (\Cref{sem-relev}). Let such $y$ be arbitrary. By the induction hypothesis, $(\mathfrak{M}, y) \vDash \psi$. Furthermore, by definition of $\leq_I$, we have that $x\leq_i y$ for all $i\in I$ and so, by definition of $\leq_i$, that $x\sqsubseteq_{\tau_i} y$ for all $i\in I$. Let such $i\in I$ be arbitrary. By definition of the specialization pre-order, for all $U\in \tau_i$ and for all $y'\in X$ such that $x\leq_i y'$, $x\in U$ implies $y'\in U$. So the intersection of the set $\tau_i^*(x)$ of all open neighbourhoods of $x$ in $\Pi_i(x)$ must be a subset of $\llbracket\psi\rrbracket$:
    \begin{align*}
        \tau_i^*(x) &= \bigcap\{U\in\tau_i \mid x\in U\}\cap \Pi_i(x) \\
        &= \{y\in \Pi_i(x)\mid \A U\in \tau_i(x\in U\Rightarrow y\in U)\} & \\
        &= \{y\in X\mid x\leq_i y\} & \text{ (Def. $\Rel(\cdot)$)}\\
        &\subseteq \llbracket\psi\rrbracket. &
    \end{align*}
    Furthermore, because $\tau_i$ is, by assumption, Alexandroff, $\bigcap\{U\in\tau_i \mid x\in U\}\in \tau_i$. For $i\in I$, let $U_i:=\bigcap\{U\in\tau_i \mid x\in U\}$. Then the set $\bigcap_{i\in I}U_i$ is open in the join topology $\tau_I$. Furthermore, $(\bigcap_{i\in I}U_i)\subseteq \llbracket\psi\rrbracket$. Because $x\in U_i$ for all $i\in I$, we have $x\in (\bigcap_{i\in I}U_i)$. But this gives us that $(\Rel(\mfM), x) \vDash\B_I\psi$, as required.
\end{proof}

\begin{lemma}\label{thm:rel-to-top-pres-truth_OLD}
    Every relational evidence model
    $\mathbf{X}={(X,\leq_i,\sim_i,\llbracket\cdot\rrbracket)}_{i\in A}$
    is $\mathcal{L}_{{\B[\A]}_I}$-equivalent to an Alexandroff multi-agent topo-e-model $\Top(\mathbf{X})={(X, \Pi_i, \tau_i, \llbracket\cdot\rrbracket)}_{i\in A}$ with the same set of states $X$ and same valuation; i.e.
    for every formula $\phi\in \mathcal{L}_{{\B[\A]}_I}(\prop)$ and for every state $x$ of $\mathbf{X}$, we have  
    \[
        (\mathbf{X}, x)\vDash \phi \quad \text{ iff } \quad (\Top(\mathbf{X}), x) \vDash \phi.
    \] 
\end{lemma}
\begin{proof}
    We construct a truth-preserving map $\mathbf{X}\mapsto \Top(\mathbf{X})$. Given $\mathbf{X}$, we define for each $i\in A$:
    \begin{enumerate}
        \item $\tau_i$ as the topology generated by $\mE_i^0:=\{\uparrow_{\leq_i}\!x\mid x\in X\}$, where $\uparrow_{\leq_i}\!x$ is the up-set of the singleton set $\{x\}$ with respect to the relation $\leq_i$;
        \item $\Pi_i:=S/\sim_i$, i.e., let $\Pi_i$ be given by the quotient space of $S$ by $\sim_i$. 
    \end{enumerate}
    To show that $\Top(\mathbf{X})$ is an Alexandroff topo-e-model, we check that each $\tau_i$ and $\Pi_i$ satisfy the conditions from \Cref{multiag-topoemodel}; and that each $\tau_i$ is Alexandroff. Let $i\in A$. It is clear that by construction, $\Pi_i$ is a partition of $X$ and $\tau_i$ is a topology on $X$; furthermore, we have that $\Pi_i\subseteq \tau_i$, due to the property of $\mathbf{X}$ that $\leq_i\ \subseteq\ \sim_i$ (\Cref{def:rela-ev-model}). To see that $\tau_i$ is Alexandroff, it suffices to show that every element of the space has a least open neighbourhood~\cite{handbook5}. In this case, the least open neighbourhood of every $x\in X$ is given by $\uparrow_{\leq_i}\!x$.

    We now prove the modal equivalence claim by induction on the complexity of $\phi$. Let $\mathbf{X}={(X,\leq_i,\sim_i,\llbracket\cdot\rrbracket)}_{i\in A}$ be a relational evidence model and consider $\Top(\mathbf{X})={(X, \Pi_i, \tau_i,\llbracket\cdot\rrbracket)}_{i\in A}$. The base case of atomic propositions and the boolean cases of the induction step are standard. So we focus on the cases involving modalities $\B_I$ and ${[\A]}_I$.

    For the case where $\phi=\B_I\psi$, suppose for the left-to-right direction that $(\mathbf{X}, x) \vDash\B_I\psi$. Then, for all $y\in X$ such that $x\leq_I y$, we have $(\mathbf{X}, y) \vDash\psi$ (\Cref{sem-relev}). By the induction hypothesis, $(\Top(\mathbf{X}), y) \vDash\psi$. So $\uparrow_{\leq_I}\!x =\{y\in X\mid x\leq_I y\}\subseteq \llbracket\psi\rrbracket$. The following equivalences show that $\uparrow_{\leq_I}\!x$ is equivalent to $\bigcap_{i\in I}\uparrow_{\leq_i}\!x$:
    \begin{align*}
        \uparrow_{\leq_I}\!x &= \{y\in X\mid x\leq_i y \text{ for all } i\in I\} \\
        &= \bigcap_{i\in I}\{y\in X\mid x\leq_i y\} \\
        &= \bigcap_{i\in I}\uparrow_{\leq_i}\!x.
    \end{align*}
    But $\bigcap_{i\in I}\uparrow_{\leq_i}\!x$ is open in the join topology $\tau_I$. Thus, with $x\in \left (\bigcap_{i\in I}\uparrow_{\leq_i}\!x\right )\subseteq\llbracket\psi\rrbracket$, we can conclude that $(\Top(\mathbf{X}), x) \vDash\B_I\psi$ (\Cref{def:our-full-language}).

    For the converse direction, suppose that $(\Top(\mathbf{X}), x) \vDash\B_I\psi$. Then for each $i\in I$ there is $U_i\in \tau_i$ such that $\bigcap_{i\in I}U_i = U\in \tau_I$ and $x\in U\subseteq \llbracket\psi\rrbracket$ (\Cref{def:our-full-language}). Now let $y\in X$ be arbitrary and suppose $x\leq_I y$. It remains to show that $(\mathbf{X}, y) \vDash\psi$. By $x\leq_I y$, we have for all $i\in I$ that $x\leq_i y$. So let $i\in I$ be arbitrary. Since $U_i$ is an up-set, we know that $y\in U_i$. Since $i$ was arbitrary, we have $y\in \bigcap_{i\in I}U_i = U$. But then it follows from the fact that $U\subseteq \llbracket\psi\rrbracket$, that $(\Top(\mathbf{X}), y) \vDash\psi$. By the induction hypothesis, $(\mathbf{X}, y) \vDash\psi$. Therefore, $(\mathbf{X}, x) \vDash\B_I\psi$ (\Cref{sem-relev}).

    For the case where $\phi = {[\A]}_I\psi$, suppose for the left-to-right direction that $(\mathbf{X}, x) \vDash {[\A]}_I\psi$. Then for all $y\sim_I x$ we have $(\mathbf{X}, y) \vDash\psi$ (\Cref{sem-relev}). Now let $y\in \Pi_I(x)$ be arbitrary. By definition of $\Top(\cdot)$ we have $y\sim_I x$, so automatically, by $(\mathbf{X}, y) \vDash\psi$ and the induction hypothesis, $(\Top(\mathbf{X}), y) \vDash\psi$. But then $\Pi_I(x)\subseteq \llbracket\psi\rrbracket$, which gives us $(\Top(\mathbf{X}), x)\vDash {[\A]}_I\psi$ (\Cref{def:our-full-language}).

    For the converse direction, suppose that $(\Top(\mathbf{X}), x)\vDash {[\A]}_I\psi$. Then $\Pi_I(x)\subseteq \llbracket\psi\rrbracket$ (\Cref{def:our-full-language}). So let $y\in X$ be arbitrary and suppose $x\sim_I y$. It suffices to show that $(\mathbf{X}, y) \vDash\psi$. But this follows directly from $y$ being in $\Pi_I(x)$, by definition of $\Top(\cdot)$, which gives us that $(\Top(\mathbf{X}), y)\vDash \psi$. By the induction hypothesis, $(\mathbf{X}, y) \vDash\psi$ and therefore, $(\mathbf{X}, x) \vDash {[\A]}_I\psi$.
\end{proof}


\subsection{Proof of Completeness and Decidability for the Logic of Group Evidence (\texorpdfstring{\Cref{corr:compness-boxall-frag}}{Theorem 6})}

Completeness for $\bm{{\B[\A]}_I}$ (resp. $\bm{{\B[\A]}_{i,A}}$) is proved by the chain of correspondences from pseudo-models for $\mathcal{L}_{{\B[\A]}_I}$ (resp. $\mathcal{L}_{{\B[\A]}_{i,A}}$), to associated models, to topo-e-models. The completeness proof for $\bm{{\B[\A]}_I}$ is similar to existing completeness proofs for logics that incorporate distributed knowledge. In particular, our proof closely resembles the proof in Appendix A of~\cite{subgroups}, which proves completeness of a logic incorporating, among other notions, distributed knowledge for all subgroups. Throughout the proof for $\bm{{\B[\A]}_I}$, which we discuss in detail, we explain how it can be adapted to $\bm{{\B[\A]}_{i,A}}$. 

Throughout the proof, fix a finite set of agents $A$ and a finite vocabulary $\prop$. 

\subsubsection{Soundness and Completeness of \texorpdfstring{$\mathcal{L}_{{\B[\A]}_I}$}{BoxForallI} w.r.t. Pseudo-Models.}\label{sec:comp-pseudo}
We first prove completeness with respect to non-standard models, which we call pseudo-models, for $\mathcal{L}_{{\B[\A]}_I}$ and for $\mathcal{L}_{{\B[\A]}_{i,A}}$. 

\begin{definition}[Pseudo-Model for $\mathcal{L}_{{\B[\A]}_I}$ and for $\mathcal{L}_{{\B[\A]}_{i,A}}$]\label{pseudo-model}
    A \emph{pseudo-model for $\mathcal{L}_{{\B[\A]}_I}$} is a structure $\mathbf{S}={(S,\leq_I,\sim_I,\llbracket\cdot\rrbracket)}_{ I\subseteq A}$,  where $S$ is a set of states; for each group $I\subseteq A$, the relation $\leq_I\ \subseteq S\times S$ is a preorder and $\sim_I\ \subseteq S\times S$ is an equivalence relation; $\llbracket\cdot\rrbracket: \prop \to \mathcal{P}(X)$ is a valuation map; and relations are required to satisfy the following two conditions:
    \begin{enumerate}
        \item \textbf{Anti-Monotonicity}. For all groups $I\subseteq A$, and $s,t\in X$: 
        \begin{itemize}
            \item If $s\leq_I t$ and $I \supseteq J\neq \emptyset$, then $s\leq_J t$;
            \item If $s\sim_I t$ and $I \supseteq J\neq \emptyset$, then $s\sim_J t$.
        \end{itemize}
        \item \textbf{Inclusion}. For all groups $I\subseteq A$: $\leq_I\ \subseteq\ \sim_I$.
    \end{enumerate}
    Pseudo-models for the fragment $\mathcal{L}_{{\B[\A]}_{i,A}}$ are obtained by imposing the same conditions, with the relations restricted to $\leq_\alpha$ and $\sim_\alpha$, with $\alpha \in \{A\}\cup A$. 
\end{definition}

We define the following semantics.

\begin{definition}[Pseudo-Model Semantics of $\mathcal{L}_{{\B[\A]}_I}$ and $\mathcal{L}_{{\B[\A]}_{i,A}}$]\label{sem-boxall-pseudo}
    Given a pseudo-model $\mathbf{S}$ for $\mathcal{L}_{{\B[\A]}_I}$ and a state $s$ of $\mathbf{S}$, we recursively define
    \[
    \begin{array}{lll}
        (\mathbf{S}, s) \vDash p & \text{ iff } & s \in \llbracket p\rrbracket\\
        (\mathbf{S}, s) \vDash \neg \phi & \text{ iff } &  (\mathbf{S}, s) \not\vDash \phi \\
        (\mathbf{S}, s) \vDash \phi \wedge \psi & \text{ iff } & (\mathbf{S}, s) \vDash \phi \text{ and } (\mathbf{S}, s) \vDash \psi \\
        (\mathbf{S}, s) \vDash \B_I \phi & \text{ iff } & \text{ for all } t\in S \text{ s.t. } s\leq_I t: (\mathbf{S}, t) \vDash \phi \\
        (\mathbf{S}, s) \vDash {[\A]}_I \phi & \text{ iff } & \text{ for all } t\in S\text{ s.t. } s\sim_I t: (\mathbf{S}, t) \vDash \phi
    \end{array}
    \]
    where $p\in \prop$ is any proposition and $I\subseteq A$ is any group. The semantics for $\mathcal{L}_{{\B[\A]}_{i,A}}$ is obtained by restricting the above definition to this language. 
\end{definition}

In fact, \emph{standard} pseudo-models can be represented as relational evidence models (and vice versa): 

\begin{definition}[Standard Pseudo-Model for $\mathcal{L}_{{\B[\A]}_I}$ and $\mathcal{L}_{{\B[\A]}_{i,A}}$]\label{standard-ps}
    A pseudo-model for $\mathcal{L}_{{\B[\A]}_I}$  is \emph{standard} if it also satisfies the following condition:\footnote{For one direction, the intersection condition reduces to anti-monotonicity: let $s,t\in S$ and let $I,J\subseteq A$ be nonempty. Then, if $s\sim_{I\cup J} t$, we have by $I\subseteq I\cup J$ that $s\sim_{I} t$; analogously, with $J\subseteq I\cup J$, we have $s\sim_{J} t$.}
    \begin{enumerate}
        \item[3.] \textbf{Intersection}. For all groups $I,J\subseteq A$:
        \begin{itemize}
            \item $\leq_{I\cup J}$ is the intersection of $\leq_I$ and $\leq_J$; 
            \item $\sim_{I\cup J}$ is the intersection of $\sim_I$ and $\sim_J$. 
        \end{itemize}
    \end{enumerate}
    Restricting this definition, a pseudo-model for $\mathcal{L}_{{\B[\A]}_{i,A}}$ is standard if $\leq_A=\bigcap_{i\in A}\leq_i$ and $\sim_A=\bigcap_{i\in A}\sim_i$.
\end{definition}

\begin{lemma}\label{fact-rel-standard_OLD}
    For every relational evidence model there exists a modally equivalent standard pseudo-model, and vice versa.
\end{lemma}
\begin{proof}
    We can represent a relational evidence model $\mathbf{X}$ as a standard pseudo-model $\mathbf{S}$ for ${\mathcal{L}_{{\B[\A]}_I}}$ by setting $\leq_I\ := \bigcap_{i\in I}\leq_i$ and setting $\sim_I\ := \bigcap_{i\in I}\sim_i$. Conversely, we represent a standard pseudo-model $\mathbf{S}$ as a relational evidence model $\mathbf{X}$ by setting $\leq_i\ :=\ \leq_{\{i\}}$ and $\sim_i\ :=\ \sim_{\{i\}}$.

    The interpretation of any formula $\phi\in {\mathcal{L}_{{\B[\A]}_I}}$ on the relational evidence model $\mathbf{X}$ (according to \Cref{sem-relev}) agrees with the interpretation of $\phi$ on the standard pseudo-model $\mathbf{S}$ (according to \Cref{sem-boxall-pseudo}), because the abbreviations $\leq_I := \bigcap_{i\in I}\leq_i$ and $\sim_I := \bigcap_{i\in I}\sim_i$ on $\mathbf{X}$ coincide with the directly defined group relations $\leq_I$ and $\sim_I$ on $\mathbf{S}$.
\end{proof}

Thus, in order to prove completeness with respect to relational evidence models, it suffices to prove the claim with respect to standard pseudo-models for $\mathcal{L}_{{\B[\A]}_I}$. Before showing this, we first prove soundness and completeness with respect to \emph{general} pseudo-models for $\mathcal{L}_{{\B[\A]}_I}$. The structure of this proof follows the structure of the proof in Appendix A.1 of~\cite{subgroups}. \Cref{sound_OLD} takes care of soundness. 

\begin{proposition}\label{sound_OLD}
    The proof system $\bm{{\B[\A]}_I}$ is sound with respect to pseudo-models for $\mathcal{L}_{{\B[\A]}_I}$, and the proof system $\bm{{\B[\A]}_{i,A}}$ is sound with respect to pseudo-models for $\mathcal{L}_{{\B[\A]}_{i,A}}$. 
\end{proposition}
\begin{proof}
    We omit the proof, as it is a routine check.
\end{proof}

For completeness, fix a consistent formula $\phi_0\in \mathcal{L}_{{\B[\A]}_I}(\prop)$. We show that $\phi_0$ is satisfiable in a finite pseudo-model (namely the \emph{filtrated pseudo-model for $\mathcal{L}_{{\B[\A]}_I}$}), which additionally gives us the finite model property for the logic of $\mathcal{L}_{{\B[\A]}_I}$. The \emph{filtrated pseudo-model for $\mathcal{L}_{{\B[\A]}_I}$} can be thought of as a finite filtration of the usual notion of a canonical model, with respect to $\Phi$ (see~\cite{blackbrijkeven} for details): we identify each set of states in the canonical model that agrees on a finite set of formulas (the \emph{closure} of $\phi_0$). 

\begin{definition}[Closure (${{\B[\A]}_I}$)]\label{def:closure}
    Given a formula $\phi_0\in \mathcal{L}_{{\B[\A]}_I}(\prop)$, the closure $\Phi=\Phi(\phi_0)$ of $\phi$ is the smallest set of formulas over $\mathcal{L}_{{\B[\A]}_I}(\prop)$ satisfying, for all nonempty $J\subsetneq I\subseteq A$, and for all formulas $\psi,\theta\in \mathcal{L}_{{\B[\A]}_I}(\prop)$:
    \begin{enumerate}
        \item $\phi_0\in\Phi$;
        \item If $\psi \in \Phi$ and $\theta$ is a subformula of $\psi$, then $\theta \in \Phi$;
        \item $\Phi$ is closed under single negations\footnote{The single negation $\sim \phi$ is defined as: $\sim \phi:=\theta$ if $\phi$ is of the form $\neg \theta$; and $\sim \phi:=\neg \varphi$ if $\varphi$ is not of the form $\neg \theta$.} 
        $\sim$: if $\psi \in \Phi$, then $(\sim \psi) \in \Phi$;
        \item If ${[\A]}_J\psi\in\Phi$, then ${[\A]}_I\psi\in\Phi$;
         \item If ${[\A]}_I\psi\in\Phi$, then $\B_I{[\A]}_I\psi\in\Phi$;
        \item If $\neg{[\A]}_I\psi\in\Phi$, then $\B_I\neg{[\A]}_I\psi\in\Phi$;
        \item If $\B_J\psi\in\Phi$, then $\B_I\psi\in\Phi$;
        \item If ${[\A]}_I\psi\in\Phi$, then $\B_I\psi\in\Phi$.
    \end{enumerate}
\end{definition}

Now let $\Phi=\Phi(\phi_0)$ be the closure of $\phi_0$. The closure of $\phi_0$ is finite, which will ensure a finite filtrated pseudo-model.

\begin{lemma}\label{fin-closure_OLD}
    Every formula $\phi_0\in \mathcal{L}_{{\B[\A]}_I}(\prop)$ has a finite closure $\Phi(\phi_0)$. 
\end{lemma}
\begin{proof}
    We omit the proof, as it is straightforward.
\end{proof}

We use the closure of $\phi_0$ to define the filtrated pseudo-model $\mathbf{S}^C$, on which $\phi_0$ will be satisfied.

\begin{definition}\label{def:filtr-can-pseudo}
    Fix a maximally consistent theory\footnote{This theory exists by the Lindenbaum Lemma (see e.g.~\cite{blackbrijkeven}) and consistency of $\phi_0$.} $T_0\subseteq \Phi$ with our fixed formula $\phi_0\in T_0$. The \emph{filtrated pseudo-model for $\mathcal{L}_{{\B[\A]}_I}$} for $\phi_0$ is the finite structure $\mathbf{S}^C = {(S^C, \leq_I, \sim_I,\llbracket\cdot\rrbracket)}_{I\subseteq A}$, where $S^C$ is defined as 
    \begin{align*}
        S^C:=\{T\subseteq \Phi \mid T \subseteq \mathcal{L}_{{\B[\A]}_I}(\prop) \text{ and } T \text{ is a maximally consistent subset of }\Phi \}
    \end{align*}
    and for all groups $I\subseteq A$, the relations $\leq_I$ and $\sim_I$ on $S^C$ are given by putting
    \begin{align*}
        &T \sim_I W \quad \text { iff} \quad 
        {[\A]}_J\phi \in T \Leftrightarrow {[\A]}_J\phi \in W \text{ holds for all groups $J\subseteq I$;} \\
        &T \leq_I W \quad \text { iff} \quad  \B_J \phi \in T \Rightarrow \B_J\phi \in W \text{ holds for all groups $J\subseteq I$.}
    \end{align*}
    Finally, we define for all $p\in \prop$:
    \[\llbracket p \rrbracket := \{T \in S^C \mid p \in T\}.
    \]
\end{definition}

Since we ensured that $\Phi$ is finite, the model $\mathbf{S}^C$ is finite: its size is $|S^C|\leq |2^\Phi|$, as the collection of maximally consistent subsets of $\Phi$ is a subset of the powerset of $\Phi$. Furthermore, it can be checked that $\mathbf{S}^C$ is indeed a pseudo-model.

We need the Truth Lemma to prove our claim that $\phi_0$ is satisfied in $\mathbf{S}^C$. 

\begin{lemma}[Truth Lemma]\label{truth}
    Given the filtrated pseudo-model $\mathbf{S}^C$ for $\mathcal{L}_{{\B[\A]}_I}$ over a closure $\Phi$, we have for all $\phi \in \Phi$:
    \begin{align*}
        T \vDash_{\mathbf{S}^C} \phi \text { iff } \phi \in T, \text { for every } T \in S^C.
    \end{align*}
\end{lemma}
\begin{proof}
    The Truth Lemma is a standard lemma in canonical-model constructions (see e.g.~\cite{blackbrijkeven}) and its proof is straightforward. The cases for soft and hard evidence are similar to the case for distributed knowledge in the proof of Lemma 1.2 in Appendix A.1 in~\cite{subgroups}.  
\end{proof}

\begin{corollary}\label{comp-pseudo_OLD}
    The proof system $\bm{{\B[\A]}_I}$ (displayed in \Cref{pf-syst}) is sound and weakly complete with respect to pseudo-models for $\mathcal{L}_{{\B[\A]}_I}$, and the logic of $\mathcal{L}_{{\B[\A]}_I}$ is decidable. All properties are inherited by the proof system $\bm{{\B[\A]}_{i,A}}$ and the logic $\mathcal{L}_{{\B[\A]}_{i,A}}$.
\end{corollary}
\begin{proof}
    Soundness of $\bm{{\B[\A]}_I}$ was established in \Cref{sound_OLD}. For completeness, let $\phi_0\in \mathcal{L}_{{\B[\A]}_I}(\prop)$ be any consistent formula and construct the filtrated pseudo-model $\mathbf{S}^C$ for $\mathcal{L}_{{\B[\A]}_I}(\prop)$, for $\phi_0$. By the Lindenbaum Lemma, there exists some maximally consistent theory $T_0$ in $\mathbf{S}^C$ with $\phi_0 \in T_0$. By the Truth Lemma (\Cref{truth}), $T_0$ satisfies $\phi_0$ in $\mathbf{S}^C$. Since $\mathbf{S}^C$ is finite, this gives us weak completeness with respect to finite pseudo-models for $\bm{{\B[\A]}_I}$ (and hence also with respect to all pseudo-models).

    Since $\bm{{\B[\A]}_I}$ is weakly complete with respect to finite pseudo-models for the language, the logic ${\mathcal{L}_{{\B[\A]}_I}}$ has the finite pseudo-model property. Therefore, it is decidable: to decide $\phi_0\in \mathcal{L}_{{\B[\A]}_I}(\prop)$, let $\Phi:=\Phi(\phi_0)$ be its closure and generate all pseudo-models (up to isomorphism) that are at most of the size $2^{|\Phi|}$. Then model-check $\phi_0$ on these models: if $\phi_0$ is satisfied at any state in any of the models, then it is satisfiable (on pseudo-models for $\mathcal{L}_{{\B[\A]}_I}$); otherwise, it is unsatisfiable.

    For $\bm{{\B[\A]}_{i,A}}$, the proof is obtained simply by restricting to $\mathcal{L}_{{\B[\A]}_{i,A}}$ all the constructions in the proof for $\mathcal{L}_{{\B[\A]}_I}$: to construct the filtrated pseudo-model $\mathbf{S}^C$ for $\mathcal{L}_{{\B[\A]}_{i,A}}$, for any consistent formula $\phi_0\in \mathcal{L}_{{\B[\A]}_{i,A}}(\prop)$, restrict the formulas in the closure $\Phi(\phi_0)$ to $\mathcal{L}_{{\B[\A]}_{i,A}}$; and define the filtrated pseudo-model $\mathbf{S}^C$ as a pseudo-model for $\mathcal{L}_{{\B[\A]}_{i,A}}(\prop)$ (that is, restrict the relations from the definition of the canonical pseudo-model for $\mathcal{L}_{{\B[\A]}_I}$ to those labeled by $A$ and $\{i\}$ for all $i\in A$). The rest of the proof goes through exactly as in the proof for $\bm{{\B[\A]}_I}$.
\end{proof}

\subsubsection{From Pseudo-Models to Models.}\label{sec:pseudo-to-model-boxall} To prove completeness with respect to \emph{standard} pseudo-models, we show how to go from a general pseudo-model to a standard pseudo-model satisfying the same formulas: given a pseudo-model $\mathbf{S}$ for $\mathcal{L}_{{\B[\A]}_I}$, we use \emph{model unraveling} to construct an \emph{associated} model $\mathbf{X}$. This will be a relational evidence model, structured as a \emph{tree}, on which we impose the desired properties. The challenge of this proof is to ensure that the relations on the unraveled tree satisfy the intersection condition of a pseudo-model for $\mathcal{L}_{{\B[\A]}_I}$ (\Cref{standard-ps}), such that it is indeed standard.

We define the correspondence with respect to pseudo-models for $\mathcal{L}_{{\B[\A]}_I}$, after which we show how to adapt the proof for $\mathcal{L}_{{\B[\A]}_{i,A}}$. The structure of this proof closely follows the structure of the proof in Appendix A.2 of~\cite{subgroups}. For an introduction into model unraveling for completeness proofs, we refer to~\cite{blackbrijkeven}. 

Throughout this proof, \textit{we fix a pseudo-model} $\mathbf{S}={(S,\leq_I,\sim_I,\llbracket\cdot\rrbracket_{\mathbf{S}})}_{I\subseteq A}$ for $\mathcal{L}_{{\B[\A]}_I}$, and \textit{a designated state} $s_0 \in S$. The state space of the associated model will consist of all $s_0$-originated \emph{histories}:
\begin{definition}[Histories]
    The set $H$ of all \emph{($s_0$-generated) histories} over the pseudo-model $\mathbf{S}$ consists of all finite sequences $h=\left(s_0, R_{{G_1}}, \ldots, R_{{G_n}}, s_n\right)$ satisfying the following conditions:
    \begin{enumerate}
        \item The sequence $h$ has length $n\geq 0$ and we have $s_i \in S$ for all $i\leq n$ (with $s_0$ being the fixed state in the model);
        \item The subgroups ${G_1}, \ldots, {G_n} \subseteq A$ are nonempty;
        \item For each $k\in\{1, \ldots, n\}$, we have one of the following two cases:
        \begin{enumerate}
            \item $R_{I^k}$ refers to $\leq_{I^k}$, and we have $s_{k-1} \leq_{I^k} s_k$
            \item $R_{I^k}$ refers to $\sim_{I^k}$, and we have $s_{k-1} \sim_{I^k} s_k$.
        \end{enumerate}
    \end{enumerate}
\end{definition}
Given a history $h=\left(s_0, R_{{G_1}}, \ldots, R_{{G_n}}, s_n\right)\in H$, we denote by $last(h):=s_n$ the last state in the history.

Next, we construct the relations $\leq_I$ and $\sim_I$ for all groups $I\in A$ in intermediate steps (with the resulting relations being defined in \Cref{def:final-rels}), ensuring in particular that $\leq_I=\bigcap_{i\in I}\leq_i$ and $\sim_I=\bigcap_{i\in I}\sim_i$ (we will show this in \Cref{assoc-as-standardps_OLD}).

\begin{definition}[One-step relations, immediate successor]\label{def:intermediate-rels}
    We first define \emph{one-step relations} $\xrightarrow{\operatorname{P}}_I$ and $\xrightarrow{\operatorname{E}}_I$ on histories in $H$ (labeled by P  for `pre-order' or E for `equivalence' relation, and by groups $I\subseteq A$), by putting:
    \[
    \begin{array}{lll}
        h \xrightarrow{\operatorname{P}}_I h' &\text{ iff}\quad h'=(h, \leq_I, s') &\text{ with }\quad last(h) \leq_I s'=last(h') \\
        h \xrightarrow{\operatorname{E}}_I h' &\text{ iff}\quad h'=(h, \sim_I, s') &\text{ with }\quad last(h) \sim_I s'=last(h').
    \end{array}
    \]
    We also define the \emph{immediate successor relation} $\rightarrow$ on histories as the union of all one-step relations:
    \[
    h \rightarrow h' \quad \text{ iff}\quad h\ (\xrightarrow{\operatorname{P}}_I\cup \xrightarrow{\operatorname{E}}_I)\ h' \text{ for some } I \subseteq A.
    \]
    We close these relations under monotonicity by defining, for all groups $J\subseteq A$, 
    \begin{align*}
        h \xrightarrow{\leq}_J h' \quad &\text{ iff } \quad
        h \xrightarrow{\operatorname{P}}_{I} h' \text{ for some } I\supseteq J \\
        h \xrightarrow{\sim}_J h' \quad &\text{ iff } \quad
        h \xrightarrow{\operatorname{E}}_{I} h' \text{ for some } I \supseteq J.
    \end{align*}
\end{definition}

Note that $H$ has the structure of a tree rooted at $s_0$ (that is, the history given by the sequence ($s_0$)): the immediate successor relation on $H$ has the tree property, i.e., it connects every two nodes $h, h'$ of the tree by a unique non-redundant path~\cite{blackbrijkeven}.

We now define the final relations $\leq_I$ and $\sim_I$, which satisfy the conditions of a relational evidence model. In particular, we obtain individual relations $\leq_i\ :=\ \leq_{\{i\}}$ and $\sim_i\ :=\ \sim_{\{i\}}$.

\begin{definition}[Relations on the Associated Model]\label{def:final-rels}
    Let $I\subseteq A$ be a group and let $\xrightarrow{\leq}_I$ and $\xrightarrow{\sim}_I$ be as in \Cref{def:intermediate-rels}. We define
    \begin{align*}
        \leq_I &:= {\left(\xrightarrow{\leq}_I\right)}^* \\
        \sim_I &:= {\left(\xrightarrow{\leq}_I\cup \xleftarrow{\leq}_I\cup\xrightarrow{\sim}_I\cup \xleftarrow{\sim}_I\right)}^*
    \end{align*}
    where $R^*$ denotes the reflexive-transitive closure of $R$, and $\xleftarrow{\leq}_I$ and $\xleftarrow{\sim}_I$ denote the converses of $\xrightarrow{\leq}_I$ and $\xrightarrow{\sim}_I$, respectively.
\end{definition}

The following lemmas state a number of properties of the relations from \Cref{def:final-rels}, which we will use to show in \Cref{lem:standard-ps} that the relations satisfy the conditions of a standard pseudo-model for $\mathcal{L}_{{\B[\A]}_I}$ and, subsequently, in \Cref{bisim_OLD} when we prove a bisimulation between the associated model and the original pseudo-model for ${\mathcal{L}_{{\B[\A]}_I}}$.

\begin{lemma}\label{leq-equiv-set_OLD}
    For all groups $I\subseteq A$, and histories $h, h' \in H$, the following are equivalent:
    \begin{enumerate}
        \item  $h \leq_I h'$;
        \item the non-redundant path from $h$ to $h'$ consists only of steps of the form $h_{n-1} \xrightarrow{\operatorname{P}}_{{G_n}} h_{n}$, with $I\subseteq {G_n}$.
    \end{enumerate}
\end{lemma}
\begin{proof}    
    Let $I\subseteq A$ be a group, and let $h, h' \in H$. For the left-to-right direction, suppose $h \leq_I h'$. Then, by definition of $\leq_I$ (\Cref{def:final-rels}), we have $h {\left(\xrightarrow{\leq}_I\right)}^* h'$, that is, from $h$ we can reach $h'$ via a finite non-redundant path under the relation $\xrightarrow{\leq}_I$. More importantly, by the properties of a tree-like model, this non-redundant path is unique. The claim now follows immediately from the definition of $\xrightarrow{\leq}_I$ (\Cref{def:intermediate-rels}): each step $h_{n-1} \xrightarrow{\leq}_I h_n$ on the path implies that for some ${G_n}\supseteq I$ we have $h_{n-1} \xrightarrow{\operatorname{P}}_{{G_n}} h_{n}$.

    For the converse direction, the claim is immediate: assuming that the non-redundant path from $h$ to $h'$ consists only of steps of the form $h_{n-1} \xrightarrow{\operatorname{P}}_{{G_n}} h_{n}$, with $I\subseteq {G_n}$, we have for every step $h_{n-1} \xrightarrow{\operatorname{P}}_{{G_n}} h_{n}$ on the path that $h_{n-1} \xrightarrow{\leq}_{I} h_{n}$ (\Cref{def:intermediate-rels}), and thereby, $h {\left(\xrightarrow{\leq}_I\right)}^* h'$ (\Cref{def:final-rels}).
\end{proof}

\Cref{sim-equiv-set_OLD} is the analogue of the previous lemma, for the equivalence relations $\sim_I$.

\begin{lemma}\label{sim-equiv-set_OLD}
    The following are equivalent, for all groups $I \subseteq A$ and histories $h, h' \in H$:
    \begin{enumerate}
        \item $h \sim_I h'$;
        \item each of the steps on the non-redundant path from $h$ to $h'$ is of one of the following forms:
        \begin{multicols}{2}
       \begin{enumerate}
            \item $h_{n-1} \xrightarrow{\operatorname{P}}_{{G_n}} h_{n}$
            \item $h_{n-1} \xleftarrow{\operatorname{P}}_{{G_n}} h_{n}$
            \item $h_{n-1} \xrightarrow{\operatorname{E}}_{{G_n}} h_{n}$
            \item $h_{n-1} \xleftarrow{\operatorname{E}}_{{G_n}} h_{n}$
        \end{enumerate}
    \end{multicols}
    with $I \subseteq {G_n}$.
    \end{enumerate}
\end{lemma}
\begin{proof}
    Let $I\subseteq A$ be a group, and let $h, h' \in H$. For the left-to-right direction, suppose $h \sim_I h'$. Then, by definition of $\sim_I$ (\Cref{def:final-rels}), we have $h {\left(\xrightarrow{\leq}_I\cup \xleftarrow{\leq}_I\cup\xrightarrow{\sim}_I\cup \xleftarrow{\sim}_I\right)}^* h'$, that is, from $h$ we can reach $h'$ via a finite non-redundant path under the relation $\left(\xrightarrow{\leq}_I\cup \xleftarrow{\leq}_I\cup\xrightarrow{\sim}_I\cup \xleftarrow{\sim}_I\right)$. More importantly, by the properties of a tree-like model, this non-redundant path is unique. Consider an arbitrary step $h_{n-1} \left(\xrightarrow{\leq}_I\cup \xleftarrow{\leq}_I\cup\xrightarrow{\sim}_I\cup \xleftarrow{\sim}_I\right) h_n$ on this path. We have one of the following four cases:
    \begin{multicols}{2}
    \begin{enumerate}[(a)]
        \item $h_{n-1} \xrightarrow{\leq}_I h_n$;
        \item $h_{n-1} \xleftarrow{\leq}_I h_n$;
        \item $h_{n-1} \xrightarrow{\sim}_I h_n$;
        \item $h_{n-1} \xleftarrow{\sim}_I h_n$.
    \end{enumerate}
\end{multicols}
    The claim then follows from unfolding the respective definitions of these relations (\Cref{def:intermediate-rels}).

    For the converse direction, the claim is immediate: assuming that the non-redundant path from $h$ to $h'$ consists only of steps of the form (a)-{(d)} as listed in \Cref{sim-equiv-set_OLD}, with $I\subseteq {G_n}$ for each step from $h_{n-1}$ to $h_n$, we can apply the corresponding definitions from \Cref{def:intermediate-rels} to each step, to obtain that 
    \[ 
        h {\left(\xrightarrow{\leq}_I\cup \xleftarrow{\leq}_I\cup\xrightarrow{\sim}_I\cup \xleftarrow{\sim}_I\right)}^* h'
    \]
    i.e., $h \sim_I h'$ (\Cref{def:final-rels}).
\end{proof}

We can now show that the relations from \Cref{def:final-rels} satisfy the requirements of a \emph{standard} pseudo-model for $\mathcal{L}_{{\B[\A]}_I}$. 

\begin{lemma}\label{lem:standard-ps}
    Let $I\subseteq A$ be a group. The relations $\leq_I$ and $\sim_I$ from \Cref{def:final-rels} satisfy the relational conditions of a pseudo-model for $\mathcal{L}_{{\B[\A]}_I}$ (\Cref{pseudo-model}): $\leq_I\ \subseteq\ \sim_I$ (the inclusion condition); $\leq_I$ is a pre-order; and $\sim_I$ is an equivalence relation. Furthermore, for all groups $J\subseteq A$, $\leq_I$ and $\leq_J$ satisfy the anti-monotonicity and intersection conditions, as well as $\sim_I$ and $\sim_J$. 
\end{lemma}
\begin{proof}
    The inclusion condition is satisfied by construction of $\sim_I$: let $h,h'\in H$ and suppose $h\leq_I h'$. Then from $h$, we can reach $h'$ via a unique non-redundant path under the relation $\xrightarrow{\leq}_I$. Since the relation $\xrightarrow{\leq}_I$ is a subset of the relation $\left(\xrightarrow{\leq}_I\cup \xleftarrow{\leq}_I\cup\xrightarrow{\sim}_I\cup \xleftarrow{\sim}_I\right)$, $h$ and $h'$ are automatically connected by the same path, under the relation $\left(\xrightarrow{\leq}_I\cup \xleftarrow{\leq}_I\cup\xrightarrow{\sim}_I\cup \xleftarrow{\sim}_I\right)$. By definition of $\sim_I$, we have $h\sim_I h'$.

    The relation $\leq_I$ is a pre-order by construction: it is the reflexive-transitive closure of $\xrightarrow{\leq}_I$.

    Reflexivity and transitivity of $\sim_I$ are immediate by \Cref{def:final-rels}, since $\sim_I$ is the reflexive-transitive closure of a union of relations. For symmetry, let $h,h'\in H$ and suppose $h\sim_I h'$. Then each of the steps on the non-redundant path form $h$ to $h'$ is of one of the forms listed in \Cref{sim-equiv-set_OLD}. Observe that the converse of each of these steps is also listed, which means that each of the steps on the non-redundant path from $h'$ to $h$ is also of one of the listed forms, i.e., we have $h'\sim_I h$.

    We prove the anti-monotonicity claim only for the $\sim$ relations, since the proof for $\leq$ is similar and less complicated. To see that the $\sim$ relations satisfy the anti-monotonicity condition, let $I,J\subseteq A$ be two groups and let $h,h'\in H$. Suppose that $J\subseteq I$ and $h\sim_I h'$. We claim that $h\sim_J h'$. By $h\sim_I h'$, we know that each of the steps on the non-redundant path form $h$ to $h'$ is of one of the forms listed in \Cref{sim-equiv-set_OLD}. Consider an arbitrary step on this path, from a history $h_{n-1}$ to another history $h_n$. We distinguish the four\footnote{For the proof of anti-monotonicity for $\leq$, case (a) is the only possible case for any step on the path (\Cref{leq-equiv-set_OLD}).} cases from \Cref{sim-equiv-set_OLD}, with ${G_n}$ being an arbitrary superset of $I$:
    \begin{enumerate}
            \item $h_{n-1} \xrightarrow{\operatorname{P}}_{{G_n}} h_{n}$. With $J\subseteq I$, clearly, $J\subseteq {G_n}$. By construction of $\xrightarrow{\leq}_J$ (\Cref{def:intermediate-rels}), we get $h_{n-1} \xrightarrow{\leq}_J h_{n}$.
            \item $h_{n-1} \xleftarrow{\operatorname{P}}_{{G_n}} h_{n}$. This is equivalent to having $h_{n} \xrightarrow{\operatorname{P}}_{{G_n}} h_{n-1}$ and thus, by item (a), we have $h_{n} \xrightarrow{\leq}_J h_{n-1}$, i.e., $h_{n-1} \xleftarrow{\leq}_J h_{n}$.
            \item $h_{n-1} \xrightarrow{\operatorname{E}}_{{G_n}} h_{n}$. Similar to case (a): clearly, $J\subseteq {G_n}$. By construction of $\xrightarrow{\sim}_J$, we get $h_{n-1} \xrightarrow{\sim}_J h_{n}$.
            \item $h_{n-1} \xleftarrow{\operatorname{E}}_{{G_n}} h_{n}$. Similar to (b). We get $h_{n-1} \xleftarrow{\sim}_J h_{n}$.
    \end{enumerate}
    Thus, each step on the path is of the form $h_{n-1} \left(\xrightarrow{\leq}_J\cup \xleftarrow{\leq}_J\cup\xrightarrow{\sim}_J\cup \xleftarrow{\sim}_J\right) h_{n}$. By definition of $\leq_J$ (\Cref{def:final-rels}), we get that $h\leq_J h'$, as required.

    Similarly, we prove the intersection condition only for $\sim$: the proof for $\leq_I$ is similar and less complicated. Let $I,J\subseteq A$ be groups. We show that for any $h,h'\in H$, we have $h\sim_{I\cup J} h'$ if and only if $h\sim_I h'$ and $h\sim_J h'$. Observe that if we assume $\sim_I=\bigcap_{i\in I}\sim_i$, then the result follows directly: by $\sim_{I\cup J} = \bigcap_{i\in I\cup J}\sim_i$, and $\sim_I = \bigcap_{i\in I}\sim_i$ and $\sim_J = \bigcap_{j\in J}\sim_j$, we get that
        \[
        \begin{array}{llll}
            \sim_{I\cup J} &= \bigcap_{i\in I\cup J}\sim_i
            &= \ \bigl (\bigcap_{i\in I}\sim_i\bigr ) \cap \bigl (\bigcap_{j\in J}\sim_j\bigr )
            &= \ \sim_I \cap \sim_J.
        \end{array}
        \]
        It remains to prove the claim. We state it for both $\sim_I$ and $\leq_I$.
        \begin{claim}\label{claim:is-intersect_OLD}
            For all groups $I\subseteq A$, we have that $\sim_I=\bigcap_{i\in I}\sim_i$ and $\leq_I=\bigcap_{i\in I}\leq_i$.
        \end{claim}
        \begin{proof}[Proof of claim.]
    We prove the claim only for $\sim_I$: the proof for $\leq_I$ is similar and less complicated.

    For the left-to-right direction, the claim reduces to anti-monotonicity, which we already proved. For the converse direction, let $h,h'\in H$ and suppose that $h\sim_i h'$ for all $i\in I$. Let $i\in I$ be arbitrary. By definition of $\sim_i$, each of the steps on the non-redundant path form $h$ to $h'$ is of one of the forms listed in \Cref{sim-equiv-set_OLD}. Consider an arbitrary step on this path, from a history $h_{n-1}$ to a history $h_n$. Since the proofs for the different cases from \Cref{sim-equiv-set_OLD} are symmetrical, we only show the proof for case (a)\footnote{For the proof of intersection for $\leq_I$, (a) is the only possible case for any step on the path (\Cref{leq-equiv-set_OLD}).}.

    Suppose that (a) the step is of the form $h_{n-1} \xrightarrow{\operatorname{P}}_{{G_n}} h_{n}$ for some ${G_n}\supseteq \{i\}$. Recall that $i$ was arbitrary, and that this path is unique. It follows that ${G_n}\supseteq \{i'\}$ for all $i'\in I$. But then ${G_n}\supseteq I$. Thus, by definition of $\xrightarrow{\leq}_I$ (\Cref{def:intermediate-rels}), we have that $h_{n-1} \xrightarrow{\leq}_I h_{n}$.

    Combining this with the proofs of the other cases, we get that $h_{n-1}$ and $h_{n}$ must be related by one of the one-step relations $\xrightarrow{\leq}_I, \xleftarrow{\leq}_I,\xrightarrow{\sim}_I$, or $\xleftarrow{\sim}_I$ for $I$. In other words, $h_{n-1} \left(\xrightarrow{\leq}_I\cup \xleftarrow{\leq}_I\cup\xrightarrow{\sim}_I\cup \xleftarrow{\sim}_I\right) h_{n}$. Since this was an arbitrary step on the unique non-redundant path from $h$ to $h'$, we can conclude that $h {\left(\xrightarrow{\leq}_I\cup \xleftarrow{\leq}_I\cup\xrightarrow{\sim}_I\cup \xleftarrow{\sim}_I\right)}^* h'$, i.e., $h \sim_I h'$, as required.
    \end{proof}
    In conclusion, all relational conditions of a pseudo-model for $\mathcal{L}_{{\B[\A]}_I}$ are satisfied by the defined relations.
\end{proof}

We can now define our associated model $\mathbf{X}$ for $\mathbf{S}$. We represent it as a relational evidence model (as opposed to a standard pseudo-model): we explicitly define only the individual relations.  

\begin{definition}[Associated Model]\label{def:assoc-model}
    The \emph{associated model} for $\mathbf{S}$ is a structure $\mathbf{X} = {(H,\leq_i,\sim_i,\llbracket\cdot\rrbracket_{\mathbf{X}})}_{i\in A}$, where 
    \begin{enumerate}
        \item $H$ is the set of all histories on $S$;
        \item For all $i\in A$, $\leq_i\ =\ \leq_{\{i\}}$ and $\sim_i\ =\ \sim_{\{i\}}$, with $\leq_{\{i\}}$ and $\sim_{\{i\}}$ as defined in \Cref{def:final-rels};
        \item The valuation $\llbracket\cdot\rrbracket_{\mathbf{X}}: \prop\to \mathcal{P}(H)$ on histories is defined as $\llbracket p\rrbracket_{\mathbf{X}} =\{h\in H \mid last(h)\in \llbracket p\rrbracket_{\mathbf{S}}\}$.
    \end{enumerate}
\end{definition}
From \Cref{lem:standard-ps}, we conclude that $\mathbf{X}$ is a relational evidence model (as defined in \Cref{def:rela-ev-model}).

\begin{proposition}\label{assoc-as-standardps_OLD}
    We can consider the associated model $\mathbf{X}$ as a standard pseudo-model $\mathbf{X}={(H,\leq_I,\sim_I,\llbracket\cdot\rrbracket_{\mathbf{X}})}_{I\subseteq A}$ for the language $\mathcal{L}_{{\B[\A]}_I}$, by explicitly representing the group relations $\leq_I$ and $\sim_I$ for all groups $I\subseteq A$, as defined in \Cref{def:final-rels}.
\end{proposition}
\begin{proof}
    We proved in \Cref{lem:standard-ps} that the group relations $\leq_I$ and $\sim_I$ for nonempty groups $I\subseteq A$ on $\mathbf{X}$ satisfy all conditions of a pseudo-model for $\mathcal{L}_{{\B[\A]}_I}$, and in particular, that the intersection condition for standard pseudo-models for $\mathcal{L}_{{\B[\A]}_I}$ (\Cref{standard-ps}) is satisfied.
\end{proof}

Given \Cref{assoc-as-standardps_OLD}, and since the pseudo-model-based semantics from \Cref{sem-boxall-pseudo} agrees with the model-based semantics from \Cref{sem-relev} for $\mathbf{X}$, we can compare $\mathbf{S}$ and $\mathbf{X}$ directly as pseudo-models, that is, by explicitly representing the \emph{group} relations.

For pseudo-models for the fragment $\mathcal{L}_{{\B[\A]}_{i,A}}$, the associated model will also be restricted to $\mathcal{L}_{{\B[\A]}_{i,A}}$. It is constructed in the same way: 

\begin{fact}\label{assoc-frag}
    Given a pseudo-model $\mathbf{S_f}$ for the fragment $\mathcal{L}_{{\B[\A]}_{i,A}}$, the \emph{associated model} for $\mathbf{S_f}$ is a structure $\mathbf{X_f}={(H,\leq_i,\sim_i, \llbracket\cdot\rrbracket_{\mathbf{X}})}_{i\in A}$, which is obtained by restricting the construction of the associated model from \Cref{def:assoc-model} to relations for individual agents and for the full group, i.e., the relations labeled by $A$ itself or by groups of the form $\{i\}\subseteq A$. The resulting structure $\mathbf{X_f}$ is a relational evidence model.
\end{fact}

To extend our completeness proof with respect to pseudo-models (\Cref{comp-pseudo_OLD}) to relational evidence models, we prove that every formula satisfiable on a pseudo-model for $\mathcal{L}_{{\B[\A]}_{i,A}}$ is also satisfiable on the associated model; we show that the map $last(\cdot)$ from histories to states is a p-morphism, i.e., a functional bisimulation (see~\cite{blackbrijkeven}). We need the following lemma. 

\begin{lemma}\label{leq-last_OLD}
$ $
    \begin{enumerate}
        \item For all groups $I\subseteq A$, if $h \xrightarrow{\leq}_I h'$, then $last(h) \leq_I last(h')$.
        \item For all groups $I\subseteq A$, if $h \xrightarrow{\sim}_I h'$, then $last(h) \sim_I last(h')$. 
    \end{enumerate}
\end{lemma}
\begin{proof} 
    \begin{enumerate}
        \item Suppose $h \xrightarrow{\leq}_I h'$. By \Cref{def:intermediate-rels} of $\xrightarrow{\leq}_I$, there is $G\supseteq I$ such that $h \xrightarrow{\operatorname{P}}_G h'$. By \Cref{def:intermediate-rels} of $\xrightarrow{\operatorname{P}}_G$, we have $h'=(h,\leq_G,s')$ with $last(h)\leq_G s'=last(h')$. By the anti-monotonicity condition on pseudo-models (\Cref{pseudo-model}), we get that $last(h) \leq_I last(h')$.
        \item Similar to the proof of (1).
    \end{enumerate}
\end{proof}

For our final step, we inductively extend the properties from Lemma~\ref{leq-last_OLD} to groups:

\begin{lemma}\label{extend-leq_OLD}
$ $
    \begin{enumerate}
        \item For all groups $I\subseteq A$, if $h \leq_I h'$, then $last(h) \leq_I last(h')$.
        \item For all groups $I\subseteq A$, if $h \sim_I h'$, then $last(h) \sim_I last(h')$.
    \end{enumerate}
    
\end{lemma}
\begin{proof}
    \begin{enumerate}
        \item By induction on the length $n$ of the non-redundant path from $h$ to $h'$. For the base case, where $h \leq_I h'$ with $n=0$, we have $h=h'$. So the claim that $last(h) \leq_I last(h')$ follows immediately from reflexivity of $\leq_I$.

    For the inductive step, suppose the claim holds for paths of length $n$, and suppose the non-redundant path from $h$ to $h'$ has length $n+1$. By \Cref{leq-equiv-set_OLD}, the last step of the non-redundant path from $h$ to $h'$ must be of the form $h_{n} \xrightarrow{\operatorname{P}}_{I^{n+1}} h_{n+1}=h'$, with $I^{n+1} \supseteq I$. So by definition of $\xrightarrow{\leq}_I$, we have $h_n\xrightarrow{\leq}_I h_{n+1}$. Using \Cref{leq-last_OLD}.1 we obtain that $last(h_n)\leq_I last(h_{n+1})$. By transitivity of $\leq_I$, it now suffices to show that $last(h)\leq_I last(h_n)$ (since that would give us that $last(h)\leq_I last(h_{n+1})$).

    Since the path from $h$ to $h_n$ has length $n$, we can apply the induction hypothesis to the fact that $h\leq_I h_n$ (which follows from our assumption that $h\leq_I h_{n+1}$). This gives us that $last(h) \leq_I last(h_n)$.

    \item By induction on the length $n$ of the non-redundant path from $h$ to $h'$. For the base case, where $h \sim_I h'$ with $n=0$ we have $h=h'$. So the claim that $last(h) \sim_I last(h')$ follows immediately from reflexivity of $\sim_I$.

    For the inductive step, suppose the claim holds for paths of length $n$, and suppose the non-redundant path from $h$ to $h'$ has length $n+1$. By \Cref{sim-equiv-set_OLD}, the last step of the non-redundant path from $h$ to $h'$ must be of one of the forms
    \begin{multicols}{2}
    \begin{enumerate}
        \item $h_{n} \xrightarrow{\operatorname{P}}_{I^{n+1}} h_{n+1}$
        \item $h_n \xleftarrow{\operatorname{P}}_{I^{n+1}} h_{n+1}$
        \item $h_{n} \xrightarrow{\operatorname{E}}_{I^{n+1}} h_{n+1}$
        \item $h_{n} \xleftarrow{\operatorname{E}}_{I^{n+1}} h_{n+1}$
    \end{enumerate}
    \end{multicols}
    with $h_{n+1}=h'$ and $I^{n+1} \supseteq I$. So applying the definitions of $\xrightarrow{\leq}_I$ and $\xrightarrow{\sim}_I$ to these cases, one of the following is the case:
    \begin{multicols}{2}
    \begin{enumerate}
        \item $h_{n} \xrightarrow{\leq}_{I^{n+1}} h_{n+1}$
        \item $h_n \xleftarrow{\leq}_{I^{n+1}} h_{n+1}$
        \item $h_{n} \xrightarrow{\sim}_{I^{n+1}} h_{n+1}$
        \item $h_{n} \xleftarrow{\sim}_{I^{n+1}} h_{n+1}$.
    \end{enumerate}
\end{multicols}
    First observe that the path from $h$ to $h_n$ has length $n$ and we can therefore apply the induction hypothesis to the fact that $h\sim_I h_n$ (which follows from the assumption that $h\sim_I h_{n+1}$ and from the definition of $\sim_I$). This gives us that $last(h)\sim_I last(h_n)$. It remains to show that $last(h_n)\sim_I last(h_{n+1})$, which by transitivity of $\sim_I$ will give us that $last(h)\sim_I last(h_{n+1})=last(h')$, as required.

    We use \Cref{leq-last_OLD}.1 for cases (1) and (2), and \Cref{leq-last_OLD}.2 for cases (3) and (4), to obtain that either $last(h_n)\leq_I last(h_{n+1})$, or $last(h_n)\sim_I last(h_{n+1})$, or one of their converses is true. In the cases of $last(h_n)\sim_I last(h_{n+1})$ and $last(h_{n+1})\sim_I last(h_n)$ we are done, so suppose that $last(h_n)\leq_I last(h_{n+1})$ or $last(h_{n+1})\leq_I last(h_n)$ is the case. But then we have by the inclusion condition on pseudo-models that $last(h_n)\sim_I last(h_{n+1})$, so we can conclude that $last(h)\sim_I last(h')$.
    \end{enumerate}
\end{proof}

We can now show that the function $last(\cdot)$ from associated models to pseudo-models is a \emph{p-morphism}.\footnote{A functional bisimulation, see~\cite{blackbrijkeven}.}

\begin{proposition}\label{bisim_OLD}
    Let $\mathbf{S}$ be a pseudo-model and let its associated model be given by $\mathbf{X}$. Then the map $last: H \rightarrow S$, mapping every history $h \in H$ to its last element $last(h)$, defines a p-morphism from $\mathbf{X}$ to $\mathbf{S}$ (with $\mathbf{X}$ and $\mathbf{S}$ seen as Kripke models with basic relations $\sim_I$ for all groups $I\subseteq A$). 
\end{proposition}
\begin{proof}
    The function $last(\cdot)$ is well-defined: since every history $h\in H$ is by definition a nonempty sequence, it contains at least one state. Since it is also finite, it must have a last state: $last(h)$ exists. To see that $last(\cdot)$ is a p-morphism, we check the following three conditions:

    \emph{Atomic preservation for atoms} $p \in \prop$: This is immediate by definition of the valuation function $\llbracket\cdot\rrbracket_{\mathbf{X}}$ for associated models.

    \emph{Forth condition}: let $I\subseteq A$ be a group. For $\leq_I$, assume $h\leq_I h'$; then $last(h) \leq_I last(h')$ is immediate from \Cref{extend-leq_OLD}.1. For $\sim_I$, assume $h \sim_I h'$ then $last(h) \sim_I last(h')$ is immediate from \Cref{extend-leq_OLD}.2.

    \emph{Back condition}: let $I\subseteq A$ be a group. For $\leq_I$, assume $last(h)\leq_I s'$. We need to prove that there is $h'\in H$ such that $h \leq_I h'$ and $last(h')=s'$. From $last(h)\leq_I s'$, we know that $(h,\leq_I, s')$ is a history in $H$. So we can take $h':=(h,\leq_I, s')$. Similarly, for $\sim_I$, assume $last(h)\sim_I s'$. Again, we can take $h':=(h,\leq_I, s')$ to prove that there is $h'\in H$ such that $h \sim_I h'$ and $last(h')=s'$.
\end{proof}

\begin{corollary}\label{coroll-bisim_OLD}
    The same formulas in $\mathcal{L}_{{\B[\A]}_I}$ are satisfiable in the associated model $\mathbf{X}$, as in its p-morphic image contained in the pseudo-model $\mathbf{S}$ for $\mathcal{L}_{{\B[\A]}_I}$. More precisely, for every history $h \in H$ and every formula $\phi\in \mathcal{L}_{{\B[\A]}_I}$, we have:
    \begin{align*}
        h \vDash_{\mathbf{X}} \phi \quad \text{ iff }\quad last(h) \vDash_{\mathbf{S}} \phi.
    \end{align*}
\end{corollary}
\begin{proof}
     By \Cref{bisim_OLD}, the map $last(\cdot): H \rightarrow S$ is a bisimulation between $\mathbf{S}$ and its image in $\mathbf{X}$, seen as Kripke models for the language with modalities $\B_I$ and ${[\A]}_I$ for all groups $I\subseteq A$. Since $\mathcal{L}_{{\B[\A]}_I}$ is the basic modal language for this vocabulary, formulas in $\mathcal{L}_{{\B[\A]}_I}$ are preserved by $last(\cdot)$ (by the standard results on preservation of modal formulas under bisimulations, cf.~\cite{blackbrijkeven}).
\end{proof}

We naturally extend \Cref{coroll-bisim_OLD} to the fragment of the language:

\begin{corollary}\label{coroll-bisim-frag_OLD}
    Let $\mathbf{S_f}$ be a pseudo-model for $\mathcal{L}_{{\B[\A]}_{i,A}}$. The same formulas in $\mathcal{L}_{{\B[\A]}_{i,A}}$ are satisfiable in the associated model $\mathbf{X_f}$, as in its p-morphic image in $\mathbf{S_f}$.
\end{corollary}
\begin{proof}
    The proof is obtained by restricting all the constructions in the proof of \Cref{coroll-bisim_OLD} to $\mathcal{L}_{{\B[\A]}_{i,A}}$. This gives us a bisimulation between $\mathbf{S_f}$ and $\mathbf{X_f}$. An argument following the same line of reasoning as \Cref{coroll-bisim_OLD} then concludes our proof.
\end{proof}

We finally prove \Cref{corr:compness-boxall-frag}: 

\medskip
\paragraph{Proof of \Cref{corr:compness-boxall-frag}} 
For $\bm{{\B[\A]}_I}$, soundness of the axioms and rules from \Cref{pf-syst} is a routine check.

As for completeness, let $\phi\in \mathcal{L}_{{\B[\A]}_I}(\prop)$ be any consistent formula. By \Cref{comp-pseudo_OLD}, there exists a pseudo-model $\mathbf{S}={(S,\leq_I,\sim_I, \llbracket\cdot\rrbracket_{\mathbf{S}})}_{I\subseteq A}$ for $\mathcal{L}_{{\B[\A]}_I}$ and some state $s_0\in S$, such that $(\mathbf{S},s)\vDash\phi$. Consider the associated model $\mathbf{X}={(H,\leq_i,\sim_i, \llbracket\cdot\rrbracket_{\mathbf{X}})}_{i\in A}$ for $\mathbf{S}$, where $H$ is given by the set of $s_0$-generated histories in the pseudo-model $\mathbf{S}$.

By \Cref{coroll-bisim_OLD}, the same formulas in $\mathcal{L}_{{\B[\A]}_I}(\prop)$ are satisfied in the associated model $\mathbf{X}$ as in its p-morphic image in $\mathbf{S}$. Note that $s_0$ is contained in the p-morphic image of $\mathbf{X}$ in $\mathbf{S}$, since the sequence $h:=(s_0)$ is an $s_0$-generated history in $H$ with $last(h)=s_0$. Therefore, $\phi$ is satisfied on $\mathbf{X}$.

This gives us weak completeness of $\bm{{\B[\A]}_I}$ with respect to relational evidence models. By \Cref{cor:rel-sem-equiv-concl_OLD}, we obtain weak completeness with respect to multi-agent topo-e-models. Decidability of the logic of $\mathcal{L}_{{\B[\A]}_I}$ follows from the fact that it has the finite pseudo-model property (see \Cref{comp-pseudo_OLD}).

For $\bm{{\B[\A]}_{i,A}}$, soundness follows directly from soundness for the proof system $\bm{{\B[\A]}_I}$, given that the axioms and rules of $\bm{{\B[\A]}_{i,A}}$ are contained in $\bm{{\B[\A]}_I}$.

The completeness proof for $\bm{{\B[\A]}_{i,A}}$ follows the same line of reasoning as the proof for $\bm{{\B[\A]}_I}$: let $\phi\in \mathcal{L}_{{\B[\A]}_{i,A}}(\prop)$ be any consistent formula. By \Cref{comp-pseudo_OLD}, there exists a pseudo-model $\mathbf{S}$ for $\mathcal{L}_{{\B[\A]}_{i,A}}$ that satisfies $\phi$ at some state $s_0$. By \Cref{coroll-bisim-frag_OLD}, there exists an associated model $\mathbf{X}$ for $\mathbf{S}$, such that its state space $H$ is defined by the $s_0$-originated histories of $\mathbf{S}$, and therefore satisfies $\phi$. The associated model is a relational evidence model, which gives us weak completeness for $\bm{{\B[\A]}_{i,A}}$ with respect to relational evidence models. By \Cref{cor:rel-sem-equiv-concl_OLD}, we obtain weak completeness for topo-e-models. Decidability of $\bm{{\B[\A]}_{i,A}}$ follows from the fact that it has the finite pseudo-model property (see \Cref{comp-pseudo_OLD}).


\subsection{Proof of Completeness and Decidability for the Restricted Logic of Group Knowledge and Belief (\texorpdfstring{\Cref{corr:compness-kb}}{Theorem 8})}

The structure of this proof bears a superficial resemblance to the structure of the proof of \Cref{corr:compness-boxall-frag}. The construction for $\mathcal{L}_{KB_{i,A}}$ requires an extra step. We first define the relevant pseudo-models, which explicitly represent the relations corresponding to knowledge and belief. We prove completeness with respect to these structures via the standard canonical-model construction; we refer to~\cite{blackbrijkeven} for a detailed discussion of this construction. 

The crucial step is the \emph{representation theorem}
(\Cref{boxforall-on-kb_OLD}). It states that each pseudo-model $\mathbf{M}$ for $\mathcal{L}_{KB_{i,A}}$ can alternatively be represented as a pseudo-model $\mathbf{S}$ for $\mathcal{L}_{{\B[\A]}_{i,A}}$, which agrees with $\mathbf{M}$ on the interpretation of formulas over the language $\mathcal{L}_{KB_{i,A}}$. We recover relations for knowledge and belief on pseudo-models for $\mathcal{L}_{{\B[\A]}_{i,A}}$ of evidence and, conversely,  evidence relations on a pseudo-model for $\mathcal{L}_{KB_{i,A}}$ of knowledge and belief. The former construction is straightforward, since the relations for knowledge and belief are uniquely determined by the evidence relations. The latter is more complicated.

Throughout this proof, fix a finite group of agents $A$ and a finite vocabulary $\prop$. 

\subsubsection{Pseudo-Models for \texorpdfstring{$\mathcal{L}_{KB_{i,A}}$.}{KBiA}}\label{sec:kb-def-pseudo} For several relations $R$ on our models, we will be using the notion of $R$-maximal worlds. We define $R$-maximality as follows: 

\begin{definition}[$R$-maximal worlds]\label{maximal}
    Given a set of states $S$ and a relation $R$ on $S$, define $R$-maximal worlds of $S$ as $Max_R(S):=\{s\in S\mid \A w\in S(s R w \Rightarrow w R s)\}$.
\end{definition} 

We now define the relevant pseudo-models.

\begin{definition}[Pseudo-Model for $\mathcal{L}_{KB_{i,A}}$]\label{kb-pseudo}
    A \emph{pseudo-model for $\mathcal{L}_{KB_{i,A}}$} is a structure $\mathbf{M}={(S, \kn_i, \rightarrow_i, \kn_A, \rightarrow_A,\llbracket\cdot\rrbracket)}_{i\in A}$, where $S$ is a set of states and $V$ is a valuation. A pseudo-model for $\mathcal{L}_{KB_{i,A}}$ is required to satisfy the following conditions: 
    \begin{enumerate}
        \item \textbf{Stalnaker's conditions}. The knowledge and belief modalities $\kn_i$ and $\rightarrow_i$, for $i\in A$, and $\kn_A$ and $\rightarrow_A$, for the full group $A$, each satisfy the relational correspondents of Stalnaker's axioms (see~\cite{stalnaker}). That is, for all $\alpha\in \{A\}\cup A$ we have:
    \begin{itemize}
        \item The $\kn_\alpha$ (knowledge) relation is \textbf{S4}, i.e., $\kn_\alpha$ is a preorder;
        \item The $\rightarrow_\alpha$ (belief) relation is \textbf{KD45}, i.e., $\rightarrow_\alpha$ is serial, transitive, and Euclidean; 
        \item \textbf{Inclusion}. $\rightarrow_\alpha\subseteq \kn_\alpha$; 
        \item \textbf{Strong Transitivity}. For all $s,t,u\in S$, if $s\kn_\alpha t$ and $t\rightarrow_\alpha u$, then $s\rightarrow_\alpha u$;
        \item \textbf{Strong Euclideanity}. For all $s,t,u\in S$, if $s\kn_\alpha t$ and $s\rightarrow_\alpha u$, then $t\rightarrow_\alpha u$;
        \item \textbf{Full Belief}. For all $s,t,u\in S$, if $s\rightarrow_\alpha t$ and $t\kn_\alpha u$, then $s\rightarrow_\alpha u$.
    \end{itemize}
    \item \textbf{WM-Condition}. For all $i\in A$, $\kn_A\subseteq (\kn_i\cup \rightarrow_A)$;
    \item \textbf{Super-Introspection condition}. For all $s,t,u\in S$, if $s\kn_A t$, then we have for all $i\in A$, that $s\rightarrow_i u$ if and only if $t\rightarrow_i u$;
    \item \textbf{CBD-Condition}. For all $s\in S$ there exists $w\in S$ such that $s\ (\rightarrow_A\cap \bigcap_{i\in A} \kn_i)\ w$. 
    \end{enumerate}
\end{definition}

It can be checked that Stalnaker's conditions imply that the knowledge relation on pseudo-models for $\mathcal{L}_{KB_{i,A}}$ is weakly directed\footnote{A relation $R$ on a relational frame $\mathcal{M}=(X,R)$ is \emph{weakly directed} (also called \emph{directed} or \emph{confluent}) if we have for all $x,y,z\in X$ with $x R y$ and $x R z$, that there exists $u\in X$ such that $y R u$ and $z R u$.} (see~\cite{stalnakersknowledge}), that is, the knowledge relation is $\mathsf{S4.2}$. 

Belief relations on pseudo-models for $\mathcal{L}_{KB_{i,A}}$ have the property that $\alpha\in \{A\}\cup A$ believes $\phi$ if and only if $\phi$ is true in the $\kn_\alpha$-maximal worlds within the current information cell:

\begin{lemma}\label{nice-belief_OLD}
    On a pseudo-model $\mathbf{M} = {(S, \kn_i, \rightarrow_i, \kn_A, \rightarrow_A, \llbracket\cdot\rrbracket)}_{i \in A}$ for $\mathcal{L}_{KB_{i,A}}$, we have for all $\alpha\in \{A\}\cup A$ and for all $s,w\in S$ that
    \begin{align*}
        s\rightarrow_\alpha w &\text{ iff } s\kn_\alpha w\in Max_{\kn_\alpha}(S).
    \end{align*} 
\end{lemma}
\begin{proof}
    Let $i\in A$. We show the proof for the individual relation $\rightarrow_i$. The proof for $\rightarrow_A$ is symmetrical, as it refers only to Stalnaker's conditions on $\rightarrow_A$, which are analogous to Stalnaker's conditions on $\rightarrow_i$. For the left-to-right direction, let $s,w\in S$ and suppose $s\rightarrow_i w$. Then the inclusion condition gives us that $s\kn_i w$. To show that $w\in Max_{\kn_i}(S)$, suppose that $w\kn_i w'$. It suffices to show that $w'\kn_i w$. Observe that $s\kn_i w\kn_i w'$ gives us $s\kn_i w'$ (by transitivity of $\kn_i$); now we have $s\rightarrow_i w$ and $s\kn_i w'$ which, by strong Euclideanity, gives us $w'\rightarrow_i w$. But then, again by the inclusion condition, we obtain $w'\kn_i w$, as required.

    For the right-to-left direction, let $s,w\in S$ and suppose $s\kn_i w\in Max_{\kn_i}(S)$. We show $s\rightarrow_i w$. By seriality of $\rightarrow_i$, there exists $w'\in S$ such that $w\rightarrow_i w'$. Using the inclusion condition, we obtain $w\kn_i w'$. But then, since $w\in Max_{\kn_i}(S)$, we also have $w'\kn_i w$. Now we have $s\kn_i w\rightarrow_i w'$ which, by strong transitivity, gives us $s\rightarrow_i w'$. Finally, $s\rightarrow_i w'\kn_i w$ gives us (by full belief) that $s\rightarrow_i w$, as required.
\end{proof}

We interpret formulas over $\mathcal{L}_{KB_{i,A}}$ on pseudo-models for $\mathcal{L}_{KB_{i,A}}$ as follows:

\begin{definition}[Pseudo-Model Semantics of $\mathcal{L}_{KB_{i,A}}$]\label{sem-pseudo-kb}
    The topological semantics of $\mathcal{L}_{KB_{i,A}}(\prop)$ on pseudo-model $\mathbf{M}= {(S, \kn_i, \rightarrow_i, \kn_A, \rightarrow_A,\llbracket\cdot\rrbracket)}_{i\in A}$ for $\mathcal{L}_{KB_{i,A}}$ is defined recursively as
    \[
    \begin{array}{lll}
        (\mathbf{M}, s) \vDash p & \text{ iff } & s\in 
        \llbracket p\rrbracket
        \\
        (\mathbf{M}, s) \vDash \neg \phi & \text{ iff } &  (\mathbf{M}, s) \not\vDash \phi \\
        (\mathbf{M}, s) \vDash \phi \wedge \psi & \text{ iff } & (\mathbf{M}, s) \vDash \phi \text{ and } (\mathbf{M}, s) \vDash \psi \\
        (\mathbf{M}, s) \vDash K_i \phi & \text{ iff } & \text{ for all } t\in S \text{ s.t. } s\kn_i t: (\mathbf{M}, t) \vDash \phi \\
        (\mathbf{M}, s) \vDash B_i \phi & \text{ iff } &  \text{ for all } t\in S \text{ s.t. } s\rightarrow_i t: (\mathbf{M}, t) \vDash \phi \\
        (\mathbf{M}, s) \vDash K_A \phi & \text{ iff } &  \text{ for all } t\in S \text{ s.t. } s\kn_A t: (\mathbf{M}, t)\vDash \phi \\
        (\mathbf{M}, s) \vDash B_A \phi & \text{ iff } &  \text{ for all } t\in S \text{ s.t. } s\rightarrow_A t: (\mathbf{M}, t) \vDash \phi
    \end{array}
    \]
    where $s\in S$ is any state and $p\in \prop$ is any propositional variable.
\end{definition}

\subsubsection{Soundness and Completeness of \texorpdfstring{$\mathcal{L}_{KB_{i,A}}$}{KBiA} w.r.t. Pseudo-Models.}\label{sec:pseudo-sound-comp} We first prove soundness and completeness with respect to pseudo-models. \Cref{soundness-kb-pseudo_OLD} takes care of soundness. 

\begin{proposition}\label{soundness-kb-pseudo_OLD}
    The proof system $\bm{KB_{i,A}}$ for $\mathcal{L}_{KB_{i,A}}$ (displayed in \Cref{pf-systKB}) is sound with respect to relational pseudo-models for $\mathcal{L}_{KB_{i,A}}$. 
\end{proposition}
\begin{proof}
    The proof is a routine check of the correspondences between the axioms of $\bm{KB_{i,A}}$ and the properties of pseudo-models for $\mathcal{L}_{KB_{i,A}}$. 
\end{proof}

Since the completeness proof amounts to a straightforward ``canonical model'' construction, we omit it here. \Cref{comp-pseudo-KB} summarises our results:

\begin{corollary}\label{comp-pseudo-KB}
    The proof system $\bm{KB_{i,A}}$ is sound and weakly complete with respect to pseudo-models for $\mathcal{L}_{KB_{i,A}}$.
\end{corollary}
\begin{proof}
    We prove completeness with respect to pseudo-models for $\mathcal{L}_{KB_{i,A}}$ by showing that every consistent formula $\phi\in \mathcal{L}_{KB_{i,A}}(\prop)$ is satisfiable in the canonical pseudo-model for $\mathcal{L}_{KB_{i,A}}$ for $\mathcal{L}_{KB_{i,A}}(\prop)$. The canonical pseudo-model is defined according to the standard ``canonical model'' construction (see e.g.~\cite{blackbrijkeven}). 
\end{proof}

\textbf{Back and Forth between Pseudo-Models for \texorpdfstring{$\mathcal{L}_{KB_{i,A}}$}{KBiA} and for \texorpdfstring{$\mathcal{L}_{{\B[\A]}_I}$}{BoxForallI}.}\label{sec:repr-thm} It remains to prove completeness with respect to the intended topo-e-models. We represent the pseudo-models from \Cref{kb-pseudo} as pseudo-models for $\mathcal{L}_{{\B[\A]}_{i,A}}$, such that we can reuse the correspondence of these models with standard relational evidence models (which we obtained by using the unraveling technique, see \Cref{sec:pseudo-to-model-boxall}). 

We show both directions of the correspondence. The proof is structured as follows: 
\begin{enumerate}[(1)]
    \item \emph{From pseudo-models for $\mathcal{L}_{{\B[\A]}_{i,A}}$ to pseudo-models for $\mathcal{L}_{KB_{i,A}}$}. This is the straightforward direction of the proof. We recover the (uniquely determined) relations for knowledge and belief on a pseudo-model $\mathbf{S}$ for $\mathcal{L}_{{\B[\A]}_{i,A}}$, and show that the semantics for $\mathcal{L}_{{\B[\A]}_{i,A}}$ in terms of these relations (defined in \Cref{sem-pseudo-kb}), applied to $\mathbf{S}$, agrees with the original semantics on $\mathbf{S}$ (defined in \Cref{sem-boxall-pseudo}), on all formulas over $\mathcal{L}_{KB_{i,A}}$ where we recall that $\mathcal{L}_{KB_{i,A}}$ is a fragment of $\mathcal{L}_{{\B[\A]}_{i,A}}$. 
    
    \item The crucial step: \emph{from pseudo-models for $\mathcal{L}_{KB_{i,A}}$ to pseudo-models for $\mathcal{L}_{{\B[\A]}_{i,A}}$}. We present an approach to recover the evidence relations on a pseudo-model $\mathbf{M}$ for $\mathcal{L}_{KB_{i,A}}$ using the existing relations for knowledge and belief.
    
    \item Representing $\mathbf{M}$ as a pseudo-model $\mathbf{S}$ for $\mathcal{L}_{{\B[\A]}_{i,A}}$, consisting of the newly defined evidence relations, we use the approach from (1) to recover the corresponding knowledge and belief relations from these evidence relations. 
    
    \item Finally, we show that the newly recovered relations for knowledge and belief on $\mathbf{S}$ coincide with the original relations for knowledge and belief on $\mathbf{M}$. 
    \item Using the result from (1), we conclude that on the two representations of $\mathbf{M}$, as a pseudo-model for $\mathcal{L}_{KB_{i,A}}$ and as a pseudo-model for $\mathcal{L}_{{\B[\A]}_{i,A}}$, the interpretations of knowledge and belief coincide.
    \item We derive that the semantics for the two representations agree on all formulas in the language $\mathcal{L}_{KB_{i,A}}$. 
\end{enumerate}

\textbf{From pseudo-models for \texorpdfstring{$\mathcal{L}_{{\B[\A]}_{i,A}}$}{BoxForalliA} to pseudo-models for \texorpdfstring{$\mathcal{L}_{KB_{i,A}}$}{KBiA}}\label{sec:box-model-to-pseudo}. On pseudo-models for $\mathcal{L}_{{\B[\A]}_{i,A}}$, relations for knowledge and belief can directly be recovered from the evidence relations by unfolding the interpretations of knowledge and belief as abbreviations; we show this in \Cref{kb-on-boxforall_OLD}.

We define a map that requires the $\leq$ relations on the pseudo-model to have a particular property, which we refer to as \emph{max-density}: 

\begin{definition}[Max-dense]\label{max-dense}
    Given a set of states $S$ and a pre-order $R$ on $S$, we say that $R$ is \emph{max-dense} if, for all $s\in S$, there exists $t \in Max_{R}(S)$, such that $s R t$. Equivalently, the pre-order $R$ is max-dense if the set $Max_{R}(S)$ is dense in the up-set topology (see~\cite{aybukephd}, Chapter 3.1.2) with respect to $R$. 
\end{definition}

\begin{observation}
    Relations on finite models are automatically max-dense, by the absence of infinite $R$-chains. Thus, we can consider max-density as a generalization of finiteness.
\end{observation}

We now define the map and prove that it preserves the interpretation of knowledge and belief. 

\begin{proposition}\label{kb-on-boxforall_OLD} 
    Let $\mathbf{S}={(S,\leq_i,\sim_i,\leq_A,\sim_A,\llbracket\cdot\rrbracket)}_{i\in A}$ be a pseudo-model for $\mathcal{L}_{{\B[\A]}_{i,A}}$ such that for each $\alpha\in \{A\}\cup A$, the relation $\leq_\alpha$ is max-dense. Let $\alpha\in \{A\}\cup A$. If we set 
    \begin{align*}
        s\rightarrow_\alpha^{\mathbf{S}} w \quad &\text{ iff } \quad s\sim_\alpha w \in Max_{\leq_\alpha}(S) \\
        s\kn_\alpha^{\mathbf{S}} w \quad &\text{ iff } \quad s \leq_\alpha w \text{ or } s\rightarrow_\alpha^{\mathbf{S}} w 
    \end{align*} 
    then the following statements hold for $\mathbf{S}$:
    \begin{enumerate}
    \item the structure
      $\mathbf{M}_{\mathbf{S}} = {(S, \kn_i^{\mathbf{S}}, \rightarrow_i^{\mathbf{S}}, \kn_A^{\mathbf{S}},\rightarrow_A^{\mathbf{S}},\llbracket\cdot\rrbracket)}_{i\in A}$ is a pseudo-model for $\mathcal{L}_{KB_{i,A}}$;
        \item  we have for all $s\in S$ and for all formulas $\phi$ over the language $\mathcal{L}_{KB_{i,A}}$:
        \[
        \begin{array}{llll}
            (a) & (\mathbf{S},s)\vDash B_\alpha\phi & \text{ iff } & \text{ for all } t\in S \text{ s.t. } s\rightarrow_\alpha^{\mathbf{S}} t: (\mathbf{M}_{\mathbf{S}}, t) \vDash \phi \\
             (b) & (\mathbf{S},s)\vDash K_\alpha\phi & \text{ iff } &\text{ for all } t\in S \text{ s.t. } s \kn_\alpha^{\mathbf{S}} t: (\mathbf{M}_{\mathbf{S}}, t)\vDash \phi.
        \end{array}
        \]
    \end{enumerate}
\end{proposition}
\begin{proof}
    The interpretations of the modalities and the definitions of the corresponding relations for the full group $A$ are analogous to those for individual agents, therefore we only show the cases for individual agents in both proofs. 

    For (1), we show that $\mathbf{M}_{\mathbf{S}}$ satisfies the conditions of a pseudo-model for $\mathcal{L}_{KB_{i,A}}$ (\Cref{kb-pseudo}). Let $i\in A$. 
        \begin{itemize}
            \item Stalnaker's conditions. First, $\kn_i^{\mathbf{S}}$ is a pre-order. For reflexivity, observe that $s \kn_i^{\mathbf{S}} s$ follows from the fact that $s\leq_i s$ (by reflexivity of $\leq_i$ in \Cref{pseudo-model}). For transitivity, suppose $s\kn_i^{\mathbf{S}} w \kn_i^{\mathbf{S}} v$. Applying the definition of $\kn_i^{\mathbf{S}}$, we have one of the following four cases:
            \begin{enumerate}[(a)]
                \item $s\leq_i w \leq_i v$. Then $s\leq_i v$ (by transitivity of $\leq_i$, \Cref{pseudo-model}), so $s\kn_i^{\mathbf{S}} v$.  
                \item $s\leq_i w \rightarrow_i^{\mathbf{S}} v$. So $w\sim_i v$ with $v \in Max_{\leq_i}(S)$. By the inclusion condition on $\mathbf{S}$, we have $s\sim_i w$, so with $\sim_i$ being an equivalence relation $s\sim_i v$. But then $s\rightarrow_i^{\mathbf{S}} v$, so $s\kn_i^{\mathbf{S}} v$.  
                \item $s\rightarrow_i^{\mathbf{S}} w \rightarrow_i^{\mathbf{S}} v$. Then $s\sim_i w$ and $w\sim_i v$, with $v \in Max_{\leq_i}(S)$, so $s\rightarrow_i^{\mathbf{S}} v$, and therefore, $s\kn_i^{\mathbf{S}} v$.
                \item $s\rightarrow_i^{\mathbf{S}} w \leq_i v$. Then $s\sim_i w$ with $w\in Max_{\leq_i}(S)$, so with $w \leq_i v$, it must be that $v\in Max_{\leq_i}(S)$. With $w\leq_i v$, we have by the inclusion condition on $\mathbf{S}$ that $w\sim_i v$, so by $\sim_i$ being an equivalence relation, we have $s\sim_i v$. So $s\kn_i^{\mathbf{S}} v$.  
            \end{enumerate}
            Next, we show that $\rightarrow_i^{\mathbf{S}}$ is serial, transitive, and Euclidean. For seriality, let $s\in S$. Note that $\leq_i$ is max-dense. Thus, there is $t\in Max_{\leq_i}(S)$ such that $s\leq_i t$. By the inclusion condition on $\mathbf{S}$, we obtain $s\sim_i t$ which, by the fact that $t\in Max_{\leq_i}(S)$, gives us that $s\rightarrow_i^{\mathbf{S}} t$. For transitivity, see item (c) above. For Euclideanity, let $s\rightarrow_i^{\mathbf{S}} w$ and $s\rightarrow_i^{\mathbf{S}} v$. Then $s\sim_i w$ an $s\sim_i v$, with both $w\in Max_{\leq_i}(S)$ and $v\in Max_{\leq_i}(S)$. As $
            \sim_i$ is an equivalence relation, we have $w\sim_i v$, giving us $w \rightarrow_i^{\mathbf{S}} v$. 
            
            Inclusion. Suppose $s\rightarrow_i^{\mathbf{S}}w$. Then, by definition, $s\kn_i^{\mathbf{S}} w$.

            Strong transitivity. See case (b) above.

            Strong Euclideanity. Suppose $s\kn_i^{\mathbf{S}} w$ and $s\rightarrow_i^{\mathbf{S}} v$. We claim that $w\rightarrow_i^{\mathbf{S}} v$. Given the assumption that $s\kn_i^{\mathbf{S}} w$, there are two possible cases: either (1) $s\leq_i w$, or (2) $s\rightarrow_i^{\mathbf{S}} w$. In either case, we have $s\sim_i w$: in the case of (1), it follows from inclusion on $\mathbf{S}$; in the case of (2), it follows from the definition of $\rightarrow_i^{\mathbf{S}}$. Since $\sim_i$ is an equivalence relation, we have $w\sim_i v$. By definition of $\rightarrow_i^{\mathbf{S}}$, we get $w\rightarrow_i^{\mathbf{S}} v$, as required.

            Full belief. Suppose $s\rightarrow_i^{\mathbf{S}} w$ and $w\kn_i^{\mathbf{S}} v$. We claim that $s\rightarrow_i^{\mathbf{S}} v$. By definition of $\rightarrow_i^{\mathbf{S}}$, we have $s\sim_i w$ with $w\in Max_{\leq_i}(S)$. By $w\kn_i^{\mathbf{S}} v$, it must be that $v\in Max_{\leq_i}(S)$. We have $w\sim_i v$ by definition of $\rightarrow_i^{\mathbf{S}}$, which gives us that $w\rightarrow_i^{\mathbf{S}} v$, as required. 

            \item WM-condition. We show that $\kn_A^{\mathbf{S}}\subseteq (\kn_i^{\mathbf{S}}\cup \rightarrow_A^{\mathbf{S}})$. Let $s\kn_A^{\mathbf{S}}w$; we show that $s (\kn_i^{\mathbf{S}}\cup \rightarrow_A^{\mathbf{S}})w$. The assumption gives us two possible cases: either (1) $s\leq_A w$, or (2) $s\rightarrow_A^{\mathbf{S}}w$. In the case of (1), anti-monotonicity of $\leq$ gives us $s\leq_i w$ so, by definition, $s \kn_i^{\mathbf{S}} w$, and therefore, $s (\kn_i^{\mathbf{S}}\cup \rightarrow_A^{\mathbf{S}})w$. In the case of (2), the claim is immediate from $s\rightarrow_A^{\mathbf{S}}w$. 
            
            \item Super-introspection condition. Suppose $s\kn_A^{\mathbf{S}} t$. We show that $s\rightarrow_i^{\mathbf{S}} u$ if and only if $t\rightarrow_i^{\mathbf{S}} u$. We show one direction; the converse direction is symmetrical. Suppose $s\rightarrow_i^{\mathbf{S}} u$. Then $s\sim_i u$ with $u\in Max_{\leq_i}(S)$. By the inclusion condition on $\mathbf{S}$, the assumption $s\kn_A^{\mathbf{S}} t$ gives us that $s\sim_i t$. Since $\sim_i$ is an equivalence relation, $t\sim_i u$. With $u\in Max_{\leq_i}(S)$, we have $t \rightarrow_i^{\mathbf{S}} u$, as required. 
            
            \item CBD-condition. Let $s\in S$. We show that there exists $w\in S$ such that $s\ (\rightarrow_A^{\mathbf{S}}\cap \kn_i^{\mathbf{S}})\ w$ (for our fixed, arbitrary $i\in A$). We use max-density of $\leq_A$: there exists $w\in Max_{\leq_A}(S)$ such that $s\leq_A w$. By the inclusion condition on $\mathbf{S}$, we have $s\sim_A w$. Thus, $s\rightarrow_A^{\mathbf{S}} w$. To see that we also have $s \kn_i^{\mathbf{S}} w$, note that with $s\leq_A w$, anti-monotonicity of $\leq$ gives us that $s\leq_i w$. Thereby, $s \kn_i^{\mathbf{S}} w$, as required.
        \end{itemize}
        
        For the proofs of (2), let $s\in S$ and let $\phi\in \mathcal{L}_{KB_{i,A}}(\prop)$ and fix an agent $i\in A$.
        \begin{enumerate}
        \item[(a)] $(\mathbf{S},x)\vDash B_i\phi \text{ iff } \text{ for all } t\in S \text{ s.t. } s\rightarrow_i^{\mathbf{S}} t: (\mathbf{M}_{\mathbf{S}}, t) \vDash \phi$. 
        
        Unfolding the semantic definition of $B_i$ on pseudo-models for $\mathcal{L}_{{\B[\A]}_{i,A}}$ (\Cref{sem-boxall-pseudo}) in terms of the evidence relations, we obtain the following interpretation for $B_i$, which we will use:
        \[ 
        \begin{array}{lll}
            (\mathbf{S},s)\vDash B_i\phi &\text{ iff } &\A t\in S: \\ & &  s\sim_i t  \Rightarrow (\E u\in S (t\leq_i u \text{ and } \A v\in S: u\leq_i v\Rightarrow (\mathbf{S},v)\vDash \phi)).
        \end{array}
        \]
        For the left-to-right direction, suppose $(\mathbf{S},s)\vDash B_i\phi$. 
        We need to show that for all $t\in S$ with $s\rightarrow_i^{\mathbf{S}} t$, we have $(\mathbf{M}_{\mathbf{S}}, t) \vDash \phi$. So let $t\in S$ such that $s \rightarrow_i^{\mathbf{S}} t$. By definition of $\rightarrow_i^{\mathbf{S}}$, we have $s\sim_i t$ with $t\in Max_{\leq_i}(S)$. So by the unfolding of $B_i\phi$, there exists $u\in S$ such that $t\leq_i u $ and for all $v\in S$, $u\leq_i v$ implies $(\mathbf{S},v)\vDash \phi$. Furthermore, by $\leq_i$-maximality of $t$, we have that $t \leq_i u$ implies $u\leq_i t$. So by $u\leq_i t$, we have $(\mathbf{S},t)\vDash \phi$, i.e., $(\mathbf{M}_{\mathbf{S}}, t)\vDash \phi$, as required. 
        
        For the converse direction, suppose that $s\rightarrow_i^{\mathbf{S}} t$ implies that $(\mathbf{M}_{\mathbf{S}}, t) \vDash \phi$ (i.e., $(\mathbf{S}, t) \vDash \phi$), for all $t\in S$. Let $w\in S$ with $s\sim_i w$. We want to find $u\in S$ such that $w\leq_i u$ and for all $v\in S$, $u\leq_i v$ implies $(\mathbf{S},v)\vDash \phi$. 

        By max-density of $\leq_i$, there exists $u\in Max_{\leq_i}(S)$ such that $w\leq_i u$. By the inclusion condition on $\mathbf{S}$, we have $w \rightarrow_i^{\mathbf{S}} u$. If we prove that for all $v\in S$, $u\leq_i v$ implies $(\mathbf{S}, v) \vDash \phi$, then we are done. So let $v\in S$ and suppose $u\leq_i v$. As a property of $\leq_i$-maximality, it must be that $v\in Max_{\leq_i}(S)$. Now we claim that $s\rightarrow_i^{\mathbf{S}} v$: we have a chain $s\sim_i w\leq_i u\leq_i v$, so by the inclusion condition on $\mathbf{S}$ and by the properties of $\sim_i$, we have $s\sim_i v$. Therefore,  $s\rightarrow_i^{\mathbf{S}} v$. But then, by assumption, $(\mathbf{S},v)\vDash \phi$, as required. 
    
        \item[(b)] $(\mathbf{S},s)\vDash K_i\phi \text{ iff } \text{ for all } t\in S \text{ s.t. } s \kn_i^{\mathbf{S}} t: (\mathbf{M}_{\mathbf{S}}, t)\vDash \phi$.
        
        For the left-to-right direction, suppose $(\mathbf{S},s)\vDash K_i\phi$. Unfolding the interpretation of $K_i$, we have that $(\mathbf{S},s)\vDash \B_i\phi\wedge B_i\phi$. To prove the claim, let $t\in S$ and suppose that $s\kn_i^{\mathbf{S}}t$. We have two possible cases: either (1) $s\leq_i t$ or (2) $s\rightarrow_i^{\mathbf{S}} t$. In either case we have $(\mathbf{S},t)\vDash \phi$:
        \begin{enumerate}[(1)]
            \item If $s\leq_i t$, then by $(\mathbf{S},s)\vDash \B_i\phi$, we get that $(\mathbf{S},t)\vDash \phi$, i.e., $(\mathbf{M}_{\mathbf{S}}, t)\vDash \phi$.
            \item If $s\rightarrow_i^{\mathbf{S}} t$, then with $(\mathbf{S},s)\vDash B_i\phi$, (a) gives us that $(\mathbf{M}_{\mathbf{S}}, t)\vDash \phi$.
        \end{enumerate}
        
        For the converse direction, suppose that for all $t\in S$, $s \kn_i^{\mathbf{S}} t$ implies that $(\mathbf{M}_{\mathbf{S}}, t)\vDash \phi$. It suffices to show that $(\mathbf{S},s)\vDash \B_i\phi\wedge B_i\phi$. To see that $(\mathbf{S},s)\vDash \B_i\phi$, let $t\in S$ such that $s\leq_i t$. We show that $(\mathbf{S},t)\vDash \phi$. By definition of $\kn_i^{\mathbf{S}}$, $s\leq_i t$ gives us that $s\kn_i^{\mathbf{S}} t$. By assumption, this implies that $(\mathbf{M}_{\mathbf{S}}, t)\vDash \phi$, as required. 
        
        Next, to see that $(\mathbf{S},s)\vDash B_i\phi$, recall from (a) that it suffices to show that for all $t\in S$, $s\rightarrow_i^{\mathbf{S}} t$ implies $(\mathbf{M}_{\mathbf{S}}, t) \vDash \phi$. So let $t\in S$ and suppose $s\rightarrow_i^{\mathbf{S}} t$. Then, by definition of $\kn_i^{\mathbf{S}}$, we have $s \kn_i^{\mathbf{S}} t$. The claim then follows directly from our assumption that $s \kn_i^{\mathbf{S}} t$ implies that $(\mathbf{M}_{\mathbf{S}}, t)\vDash \phi$.  
    \end{enumerate}
\end{proof}

Thus, assuming a pseudo-model $\mathbf{S}$ for $\mathcal{L}_{{\B[\A]}_{i,A}}$ with max-dense $\leq$ relations, we can recover the relations corresponding to knowledge and belief to obtain a pseudo-model $\mathbf{M}_{\mathbf{S}}$ for $\mathcal{L}_{KB_{i,A}}$. The resulting model agrees with $\mathbf{S}$ on the interpretation of knowledge and belief. 

\textbf{From pseudo-models for \texorpdfstring{$\mathcal{L}_{KB_{i,A}}$}{KBiA} to pseudo-models for \texorpdfstring{$\mathcal{L}_{{\B[\A]}_{i,A}}$}{BoxForalliA}}\label{sec:box-pseudo-to-model}. For the converse direction, the representation theorem (\Cref{boxforall-on-kb_OLD}) constructs evidence relations on pseudo-models for $\mathcal{L}_{KB_{i,A}}$. It uses the following lemma, which gives us equivalent definitions of the $\sim$ relations that we will define on $\mathbf{M}$. 

\begin{lemma}\label{equivalences-kb-pseudo_OLD}
    Let $s,w,t\in S$ and let $\alpha\in \{A\}\cup A$. Let $\mathbf{M}:= {(S, \kn_i, \rightarrow_i, \kn_A, \rightarrow_A,\llbracket\cdot\rrbracket)}_{i\in A}$ be a pseudo-model for $\mathcal{L}_{KB_{i,A}}$. Then the following are equivalent on $\mathbf{M}$:
    \begin{align*}
        \E t (s\kn_\alpha t \text{ and } w\kn_\alpha t)\quad &\text{ iff } \quad\E t (s\rightarrow_\alpha t \text{ and } w\rightarrow_\alpha t)\\
        &\text{ iff } \quad\A t (s\rightarrow_\alpha t \text{ iff } w\rightarrow_\alpha t).
    \end{align*}
\end{lemma}
\begin{proof}
We prove the following chain of implications: given (1) $\E t (s\kn_\alpha t, w\kn_\alpha t)$, (2) $\E t (s\rightarrow_\alpha t, w\rightarrow_\alpha t)$, and (3) $\A t (s\rightarrow_\alpha t \text{ iff } w\rightarrow_\alpha t)$, we show that (1) $\Rightarrow$ (2) $\Rightarrow$ (3) $\Rightarrow$ (1).

(1) implies (2). Suppose (1). By seriality of $\rightarrow_i$, there exists $t_1\in S$ such that $s\rightarrow_i t_1$. By the strong Euclideanity condition on pseudo-models for $\mathcal{L}_{KB_{i,A}}$, we have $t\rightarrow_i t_1$. Now, applying the strong transitivity condition to $w\kn_i t\rightarrow_i t_1$, we have $w\rightarrow_i t_1$. Thus, there is $t_1$ such that $s\rightarrow_i t_1, w\rightarrow_i t_1$. 
        
(2) implies (3). Suppose (2). Let $u\in S$ be arbitrary and assume, without loss of generality, that $s\rightarrow_i t$. Then $w\rightarrow_i t$ follows immediately from Euclideanity and transitivity of $\rightarrow_i$. 
        
(3) implies (1). Suppose (3). Then, by seriality of $\rightarrow_i$, there is $t$ such that $s\rightarrow_i t$ if and only if $ w\rightarrow_i t$. The claim then follows directly from the inclusion condition on pseudo-models for $\mathcal{L}_{KB_{i,A}}$.
\end{proof}

It is important to note that, in contrast to the relations in the construction of the map in the converse direction (\Cref{kb-on-boxforall_OLD}), the evidence relations defined in \Cref{boxforall-on-kb_OLD} are \emph{not} uniquely determined.\footnote{In particular, an alternative, weaker condition for the relation $\leq_A^{\mathbf{M}}$ on a pseudo-model $\mathbf{M}$ (for $\mathcal{L}_{KB_{i,A}}$) replaces condition (2) in \Cref{boxforall-on-kb_OLD} with the following condition: $(2') \text{ if } s\in Max_{\kn_A}(S), \text{ then } w\ (\bigcap_{i\in A}\kn_i)\ s$. The resulting alternative definition of relations also results in a max-dense pseudo-model for $\mathcal{L}_{{\B[\A]}_{i,A}}$. Nevertheless, we chose condition (2), as it simplifies the proof of \Cref{boxforall-on-kb_OLD}.}

\begin{theorem}[Representing Pseudo-Models for $\mathcal{L}_{KB_{i,A}}$ as pseudo-models for $\mathcal{L}_{{\B[\A]}_{i,A}}$.]\label{boxforall-on-kb_OLD}
    Let $\mathbf{M}:= {(S, \kn_i, \rightarrow_i, \kn_A, \rightarrow_A,\llbracket\cdot\rrbracket)}_{i\in A}$ be a pseudo-model for $\mathcal{L}_{KB_{i,A}}$. We introduce the following relations on $\mathbf{M}$, for all $i\in A$:
    \begin{align*}
        s\sim_i^{\mathbf{M}} w &\text{ iff} \quad \E t (s\kn_i t \text{ and } w\kn_i t) \\
        s\sim_A^{\mathbf{M}} w &\text{ iff}\quad \E t (s\kn_A t \text{ and } w\kn_A t) \\
        s\leq_i^{\mathbf{M}} w &\text{ iff} \quad s\kn_i w \\
        s\leq_A^{\mathbf{M}} w  &\text{ iff} \quad
        \begin{cases}
            (1) & s\ (\kn_A\cap \bigcap_{i\in A} \kn_i)\ w;   \\
            (2) & \text{ if } s\in Max_{\kn_A}(S), \text{ then } w=s.
        \end{cases}
     \end{align*}   
    Then we have that for all $\alpha\in \{A\}\cup A$:
    \begin{enumerate}[a.]
        \item $\leq_\alpha^{\mathbf{M}}\subseteq \kn_\alpha\subseteq\ \sim_\alpha^{\mathbf{M}}$;
        \item $\leq_\alpha^{\mathbf{M}}$ is a preorder and $\sim_\alpha^{\mathbf{M}}$ is an equivalence relation; 
        \item $Max_{\leq_\alpha^{\mathbf{M}}}(S)=Max_{\kn_\alpha}(S)$;
        \item $s\rightarrow_\alpha w$ if and only if $s\sim_\alpha^{\mathbf{M}} w \in Max_{\leq_\alpha^{\mathbf{M}}}(S)$;
        \item $s\kn_\alpha w$ if and only if ($s\leq_\alpha^{\mathbf{M}} w$ or $s\rightarrow_\alpha w$).
    \end{enumerate}
    Furthermore, the structure $\mathbf{S}_{\mathbf{M}}:={(S, \leq_i^{\mathbf{M}}, \sim_i^{\mathbf{M}},\leq_A^{\mathbf{M}},\sim_A^{\mathbf{M}},\llbracket\cdot\rrbracket)}_{i \in A}$ satisfies the following conditions:
    \begin{enumerate}[(I)]
        \item $\mathbf{S}_{\mathbf{M}}$ is a pseudo-model for $\mathcal{L}_{{\B[\A]}_{i,A}}$, such that for each $\alpha\in \{A\}\cup A$, the relation $\leq_\alpha$ is max-dense. 
        \item $\mathbf{S}_{\mathbf{M}}$ and $\mathbf{M}$ agree on the interpretation of the modalities $K_i,K_A,B_i$ and $B_A$. 
    \end{enumerate}
\end{theorem}
\begin{proof}
    We first prove statements (a)-{(e)} for $\mathbf{M}$.
    \begin{enumerate}[(a)]
        \item Let $i\in A$. We prove $\leq_i^{\mathbf{M}}\subseteq \kn_i\subseteq\ \sim_i^{\mathbf{M}}$. For $\leq_i\subseteq \kn_i$, let $s\leq_i^{\mathbf{M}} w$. Then $s\kn_i w$ by definition. Now for $\kn_i\subseteq\ \sim_i^{\mathbf{M}}$, let $s\kn_i w$. For $s\sim_i^{\mathbf{M}} w$, by \Cref{equivalences-kb-pseudo_OLD} it suffices to show that there exists $t\in S$ such that $s\kn_i t, w\kn_i t$. By seriality of $\rightarrow_i$, there is $t\in S$ such that $w\rightarrow_i t$. By the inclusion condition on pseudo-models for $\mathcal{L}_{KB_{i,A}}$ we obtain $w\kn_i t$, so that we have $s\kn_i w \kn_i t$. By transitivity of $\kn_i$ we get $s\kn_i t$. But then we have $t\in S$ such that $w\kn_i t$ and $s\kn_i t$, as required. 

        For the full group, we prove $\leq_A^{\mathbf{M}}\subseteq \kn_A\subseteq\ \sim_A^{\mathbf{M}}$. For $\leq_A^{\mathbf{M}}\subseteq \kn_A$, let $s\leq_A^{\mathbf{M}} w$. Then we obtain $s\kn_A w$ from (1) of the definition of $\leq_A^{\mathbf{M}}$. Next, for $\kn_A\subseteq\ \sim_A^{\mathbf{M}}$, let $s\kn_A w$. By \Cref{equivalences-kb-pseudo_OLD}, it suffices to show that there exists $t\in S$ such that $s\rightarrow_A t, w\rightarrow_A t$. By seriality of $\rightarrow_A$, there is $t\in S$ such that $w\rightarrow_A t$. Having $s\rightarrow_A t$ and $s\kn_A w$, strong Euclideanity on pseudo-models for $\mathcal{L}_{KB_{i,A}}$ gives us that $w\rightarrow_A t$, which concludes our proof.
        
        \item  Let $i\in A$. We show that $\leq_i^{\mathbf{M}}$ is a preorder. Since $\kn_i$ is a preorder, we have that $\leq_i^{\mathbf{M}}$ is a pre-order by definition. 
        
        For the full group, we show that $\leq_A^{\mathbf{M}}$ is a preorder. For reflexivity of $\leq_A^{\mathbf{M}}$, we show that conditions (1) and (2) of the definition of $\leq_A^{\mathbf{M}}$ hold for $s$ with respect to $s$ itself. For (1), observe that we have $s(\kn_A\cap\bigcap_{i\in A}\kn_i) s$: since $\kn_A$ and $\kn_i$ for $i\in A$ are preorders, we have $s\kn_A s$ and $s\kn_i s$ for all $i\in A$. (2) follows immediately from the fact that $s=s$. For transitivity of $\leq_A^{\mathbf{M}}$, suppose that $s\leq_A^{\mathbf{M}} w\leq_A^{\mathbf{M}} t$. We show that $s\leq_A^{\mathbf{M}} t$. Condition (1) follows directly from $\kn_A$ and all $\kn_i$ being transitive: we have $s\kn_A w\kn_A t$ and $s\kn_i w\kn_i t$ by assumption, which gives us $s\kn_A t$ and $s\kn_i t$ for $i\in A$, by transitivity of the respective relations. Thus, $s(\kn_A\cap\bigcap_{i\in A}\kn_i) t$. For condition (2), suppose $s\in Max_{\kn_A}(S)$. Since we assumed $s\leq_A^{\mathbf{M}} w$, we know that $w=s$. But then $w=s\in Max_{\kn_A}(S)$. With $w\leq_A^{\mathbf{M}} t$, we obtain that $t=s$, as required. So $s\leq_A^{\mathbf{M}} t$, giving us transitivity for $\leq_A^{\mathbf{M}}$. 
            
        Let $i\in A$. We show that $\sim_i^{\mathbf{M}}$ is an equivalence relation. For reflexivity, by definition of $\sim_i^{\mathbf{M}}$, to obtain $s\sim_i^{\mathbf{M}} s$ it suffices to show that there exists $t$ such that $s\kn_i t$. By reflexivity of $\kn_i$, we have $s\kn_i s$ and we are done. For transitivity of $\sim_i^{\mathbf{M}}$, let $s\sim_i^{\mathbf{M}} w\sim_i^{\mathbf{M}} t$. So by definition of $\sim_i^{\mathbf{M}}$, there are $u_1,u_2\in S$ such that $s\kn_i u_1$ and $w\kn_i u_1$, and $w\kn_i u_2$ and $t\kn_i u_2$. To show $s\sim_i^{\mathbf{M}} t$, we need $u_3$ such that $s\sim_i^{\mathbf{M}} u_3$ and $t \sim_i^{\mathbf{M}} u_3$. This state exists, because $\kn_i$ is weakly directed: by $w\kn_i u_1$, and $w\kn_i u_2$, there must be $u_3$ such that $u_1\kn_i u_3$ and $u_2\kn_i u_3$. Now we apply transitivity of $\kn_A$ to the fact that $s\kn_i u_1\kn_i u_3$ and $t\kn_i u_2\kn_i u_3$, and we are done. Finally, for symmetry of $\sim_i^{\mathbf{M}}$, let $s\sim_i^{\mathbf{M}} w$. By the symmetric nature of the definition of $\sim_i^{\mathbf{M}}$, we automatically obtain $w\sim_i^{\mathbf{M}} s$. 
       
        For the full group, $\sim_A^{\mathbf{M}}$ is an equivalence relation: the proofs for $\sim_A^{\mathbf{M}}$ are analogous to those for $\sim_i^{\mathbf{M}}$, replacing each occurrence of $\kn_i$ with $\kn_A$. 

        \item Let $i\in A$. Then $Max_{\leq_i^{\mathbf{M}}}(S)=Max_{\kn_i}(S)$ follows immediately from the definition of $\leq_i^{\mathbf{M}}$ on pseudo-models for $\mathcal{L}_{KB_{i,A}}$: we have $s\leq_i^{\mathbf{M}} t$ if and only if $s\kn_i t$, for all $s,t\in S$. 
            
       For the full group, we show that $Max_{\leq_A^{\mathbf{M}}}(S)=Max_{\kn_A}(S)$: for the left-to-right direction, let $s\in Max_{\leq_A^{\mathbf{M}}}(S)$ and suppose for contradiction that $s\notin Max_{\kn_A}(S)$. By the CBD-condition on pseudo-models for $\mathcal{L}_{KB_{i,A}}$, there exists $t\in S$ such that $s\rightarrow t$ and $s(\bigcap_{i\in A}\kn_i) t$. By the inclusion condition, $s\rightarrow t$ gives us that $s\kn t$. So with $s(\kn\cap\bigcap_{i\in A}\kn_i)t$, condition (1) for $s\leq_A^{\mathbf{M}} t$ is satisfied. Furthermore, since we assumed that $s\notin Max_{\kn_A}(S)$, condition (2) holds trivially. Thus, we have $s\leq_A^{\mathbf{M}} t$. Now, since we assumed that $s\in Max_{\leq_A^{\mathbf{M}}}(S)$, we get $t\leq_A^{\mathbf{M}} s$. By definition of $\leq_A^{\mathbf{M}}$, we have $t\kn s$. However, recall that we also have $s\rightarrow_A t$. Strong transitivity on pseudo-models for $\mathcal{L}_{KB_{i,A}}$ then gives us that $t\rightarrow_A t$, so by \Cref{nice-belief_OLD}, $t\in Max_{\kn_A}(S)$. With $t\kn_A s$, it must be that $s\in Max_{\kn_A}(S)$: we have reached a contradiction, and we conclude that $s\in Max_{\kn_A}(S)$ after all. 

       For the converse direction, let $s\in Max_{\kn_A}(S)$ and let $t\in S$ such that $s\leq_A^{\mathbf{M}} t$. We show that $t\leq_A^{\mathbf{M}} s$. Since $s\in Max_{\kn_A}(S)$, we get by $s\leq_A^{\mathbf{M}} t$, and by definition of $\leq_A^{\mathbf{M}}$, that $t=s$. Thus, it remains to show that $t\leq_A^{\mathbf{M}} t$. Condition (1) of the definition of $\leq_A^{\mathbf{M}}$ follows from reflexivity of $\kn_A$ and $\kn_i$ for all $i\in A$; condition (2) is trivially true, since $t=t$. Therefore, $t\leq_A^{\mathbf{M}} t=s$ and so, $s\in Max_{\leq_A^{\mathbf{M}}}(S)$.
        
        \item Let $i\in A$; we show that $s\rightarrow_i w$ if and only if $s\sim_i^{\mathbf{M}} w$ and $w \in Max_{\leq_i^{\mathbf{M}}}(S)$. For the left-to-right direction, let $s\rightarrow_i w$. For $s\sim_i^{\mathbf{M}} w$, we know by the assumption that $s\rightarrow_i w$ and by \Cref{nice-belief_OLD}, that $w\in Max_{\kn_i}(S)$. So (by reflexivity of $\kn_i$ and, again by \Cref{nice-belief_OLD}), we get $w\rightarrow_i w$. But then, with $s\rightarrow_i w$ and $w\rightarrow_i w$, we have $s\sim_i^{\mathbf{M}} w$. It remains to show that $w \in Max_{\leq_i^{\mathbf{M}}}(S)$: recall that $w\in Max_{\kn_i}(S)$, so by our proof of (c) we know that $w \in Max_{\leq_i^{\mathbf{M}}}(S)$. 

        For the converse direction, let $s\sim_i^{\mathbf{M}} w$ with $w \in Max_{\leq_i^{\mathbf{M}}}(S)$. For $s\rightarrow_i w$, by \Cref{nice-belief_OLD} it suffices to show that $s\kn_i w$ and $w \in Max_{\kn_i}(S)$. By our proof of (c) we know that $w \in Max_{\kn_i}(S)$, so it remains to show that $s\kn_i w$. By definition of $\sim_i^{\mathbf{M}}$, there exists $t\in S$ such that $s\kn_i t$ and $w\kn_i t$. But then we must have $t\kn_i w$. So transitivity of $\kn_i$ gives us $s\kn_i w$, as required.

        For the full group, the proof that $s\rightarrow_A w$ if and only if $s\sim_A^{\mathbf{M}} w \in Max_{\leq_A^{\mathbf{M}}}(S)$ is analogous to the proof for individual agents $i\in A$, by replacing each individual relation with its counterpart for the full group. 
        
        \item Let $i\in A$; we show that $s\kn_i w$ if and only if ($s\leq_i^{\mathbf{M}} w$ or $s\rightarrow_i w$). For the left-to-right direction, $s\kn_i w$ implies (by definition of $\leq_i^{\mathbf{M}}$) that $s\leq_i^{\mathbf{M}} w$, as required. For the converse direction, we make a case distinction. If $s\leq_i^{\mathbf{M}} w$, then we obtain $s\kn_i w$ by definition of $\leq_i^{\mathbf{M}}$; on the other hand, if $s\rightarrow_i w$, then we have by \Cref{nice-belief_OLD} that $s\kn_i w$, as required. 
        
        For the full group, we show that $s\kn_A w$ if and only if ($s\leq_A^{\mathbf{M}} w$ or $s\rightarrow_A w$): for the left-to-right direction, let $s\kn_A w$. By the WM-condition on pseudo-models for $\mathcal{L}_{KB_{i,A}}$, we have for all $i\in A$ that $s(\kn_i\ \cup \rightarrow_A)w$. If $s\rightarrow_A w$, we are done. So suppose not. Then $s\kn_i w$ for all $i\in A$. We claim that this implies $s\leq_A^{\mathbf{M}} w$. For condition (1) of the definition of $\leq_A^{\mathbf{M}}$ on pseudo-models for $\mathcal{L}_{KB_{i,A}}$, observe that we have $s\kn_A w$ by assumption, and $s(\bigcap_{i\in A}\kn_i) w$ by the WM-condition on for $\mathcal{L}_{KB_{i,A}}$ (and by the assumption that $s\not\rightarrow_A w$). Thus, $s(\kn_A \cap \bigcap_{i\in A}\kn_i)w$. For condition (2), suppose that $s\in Max_{\kn_A}(S)$. Then by $s\kn_A w$ we obtain that $w\in Max_{\kn_A}(S)$. With $s\kn_A w$, this means that $s\rightarrow_A w$ (\Cref{nice-belief_OLD}), which we assumed was not the case: a contradiction. Therefore, $s$ cannot be contained in $Max_{\kn_A}(S)$ and condition (2) is vacuously true. We conclude that $s\leq_A^{\mathbf{M}} w$, as required. 
    \end{enumerate}
    It remains to prove statements (I) and (II). 
    \begin{enumerate}[(I)]
        \item First, to see that $\mathbf{S}_{\mathbf{M}}$ is a pseudo-model for $\mathcal{L}_{{\B[\A]}_{i,A}}$, we check the following conditions (\Cref{pseudo-model}) on the model $\mathbf{M}$.\footnote{The only difference between $\mathbf{S}_{\mathbf{M}}$ and $\mathbf{M}$ is that $\mathbf{S}_{\mathbf{M}}$ does not contain the relations for knowledge and belief from $\mathbf{M}$, and we need these relations to prove our claims.}
        
        Relations $\leq_A^{\mathbf{M}}$ and $\leq_i^{\mathbf{M}}$ for $i\in A$ are preorders, and $\sim_A^{\mathbf{M}}$ and $\sim_i^{\mathbf{M}}$ for $i\in A$ are equivalence relations. This is stated and proved in (b).  
        
        The relations $\leq^{\mathbf{M}}$ are anti-monotone: suppose that $s\leq_A^{\mathbf{M}} t$ and let $i\in A$. By definition of $\leq_A^{\mathbf{M}}$, $s\leq_A^{\mathbf{M}} t$ implies $s\kn_i t$. The definition of $\leq_i^{\mathbf{M}}$ then gives us $s\leq_i^{\mathbf{M}} t$, as required.  
            
        The relations $\sim^{\mathbf{M}}$ are anti-monotone: suppose that $s\sim_A^{\mathbf{M}} t$ and let $i\in A$. We show that $s\sim_i^{\mathbf{M}} t$. By definition of $\sim_i^{\mathbf{M}}$, there exists $u\in S$ such that $s\kn u$ and $t\kn u$. Now by seriality of $\rightarrow_i$, there exists $v_1\in S$ such that $s\rightarrow_i v_1$. With $s\kn u$ and $s\rightarrow_i v_1$, the super-introspection condition on pseudo-models for $\mathcal{L}_{KB_{i,A}}$ then gives us that $u\rightarrow_i v_1$. By the same condition, applied to $t\kn u$ and $u\rightarrow_i v_1$, we have $t\rightarrow_i v_1$. But then there exists $v_1$ such that both $s\rightarrow_i v_1$ and $u\rightarrow_i v_1$, which gives us $s\sim_i^{\mathbf{M}} u$ by the definition of $\sim_i^{\mathbf{M}}$. 
 
        Inclusion: we have $\leq_i^{\mathbf{M}}\subseteq\ \sim_i^{\mathbf{M}}$ for all $i\in A$, and $\leq_A^{\mathbf{M}}\subseteq\ \sim_A^{\mathbf{M}}$. This is stated and proved in (a).

        It remains to show that the $\leq^{\mathbf{M}}$ relations are max-dense. For individual agents, let $i\in A$ and let $s\in S$. We need to find some $t\in Max_{\leq_i^{\mathbf{M}}}(S)$, such that $s\leq_i^{\mathbf{M}} t$. By seriality of $\rightarrow_i$, there exists $t\in S$ such that $s\rightarrow_i t$. Now, by (d), $t\in  Max_{\leq_i^{\mathbf{M}}}(S)$. The inclusion condition on $\mathbf{M}$ gives us, by $s\rightarrow_i t$, that $s\kn_i t$. By definition of $\leq_i^{\mathbf{M}}$, we have $s\leq_i^{\mathbf{M}} t$, as required. 
        
        For the full group, we find some $t\in Max_{\leq_A^{\mathbf{M}}}(S)$, such that $s\leq_A^{\mathbf{M}} t$. We consider two cases:
        \begin{enumerate}
            \item $s\in Max_{\kn_A}(S)$. We take $t:=s$ and show that conditions (1) and (2) are satisfied for $s\leq_A^{\mathbf{M}} s$: condition (1) follows from reflexivity of $\kn_A$, as well as all $\kn_i$ relations. Condition (2) is trivially true, by the fact that $s=s$. Therefore, $s\leq_A^{\mathbf{M}} s$, as required.
            \item $s\notin Max_{\kn_A}(S)$. By the CBD-condition on $\mathbf{M}$, there exists $w\in S$ such that $s\rightarrow_A w$ and for all $i\in A$, $s\kn_i w$. We take $t:=w$. By (d), we have $w\in  Max_{\leq_A^{\mathbf{M}}}(S)$. It remains to show that conditions (1) and (2) are satisfied for $s\leq_A^{\mathbf{M}} w$. For (1), note that $s\rightarrow_A w$ implies $s\kn_A w$ (by the inclusion condition on $\mathbf{M}$), and recall that for all $i\in A$, $s\kn_i w$. Now, (2) is vacuously satisfied, as $s\notin Max_{\kn_A}(S)$. Therefore, $s\leq_A^{\mathbf{M}} w$, as required.
        \end{enumerate}
        In conclusion, $\mathbf{S}_{\mathbf{M}}$ is a pseudo-model for $\mathcal{L}_{{\B[\A]}_{i,A}}$, such that all $\leq^{\mathbf{M}}$ relations are max-dense. 

        \item To show that the pseudo-model $\mathbf{M}$ for $\mathcal{L}_{KB_{i,A}}$ and the pseudo-model $\mathbf{S}$ for $\mathcal{L}_{{\B[\A]}_{i,A}}$ agree on the interpretation of the modalities $K_i,B_i, K_A$, and $B_A$, we prove that the primitive knowledge and belief relations on $\mathbf{M}$ coincide with the recovered knowledge and belief relations on $\mathbf{S}$. 

        First, recall the result from \Cref{kb-on-boxforall_OLD}, which states that for all $\alpha\in \{A\}\cup A$, we can recover relations $\rightarrow_\alpha^{\mathbf{S}}$ and $\kn_\alpha^{\mathbf{S}}$ for knowledge and belief on the pseudo-model $\mathbf{S}$ for $\mathcal{L}_{{\B[\A]}_{i,A}}$, given that the $\leq$ relations for all agents and the full group are max-dense. 
        
        Now consider our pseudo-model $\mathbf{S}_{\mathbf{M}}$, which was obtained from $\mathbf{M}$ by recovering the evidence relations $\sim_\alpha^{\mathbf{M}}$ and $\leq_\alpha^{\mathbf{M}}$ for all $\alpha\in \{A\}\cup A$. By (I), it satisfies the conditions from \Cref{kb-on-boxforall_OLD}. So suppose we apply \Cref{kb-on-boxforall_OLD} to recover the uniquely determined knowledge and belief relations $\rightarrow_\alpha^{\mathbf{S}}$ and $\kn_\alpha^{\mathbf{S}}$ in terms of the (recovered, and not uniquely determined) evidence relations $\sim_\alpha^{\mathbf{M}}$ and $\leq_\alpha^{\mathbf{M}}$. 
        
        Then, by combining the results from \Cref{kb-on-boxforall_OLD} and from (d) and (e) of this proposition, we have that $\rightarrow_\alpha^{\mathbf{S}}=\rightarrow_\alpha$ and $\kn_\alpha^{\mathbf{S}}=\kn_\alpha$, where $\rightarrow_\alpha$ and $\kn_\alpha$ represent the primitive relations for knowledge and belief on the pseudo-model $\mathbf{M}$.
        
        Thus, by \Cref{kb-on-boxforall_OLD}, we have for all  $\alpha\in \{A\}\cup A$ and $s\in S$, that 
        \[
        \begin{array}{lllll}
            (\mathbf{M},x)\vDash B_\alpha \phi 
            &\text{ iff } & \text{ for all } t\in S \text{ s.t. } s\rightarrow_\alpha t: &(\mathbf{M}, t) \vDash \phi &\text{ (Def.~\ref{sem-pseudo-kb}) } \\
            &\text{ iff } & \text{ for all } t\in S \text{ s.t. } s\rightarrow_\alpha^{\mathbf{S}} t: &(\mathbf{M}, t) \vDash \phi &\text{ (Prop.~\ref{kb-on-boxforall_OLD}, (d), (e))} \\
            &\text{ iff } & & (\mathbf{S}_{\mathbf{M}},x)\vDash B_\alpha \phi &\text{ (Prop.~\ref{kb-on-boxforall_OLD})} 
        \end{array}
        \]

        \[
        \begin{array}{lllll}
            (\mathbf{M},x)\vDash K_\alpha \phi 
            &\text{ iff } &\text{ for all } t\in S \text{ s.t. } s \kn_\alpha t: &(\mathbf{M}, t)\vDash \phi &\text{ (Def.~\ref{sem-pseudo-kb}) } \\
            &\text{ iff } &\text{ for all } t\in S \text{ s.t. } s \kn_\alpha^{\mathbf{S}} t: &(\mathbf{M}, t)\vDash \phi &\text{ (Prop.~\ref{kb-on-boxforall_OLD}, (d), (e))}\\
            &\text{ iff } & & (\mathbf{S}_{\mathbf{M}},x)\vDash K_\alpha \phi &\text{ (Prop.~\ref{kb-on-boxforall_OLD})}\\
        \end{array}
        \]
        which concludes our proof.   
    \end{enumerate}
\end{proof}

Thus, we can recover evidence relations on the pseudo-model $\mathbf{M}$ for $\mathcal{L}_{KB_{i,A}}$, that result in a pseudo-model $\mathbf{S}_{\mathbf{M}}$ for $\mathcal{L}_{{\B[\A]}_{i,A}}$, that agrees with $\mathbf{M}$ on the interpretation of knowledge and belief. We extend this claim to all formulas in the language $\mathcal{L}_{KB_{i,A}}$:

\begin{corollary}\label{comp-sim_OLD}    
    The same formulas in $\mathcal{L}_{KB_{i,A}}$ are satisfiable in the pseudo-model $\mathbf{M}$ for $\mathcal{L}_{KB_{i,A}}$, as in the pseudo-model $\mathbf{S}$ for $\mathcal{L}_{{\B[\A]}_{i,A}}$. 
\end{corollary}
\begin{proof}
    By induction on the complexity of $\phi$. We compare the interpretation of formulas $\phi\in \mathcal{L}_{KB_{i,A}}(\prop)$ on the pseudo-model $\mathbf{S}$ for $\mathcal{L}_{{\B[\A]}_{i,A}}$, according to \Cref{sem-boxall-pseudo}, with the interpretation on the pseudo-model $\mathbf{M}$ for $\mathcal{L}_{KB_{i,A}}$ according to \Cref{sem-pseudo-kb}. 
    
    For the atomic case, and for the boolean cases of the inductive step, we observe that the interpretations agree on the semantics of atomic propositions and boolean combinations. For formulas of the form $K_i\psi,B_i\psi,K_A\psi$, and $B_A\psi$, the claim follows from \Cref{kb-on-boxforall_OLD} and \Cref{boxforall-on-kb_OLD} (II).
\end{proof}

We now finally prove \Cref{corr:compness-kb}: we show that the proof system $\bm{KB_{i,A}}$ from \Cref{pf-systKB} is sound and weakly complete with respect to multi-agent topo-e-models, and that the logic of $\bm{KB_{i,A}}$ is decidable. 

\medskip
\paragraph{Proof of \Cref{corr:compness-kb}} Soundness of the axioms and rules in \Cref{pf-systKB} is a routine check, therefore we omit these proofs. As for completeness, let $\phi\in \mathcal{L}_{KB_{i,A}}(\prop)$ be any consistent formula. By \Cref{comp-pseudo-KB}, $\phi$ is satisfiable on a pseudo-model $\mathbf{M}$ for $\mathcal{L}_{KB_{i,A}}$. By \Cref{comp-sim_OLD}, there exists an equivalent pseudo-model $\mathbf{S}$ for $\mathcal{L}_{{\B[\A]}_{i,A}}$. Thus, $\phi$ is satisfied on $\mathbf{S}$. By \Cref{coroll-bisim-frag_OLD}, $\phi$ is satisfiable on the associated model for $\mathbf{S}$. Since this is a relational evidence model, we obtain weak completeness for $\mathcal{L}_{KB_{i,A}}$ with respect to relational evidence models. Finally, by \Cref{cor:rel-sem-equiv-concl_OLD}, we obtain weak completeness with respect to topo-e-models. Decidability of $\mathcal{L}_{KB_{i,A}}$ follows from decidability of the larger language $\mathcal{L}_{{\B[\A]}_I}$ (see \Cref{corr:compness-boxall-frag}).


\subsection{Proof of Completeness and Decidability for the Dynamic Logics of Evidence-Sharing (\texorpdfstring{\Cref{compness-boxall-dyn}}{Theorem 10})}

We follow a standard approach in DEL literature: to show the completeness of the axiomatic system for th dynamic extension of a static logic that has already been proven to be complete, it suffices to use the dynamic reduction axioms to show that the static language is provably equally expressive as its dynamic extension. 

We apply this technique to the dynamic extensions of each of the static proof systems $\bm{{\B[\A]}_I}$ and and $\bm{KB_{i,A}}$. We first need the following two lemmas, as preliminary steps, that establish the \textit{elimination of a one-step dynamic modality} for each of these logics.

\begin{lemma}\label{lem:comp-dyn-all_OLD}
    Let $I\subseteq A$ be any group of agents. Then, for every ``static'' formula $\phi$ in the language $\mathcal{L}_{{\B[\A]}_I}$, there exists another ``static'' formuula $\phi_{I}$ in $\mathcal{L}_{{\B[\A]}_I}$, such that
    \[
    \vdash [\share_I]\phi \leftrightarrow \phi_{I}
    \]
    is provable in the system $\bm{{\B[\A]}_I [\share_I]}$.
\end{lemma}
\begin{proof}
We prove the existence of $\phi_I$, by
    induction on the complexity of the static formula $\phi$.

    For the atomic case, where $\phi:=p$, the Atomic Reduction axiom from \Cref{pf-syst-boxall-dyn} gives us that $\vdash [\share_I]p\leftrightarrow p$, so we can take $\phi_{I}:=p$.

    For $\phi:=\neg\psi$, we apply the induction hypothesis to $\psi$ to obtain that there exists $\psi_{I}\in \mathcal{L}_{{\B[\A]}_I}$ such that $\vdash [\share_I]\psi \leftrightarrow \psi_{I}$. By the Negation Reduction axiom, we have $\vdash [\share_I]\neg \psi \leftrightarrow \neg[\share_I] \psi$ which, with the induction hypothesis, gives us that $\vdash [\share_I]\neg \psi \leftrightarrow \neg \psi_{I}$. Thus, we can take $\phi_{I}:=\neg\psi_{I}$.

    For $\phi:=\psi\wedge\chi$, the proof is similar, using the derived law of Conjunction Reduction (which as we saw is a theorem in this system).

    For $\phi:= \B_J\psi$, where $J\subseteq A$ is a group, we use the $\B$-Reduction axiom, to obtain that $\vdash [\share_I]\B_J\psi\leftrightarrow \B_{J/+ I}[\share_I]\psi$. By the induction hypothesis, there exists $\psi_{I}\in \mathcal{L}_{{\B[\A]}_I}$ such that $\vdash [\share_I]\psi \leftrightarrow \psi_{I}$.
    Thus, using the normality of $[\share_I]$, we have $\vdash [\share_I]\B_J\psi\leftrightarrow \B_{J/+ I}\psi_{I}$, therefore, we can take $\phi_I:= \B_{J/+}\psi_{I}$.

    For $\phi:= {[\A]}_J\psi$,  where $J\subseteq A$ is a group, the proof is similar to the previous case, using the $[\A]$-Reduction axiom instead of the $\B$-Reduction. 
\end{proof}

\begin{lemma}\label{lem:comp-dyn-all_NEW}
    Let $\phi$ be any ``static'' formula in the language $\mathcal{L}_{KB_{i,A}}$. Then there exists some ``static'' formula $\phi_{(A)}$ in $\mathcal{L}_{KB_{i,A}}$, such that
    \[
    \vdash [\share_A]\phi \leftrightarrow \phi_{(A)}
    \]
    is provable in $\bm{KB_{i,A} [\share_A]}$.
\end{lemma}
\begin{proof}
The proof is completely similar to the one of Lemma~\ref{lem:comp-dyn-all_OLD}, but using the reduction axioms for $[\share_A]$ in Table~\ref{KB-dyn} instead of the ones in \Cref{pf-syst-boxall-dyn}.
\end{proof}

Now, we can establish the desired co-expressivity results.

\begin{proposition}\label{lem:comp2-dyn-all_OLD}
    For every ``dynamic'' formula $\phi$ in the language $\mathcal{L}_{{\B[\A]}_I [\share_I]}$, there exists some ``static'' formula $\phi^{\prime}\in \mathcal{L}_{{\B[\A]}_I}$ such that
    \[
    \vdash \phi \leftrightarrow \phi^{\prime}
    \]
    is provable in the system $\bm{{\B[\A]}_I [\share_I]}$.
\end{proposition}
\begin{proof}
    By induction on the complexity of the dynamic formula $\phi$. 
    
    For atoms $\phi:=p$, we have that $\phi\in\mathcal{L}_{{\B[\A]}_I}$, so we can take $\phi^{\prime}:=\phi$.

    For $\phi:=\neg\psi$, apply the induction hypothesis to $\psi$ to obtain $\psi^{\prime}\in \mathcal{L}_{{\B[\A]}_I}$ such that $\vdash \psi \leftrightarrow \psi^{\prime} $. But then $\vdash \neg\psi \leftrightarrow \neg\psi^{\prime}$, so we take $\phi^{\prime}:=\neg\psi^{\prime}$.

    For $\phi:=\psi\wedge\chi$, the proof is similar.

    For $\phi:= \B_I\psi$, apply the induction hypothesis to $\psi$ to obtain $\psi^{\prime}\in \mathcal{L}_{{\B[\A]}_I}$ s.t. $\vdash \psi \leftrightarrow \psi^{\prime}$. Then $\vdash \B_I\psi\leftrightarrow \B_I\psi^{\prime}$ by the normality of $\B_I$, so we can take $\phi^{\prime}:=\B_I\psi^{\prime}$.

    For $\phi:= {[\A]}_I\psi$, the proof is similar to the proof for $\phi:= \B_J\psi$.

    For $\phi:=[\share_I]\psi$, apply the induction hypothesis to $\psi$ to obtain $\psi^{\prime}\in \mathcal{L}_{{\B[\A]}_I}$ such that $\vdash \psi \leftrightarrow \psi^{\prime}$. By the normality of $[\share_I]$ we have $\vdash [\share_I]\psi \leftrightarrow [\share_I]\psi^{\prime}$, while by \Cref{lem:comp-dyn-all_OLD} we have $\vdash [\share_I]\psi^{\prime}\leftrightarrow \psi^{\prime}_I$, thus we obtain that $\vdash [\share_I]\psi \leftrightarrow \psi^{\prime}_I$. So we can take $\phi^{\prime}:=\psi^{\prime}_I$.
\end{proof}

We also have the analogue co-expressivity result for the dynamic and static logics of group knowledge and belief:

\begin{proposition}\label{lem:comp2-dyn-all_NEW}
    For every ``dynamic'' formula $\phi$ in the language $\mathcal{L}_{KB_{i,A} [\share_A]}$, there exists some ``static'' formula $\phi^{\prime}\in \mathcal{L}_{KB_{i,A}}$ such that
    \[
    \vdash \phi \leftrightarrow \phi^{\prime}
    \]
    is provable in the system $\bm{KB_{i,A} [\share_A]}$.
\end{proposition}

\begin{proof}
The proof is completely similar to the one of Proposition~\ref{lem:comp2-dyn-all_OLD}, but using Lemma~\ref{lem:comp-dyn-all_NEW}
instead of Lemma~\ref{lem:comp-dyn-all_OLD}.
\end{proof}

Finally, we prove the completeness and decidability of our dynamic logics: 

\medskip
\paragraph{Proof of \Cref{compness-boxall-dyn}} 
    We prove the claims for proof systems $\bm{{\B[\A]}_I [\share_I]}$ and $\bm{KB_{i,A} [\share_A]}$:
    
    \begin{enumerate}
        \item The \textit{soundness} of the reduction axioms from Table~\ref{pf-syst-boxall-dyn} is a routine verification. 

        The \textit{decidability} of the logic axiomatized by the proof system $\bm{KB_{i,A} [\share_A]}$ follows immediately from Proposition~\ref{lem:comp2-dyn-all_OLD} (the provable co-expressivity of the static and dynamic logics) together with the decidabilty of the logic axiomatized by the system $\bm{{\B[\A]}_I}$ (\Cref{corr:compness-boxall-frag}).

        For the \textit{completeness} of $\bm{KB_{i,A} [\share_A]}$, we also use the fact that the dynamic language $\mathcal{L}_{{\B[\A]}_I[\share_I]}$ is co-expressive with its static base $\mathcal{L}_{{\B[\A]}_I}$. Let $\phi\in \mathcal{L}_{{\B[\A]}_I[\share_I]}(\prop)$ be a consistent formula (w.r.t.\ the proof system $\bm{{\B[\A]}_I[\share_I]}$), and we need to show that $\phi$ is satisfiable. By \Cref{lem:comp2-dyn-all_OLD}, there exists $\phi^{\prime}\in \mathcal{L}_{{\B[\A]}_I}(\prop)$ such that $\vdash \phi \leftrightarrow \phi^{\prime}$ is a theorem in 
        the system $\bm{{\B[\A]}_I[\share_I]}$. By the soundness of $\bm{{\B[\A]}_I[\share_I]}$, it follows that $\varphi'$ is consistent (w.r.t.\ the system  $\bm{{\B[\A]}_I[\share_I]}$, hence w.r.t.\ the subsystem $\bm{{\B[\A]}_I}$). By \Cref{corr:compness-boxall-frag}, there exists a pointed multi-agent topo-e-model $(\mfM,x)$ such that $(\mfM,x)\vDash \phi^\prime$. Applying again the soundness of $\bm{{\B[\A]}_I[\share_I]}$, $\models \phi \leftrightarrow \phi^{\prime}$ is valid, and so also have  $(\mfM,x)\vDash \phi$, as desired. 

        \item The proofs for the system $\bm{KB_{i,A}[\share_A]}$ are completely similar, using Proposition~\ref{lem:comp2-dyn-all_NEW}
        and \Cref{corr:compness-kb} instead of Proposition~\ref{lem:comp2-dyn-all_OLD} and \Cref{corr:compness-boxall-frag}.
    \end{enumerate}

}

\end{document}